\documentclass{article}
\usepackage{float}
\usepackage[version = 4]{mhchem}
\usepackage{pdfpages}
\usepackage{mathrsfs}
\usepackage{longtable}
\usepackage{listings}
\usepackage{textcomp}
\usepackage{cancel}
\usepackage{enumerate}
\usepackage{amsmath}
\usepackage{amsthm}
\usepackage{amssymb}
\usepackage{booktabs}
\usepackage{colortbl}
\usepackage{graphicx}
\usepackage{subfigure}
\usepackage{siunitx}
\usepackage{appendix}
\usepackage{wrapfig}
\usepackage{geometry}
\usepackage{indentfirst}
\usepackage{blkarray}
\usepackage{multirow} 
\usepackage{amsmath}
\usepackage{amssymb}
\usepackage{makecell}
\usepackage{pifont}

\usepackage{JI_MathCourse_Notations}

\usepackage{jmlr2e}

\newtheorem{asmp}{Assumption}[section]
\newtheorem{cor}{Corollary}[section]
\newtheorem{defn}{Definition}[section]
\newtheorem{thm}{Theorem}[section]
\newtheorem{prop}{Proposition}[section]

\theoremstyle{definition}

\newcommand{\tw}{\tilde{w}}
\newcommand{\ttheta}{\Tilde{\theta}}

\newcommand{\bara}{\bar{a}}

\newcommand{\barw}{\bar{w}}

\newcommand{\Rmuzero}{R^{-1}\{0\}}

\title{Geometry and Local Recovery of Global Minima of Two-layer Neural Networks at Overparameterization}

\author{%
    \name Tao Luo${}^{a, b}$\thanks{Corresponding author} \email luotao41@sjtu.edu.cn \\
    \name Leyang Zhang${}^{c}$ \email leyangz\_hawk@outlook.com \\
    \name Yaoyu Zhang${}^{a}$ \email zhyy.sjtu@sjtu.edu.cn \\
    \addr ${}^a$ School of Mathematical Sciences, Institute of Natural Sciences and MOE-LSC,
    Shanghai Jiao Tong University, Shanghai, 200240, China \\
    \addr ${}^b$ CMA-Shanghai, Shanghai Artificial Intelligence Laboratory, Shanghai, China \\
    \addr ${}^c$ School of Mathematics, Georgia Institute of Technology, Atlanta, GA, 30332,
    United States}

\editor{}

\begin{document}

\maketitle

\begin{abstract}
     Under mild assumptions, we investigate the geometry of the loss landscape for two-layer neural networks in the vicinity of global minima. Utilizing novel techniques, we demonstrate: (i) how global minima with zero generalization error become geometrically separated from other global minima as the sample size grows; and (ii) the local convergence properties and rate of gradient flow dynamics. Our results indicate that two-layer neural networks can be locally recovered in the regime of overparameterization.\\
\end{abstract}

keywords: neural networks,  global minima, analytic function theory, non-linear analysis, gradient flow, recovery stability\\

\tableofcontents
~\\

\section{Introduction}

Over the past decade, neural networks—a distinct class of nonlinear models—have revolutionized the field of artificial intelligence, including the theoretical studies and applications. However, the mathematical underpinnings that set them apart from other models are not well comprehended. Deciphering these structures poses a daunting but essential challenge for the field of mathematics. Among the various aspects of neural networks, the loss landscape is obviously important in shaping their training dynamics and generalization capabilities \cite{WeinanE, RSun}. Of particular interest is the geometry of the global minima, which lies at the heart of a fundamental enigma: how neural networks are able to find well-generalizing solutions via global training from an ostensibly infinite pool of global minima—a majority of which do not generalize well—especially when the network is overparameterized \cite{CZhang, LBreiman}. \\

In our research, we propose and address the geometric structure and local recovery problem for two-layer neural networks. To be precise, we uncover an important geometric structure of the perfect global minima, namely, set that recovers the target function: this set consists of different ``branches" (subsets) which become separated from imperfect global minima at overparameterization. The separation property guarantees local recovery capability of the target function and surely reduces the difficulty in finding well-generalizing solutions globally via gradient dynamics. In fact, the separation of branches occur successively as training sample size increases. Beyond the separation property, we show that the geometry near different branches are vary significantly, which leads to the distinct limiting behavior of gradient flows nearby. Note that all these are inherent to the family of neural network models.\\

Our work demonstrates the profound impact of the fine geometry of global minima to the generalization of neural networks. In traditional machine learning problems, the global minima usually have trivial geometry, e.g., isolated points, rendering generalization a separate issue from loss landscape analysis \cite{RSun}. However, for neural networks, in particular at overparameterization, we showcase the importance of analyzing the geometry of the well-generalizing set, e.g., the perfect global minima, amid the global minima of loss landscape for the understanding of generalization. Our view of the global minima that particularly highlights the structure of the perfect global minima as the backbone and its generalization consequence is a significant refinement over the view suggested by Cooper \cite{YCooper} which focuses on the overall geometry of the global minima. We hope our work could convince the audiences in mathematics that the global minima, or more broadly the critical points of the loss landscape, of neural networks possess rich geometric structures that are  amenable to analysis, mathematically interesting, intrinsic to the model architecture,  and undoubtedly plays a fundamental role to their exceptional training and generalization performance. \\

% In contrast to previous works that focus on the geometry of critical points/manifolds of loss landscape [\cite{KFukumizu,BSimsek,EbddPrincipleShort,EbddPrincipleLong}], 
% Among all different aspects of neural networks, the loss landscape structure plays a particularly important role in determining their training dynamics and generalization performance. In particular, structure of its global minima is central to a key mystery of neural networks---how and why neural networks find well-generalized solutions among infinite global minima (most of which generalize poorly) at overparameterization. In this work, we lay a foundation to resolve this mystery by uncovering the rich geometry of the target set with $0$ generalization error and exhibiting how different part of this set are gradually peeled from the set of global minima as the increase of sample size.

% In practice, gradient methods have been widely used for training and often achieve satisfactory generalization performance, even without extra strategies. Gradient flow, which is closely related to other gradient methods, thus plays a central role in the theoretical study of machine learning. A key challenge is to characterize the loss landscape and behaviors of gradient flow. As the gradient flow often converges and reaches extremely low training error, it is particularly important and interesting to study these objects near the global minima of a loss function. 

Technically speaking, the study of global minima in the loss landscape presents two fundamental issues. The first issue is the linear independence of neurons (and their derivatives). It determines geometric structures of the loss landscape, such as dimensions of global minima. Previous works address the linear independence of neurons for the analysis of critical points \cite{BSimsek, RSun}. In our work, to understand when perfect global minima are separated from the imperfect ones, we establish the linear independence of neurons \textit{and their derivatives}. The second issue is that neural network architecture yields highly degenerate Hessian of loss at a global minimum. The loss may fail to be a Morse function, or even a generalized Morse--Bott function, thus beyond the reach of traditional methods. We devise methods for addressing gradient flow dynamics in the vicinity of degenerate critical points, which are indispensable to the analysis of neural networks. This outcome broadens the scope of most existing research that primarily focuses on gradient flow dynamics near critical points of Morse and generalized Morse--Bott functions. Furthermore, our investigations into the dynamics around global minima have yielded unprecedented results in terms of convergence rates and directions.
\\

Specifically, we start by presenting the main results of this paper in an illustrative (but informal) way in Section \ref{Section Glance at this Paper}. In Section \ref{Section Preparing Lemmas and Propositions}, we prove lemmas and propositions which will be used to derive these main results. Importantly, we investigate the linear independence of neurons (and their partial derivatives against parameters) and then define ``separating inputs" (Definition \ref{Defn Separating inputs}). The next section, Section \ref{Section Loss landscape near Minfty} gives a detailed treatment of the geometry of global minima and the functional properties of loss near the perfect global minima of it. These help us to characterize the gradient flow near the perfect global minima in Section \ref{Section Dynamics of gradient flow near Minfty}, where the convergence, convergence rate, limiting direction and recovery stability (Definition \ref{Defn Recovery stability}) of such gradient flows are investigated. Finally, we make conclusion and discussion of our work in Section \ref{Section Conclusion and discussion}. \\

The following diagram (see next page) summarizes and demonstrates the interconnections of theoretical results in the main part of our paper (Section \ref{Section Preparing Lemmas and Propositions}, \ref{Section Loss landscape near Minfty}, \ref{Section Dynamics of gradient flow near Minfty}).

\begin{figure}[H]
    \centering
    \includegraphics[width=\textwidth]{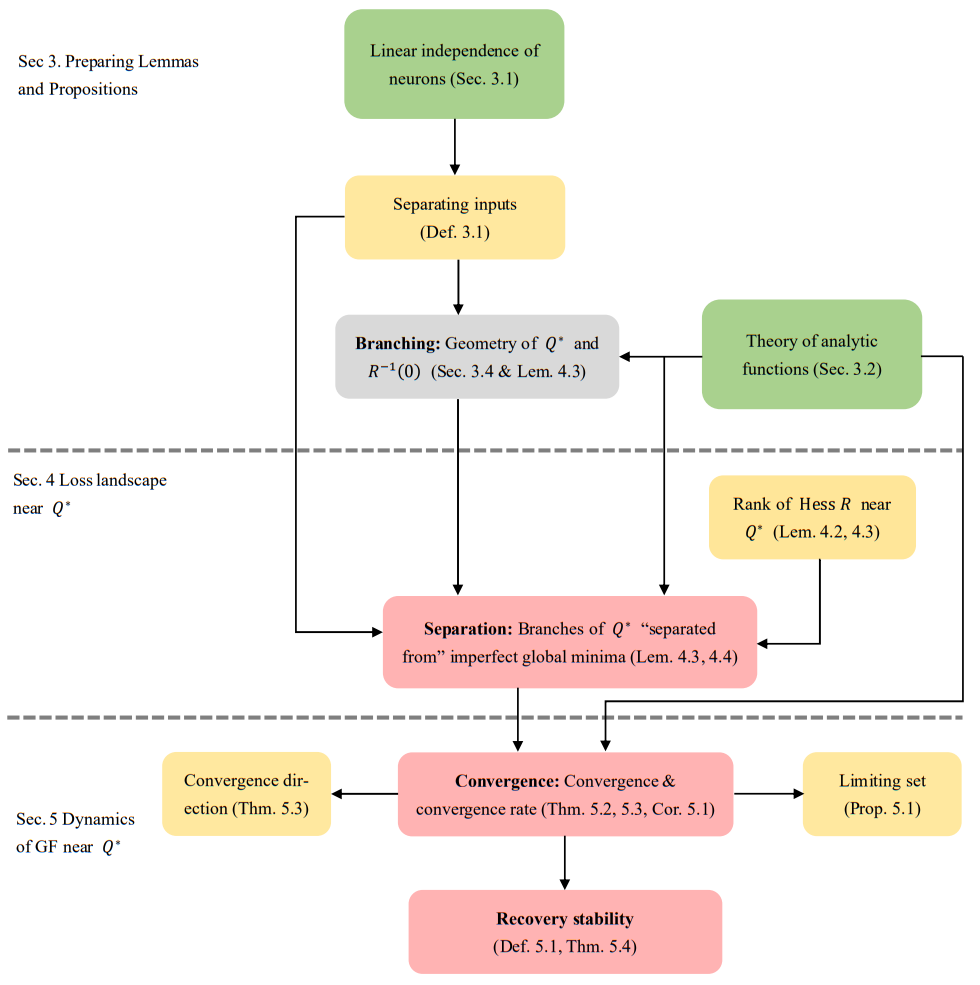}
    \caption{Overview of theoretical results and their interconnections. The main parts are in \textcolor{magenta}{dark pink} boxes, the basic theories are in \textcolor{teal}{green} boxes and the other results are in \textcolor{yellow}{yellow} boxes.}
    \label{Fig Mind Map}
\end{figure}

\section{A Glance at this Paper}\label{Section Glance at this Paper}

\subsection{Notations and Assumptions}\label{Subsection Notations and assumptions}

In this section, we make clear the notations and assumptions which we shall use throughout the paper, unless we specify it explicitly. Let $\bN$ denote the set of natural numbers $\{1, 2, 3, ...\}$. For any two elements $x, y$ of a Hilbert space, we use $x \cdot y$ and $\<x, y\>$ (interchangeably) to indicate the inner product of $x$ and $y$. For any function $f$, $\nabla f$ is the gradient of $f$ and $\operatorname{Hess} f$ the Hessian of $f$ (given that either one exists). Given a subset $E$ of a Hilbert space, we denote its closure (in the Euclidean space) by $\bar{E}$, and we say $E$ has $\lambda_k$-measure zero if the $k$-dimension Lebesgue measure of it is well-defined and is zero. We use $B(\theta, r)$ to denote the open ball centered at $\theta$ with radius $r$; when the $B(\theta, r)$ is in a Euclidean space, it is the open ball with respect to the standard norm. \\ 

Then we make assumptions on the objects we study, starting from the activations. Indeed, different kinds of activations give different loss landscape geometry. In our paper, we focus on a generic collection of analytic activations, which turns out to give a reasonable linear independence result: given $m \in \bN$ and distinct $w_1, ..., w_m \in \bR^d$, $\sigma(w_1\cdot x), ..., \sigma(w_m\cdot x)$ are linearly independent. This will then give the ``simplest possible" structure of global minima of loss function. See Section \ref{Subsection geometry of Minfty}. \\

\begin{asmp}[generic (analytic) activation]\label{generic activation}
    We consider any analytic activation $\sigma: \bR \to \bR$ such that 
    \begin{equation}
        \sigma(x) = \sum_{j=0}^\infty c_j x^j, \quad x \in (-R, R) \siq \bR
    \end{equation} 
    where $R$ is the radius of convergence, $c_0 \neq 0$, and for any $N \in \bN$ there are some odd number $j_{\mathrm{odd}}$ and even number $j_{\mathrm{even}}$ both greater than $N$ with $c_{j_{\mathrm{odd}}} \neq 0$, $c_{j_{\mathrm{even}}} \neq 0$. We call any such $\sigma$ a generic activation. 
\end{asmp} 

For example, the exponential activation $\exp(x)$ satisfies these requirements, while some other commonly-seen activation functions, including $\sigma(x) = \frac{1}{1+e^x}$, $\sigma(x) = \tanh(x) = \frac{e^x - e^{-x}}{e^x + e^{-x}}$ and $\sigma(x) = \log(1 + e^{-x})$ do not satisfy this assumption. However, almost any horizontal translation of them is a generic activation: given a non-polynomial analytic activation $\sigma$, for almost all $\vep > 0$ the function $x \mapsto \sigma(x + \vep)$ satisfies Assumption \ref{generic activation}. \\

The motivation of defining such activations is that any set of neurons constructed from generic activations preserves the number of first-layer features, i.e., the weights. Thus, the neurons are ``good feature-maps" as they preserve the information from the input-layer. Mathematically, we will show that whenever $w_1, ..., w_r \in \bR^d$ are distinct, $\sigma(w_1\cdot x), ..., \sigma(w_r\cdot x)$ are linearly independent, for any $r \in \bN$. This will be proved in Proposition \ref{independent activation}, when $\sigma$ is a generic activation. We shall also see that Assumption \ref{generic activation} (that is, $\sigma$ is a generic activation) is a necessary condition for it. \\

Having made assumptions on our activation function, we turn to the set-up of the network training -- the model, target function, and loss. In this paper, we focus on training a two-layer neural network 
\[
    g:\bR^{(d+1)m} \times \bR^d \to \bR, \quad g(\theta, x) = \sum_{k=1}^m a_k \sigma(w_k\cdot x), 
\]
Here $m \in \bN$ is fixed, which is often called the width of $g$, $x \in \bR^d$ is the input of $g$, and the parameter $\theta$ of $g$ is in $\bR^{(d+1)m}$: for a parameter we have several notations
\begin{align*}
    \bR^{(d+1)m} \ni \theta &= (a_1, w_1, ..., a_m, w_m) = (a_k, w_k)_{k=1}^m \in \prod_{k=1}^m (\bR \times \bR^d); \\
    \theta^* &= (a_1^*, w_1^*, ..., a_m^*, w_m^*) = (a_k^*, w_k^*)_{k=1}^m; \\ 
    \theta_i^j &= ((a_1)_i^j, (w_1)_i^j, ..., (a_m)_i^j, (w_m)_i^j) = ((a_k)_i^j, (w_k)_i^j)_{k=1}^m. 
\end{align*}
Throughout this paper, $i$ and $j$ are arbitrary indices to distinguish points in parameter space $\bR^{(d+1)m}$. \\

Next, we make clear about the models and target functions we consider in the paper. Starting from an abstract sense, we define our model as a function $g: X\times \bR^d \to \bR$. Here $X$ is any topological space which we call \textit{the parameter space} of $g$, and $\theta \in X$ is called a \textit{parameter}, and any $x \in \bR^d$ ($d \in \bN$ given) is an \textit{input} of $g$. The collection of all such models is defined as $\calG := \{g: X\times \bR^d \to \bR\}$. A particularly interesting case is 
\[
    X = \coprod_{m=1}^\infty \bR^{(d+1)m} =: \coprod_{m=1}^\infty \{(a_k, w_k)_{k=1}^m: a_k\in \bR, w_k \in \bR^d\}
\]
endowed with the Euclidean topologies from each $\bR^{(d+1)m}$, and $g(\theta, x) = \sum_{k=1}^m a_k \sigma(w_k\cdot x)$ for $\theta \in \bR^{(d+1)m}$, $x \in \bR^d$. Then our $\calG$ is just the collection of all two-layer NNs (of finite width) with activation $\sigma$. Notice that $\bR^{(d+1)m}$ embeds naturally into $\bR^{(d+1)m'}$ for $m \le m'$. So we can further define 
\[
    \calG_m := \left\{g(\theta, x) = \sum_{k=1}^m a_k \sigma(w_k\cdot x): \theta \in \bR^{(d+1)m}\right\} \quad \forall\, m \in \bN,  
\]
and write $\calG_m \siq \calG_{m'}$ for $m \le m'$. \\

Once given $\calG$, we assume that our target function $f$ is simply an element in it. Notice that when $X = \coprod_{m=1}^\infty \bR^{(d+1)m}$ and $g(\theta,x)$ defined as above, $f$ is just a two-layer neural network with width $m_0 \in \bN$, for some $m_0 \in \bN$. This is a natural setting, as the universal approximation theorem holds (for many commonly seen $\sigma$) on any compact subset of $\bR^d$ [9, 10, 11]. It is also clear that for any $m \ge m_0$ we have 
\[
    f^* \in \calG_m \siq \calG. 
\]
In the theory of neural network, this can be interpreted as: an NN model with no fewer features than the target function can fit it perfectly. This is the setting we consider in this paper. 

\begin{asmp}[finite-feature setting]\label{finite-feature setting}
    Given a generic activation $\sigma$ and $m, m_0 \in \bN$ with $m \ge m_0$, we consider a target function 
    \[
        f^*(x) = \sum_{k=1}^{m_0} \bara_k \sigma(\barw_k\cdot x), 
    \]
    where each $\bara_k \in \bR\cut\{0\}$ and $\barw_k \in \bR^d \cut \{0\}$ and $\barw_k \neq \barw_j$ whenever $k \neq j$. The collection of two-layer NN models we consider is $\calG_m \siq \calG$. 
\end{asmp}

Now we define our loss function as the usual empirical $L^2$ loss: 
\[
    R: \bR^{(d+1)m} \to \bR, R(\theta) = \int_{\bR^d} |g(\theta, x) - f^*(x)|^2 d\mu(x), 
\]
the measure $\mu$ being a Borel measure on $\bR^d$. Throughout the paper we are interested in Dirac masses $\mu = \sum_{i=1}^n \delta(\cdot- x_i)$ for distinct $x_i \in \bR^d$, $1 \leq i \leq n$; the $x_i$'s and $(x_i, f^*(x_i))$'s are both called samples, we do not distinguish them. In this case we have 
\begin{equation}\label{Eq Definition of Rmu}
    R(\theta) = \sum_{i=1}^n \left| g(\theta, x_i) - f^*(x_i) \right|^2 = \sum_{i=1}^n \left| \sum_{k=1}^m a_k\sigma(w_k\cdot x_i) - \sum_{k=1}^{m_0} \bar{a}_k \sigma(\bar{w}_k\cdot x_i) \right|^2. 
\end{equation}

\begin{remark}
    Some remarks on the loss function $R$. 
    \begin{itemize}
        \item [(a)] As we mentioned above, $f^* \in \calG_{m_0}$, whence there is some $\theta \in \bR^{(d+1)m}$ such that $g(\theta,\cdot) = f^*$. We define the \textit{perfect global minima}
        \begin{equation}\label{Eq Q^*}
            Q^* := \{\theta \in \bR^{(d+1)m}: g(\theta, \cdot) = f^*\}. 
        \end{equation}
    Clearly $Q^* \siq \Rmuzero$. We may also say $\Rmuzero \cut Q^*$ \textit{imperfect global minima}.
        \item [(b)] When $\mu = \sum_{i=1}^n \delta(\cdot-x_i)$ with $n \leq (d+1)m$, we say the model (or more generally, the system) is \textit{overparametrized}, otherwise it is called \textit{underparametrized}. Similarly, when $\mu = \rho dx$ for some continuous function $\rho: \bR^d \to (0, +\infty)$, the model is underparametrized. 

        Unlike traditional machine learning, in many NN trainings the models are overparametrized; this is one difficulty in the analysis of them. 
    \end{itemize}
\end{remark}

Finally, given $\theta_0 = ((a_k)_0, (w_k)_0)_{k=1}^m \in \bR^{(d+1)m}$, the gradient flow $\gamma = (a_k, w_k)_{k=1}^m: [0, +\infty) \to \bR^{(d+1)m}$ with initial value $\theta_0$ is defined as 
\begin{align*}
    \dot{\gamma}(t) = (\dot{a}_k(t), \dot{w}_k(t))_{k=1}^m = -\nabla R(\gamma(t)), \quad \gamma(0) = \theta_0. 
\end{align*}
In particular, when $R$ has the form in equation (\ref{Eq Definition of Rmu}), we have for $k=1,\dots,m$
\begin{align*}
    \dot{a}_k(t) &= - \parf{R}{a_k}(\gamma(t)) = 2 \sum_{i=1}^n (g(\gamma(t), x_i) - f^*(x_i)) \sigma(w_k(t) \cdot x_i), \\ 
    \dot{w}_k(t) &= - \parf{R}{w_k}(\gamma(t)) = 2 a_k(t) \sum_{i=1}^n (g(\gamma(t), x_i) - f^*(x_i)) \sigma'(w_k(t)\cdot x_i) x_i. 
\end{align*}

\subsection{Local Recovery Problem}\label{Subsection Motivation}
In this paper, we propose the concept of the local recovery problem in the context of neural networks at overparametrization. This problem refers to the challenge of ensuring that a neural network can perfectly recover the target function (the learned function has zero generalization error), in a local region of its parameter space. The local recovery problem arises due to the complex and non-convex nature of the loss landscape in neural networks, where infinitely many global minima exist and form complex patterns. \\

The key aspects of the local recovery problem include:

\begin{itemize}
    \item [(a)] \textbf{Geometric structure:} As we investigate the recovery of our target function, we are particularly interested in the geometry of global minima and the relation between sample size and global minima geometry. 

    \item [(b)] \textbf{Separation of minima:} In overparameterized neural networks, the possibility that perfect global minima, i.e., those that accurately recover the target function, can become separated from other global minima. This ensures our model to recover target function. 

    \item [(c)] \textbf{Convergent gradient dynamics:} The limiting properties of gradient-based optimization methods near these minima is affected by the local recovery problem. Understanding the dynamics in the vicinity of perfect global minima is essential for efficient training.
\end{itemize}

By proposing the local recovery problem, we aim to highlight how these geometric and dynamic properties impact the network's ability to find these perfectly generalizing solutions throughout training process. Addressing this problem involves developing strategies to navigate the complex loss landscape and understand why the network can recover the target function effectively in local regions.\\

To illustrate the local recovery problem more concretely, we will use a specific example. Consider a two-neuron model with exponential activation function and with one-dimensional input, i.e., $g(\theta, x) = a_1 e^{w_1\cdot x} + a_2 e^{w_2\cdot x}$, where $x \in \bR^2$ and $\theta=(a_1,w_1,a_2,w_2)\in\bR^6$. Consider the training data $\{(x_i, y_i)\}_{i=1}^n$ and hence the loss 
\begin{align*}
    R(\theta) = \sum_{i=1}^n \left| g(\theta, x_i) - y_i \right|^2. 
\end{align*}
By Cooper's results \cite{YCooper}, up to an arbitrarily small perturbation of the $y_i$'s, $\Rmuzero$ is a submanifold of $\bR^6$ with dimension $\max\{0, 6 - n\}$. In our paper, by considering the cases where $y_i$'s are sampled from a target function $f^*$ expressible by the given two-layer NN, which allows for perfect generalization, we uncover more detailed structure of the global minima of $R$. For illustration, we consider a simple example with $f^*(x) = \bara e^{\barw\cdot x}$ ($\bara \ne 0$). The loss function writes 
\begin{align*}
    R(\theta) = \sum_{i=1}^n \left| g(\theta, x_i) - \bara e^{\barw \cdot x_i} \right|^2. 
\end{align*}
First, note that $e^{w_1\cdot x}$ and $e^{w_2\cdot x}$ are linearly independent if and only if $w_1 \ne w_2$. Thus, based on the number of distinct $w_k$'s, we have a partition of the perfect global minima as $Q^* = Q_1 \cup Q_2 \cup Q_3$ independent of the training inputs, where 
\begin{align*}
    Q_1 &= \{(a, \barw, \bara-a, \barw): a \in \bR\} \\ 
    Q_2 &= \{(\bara, \barw, 0, w): w \in \bR^2 \cut \{\barw\}\} \\
    Q_3 &= \{(0, w, \bara, \barw): w \in \bR^2 \cut\{\barw\}\}. 
\end{align*}
Geometrically, $Q_1, Q_2$ and $Q_3$ look like three ``branches" with different dimensions. See also the figure below. 
\begin{figure}[H]
    \centering
    \includegraphics[width = 0.675\textwidth]{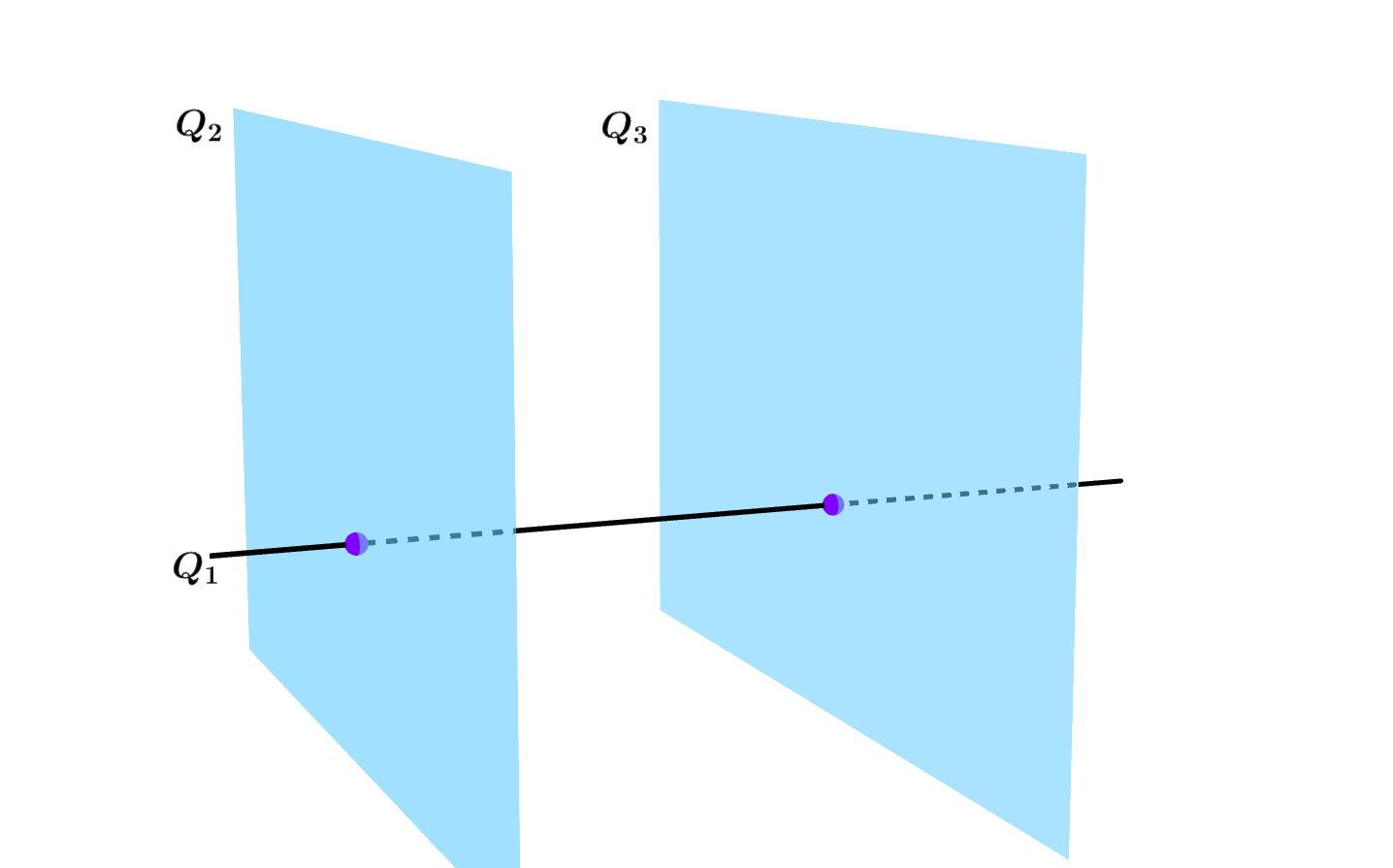}
    \caption{Illustration of $Q^*$. The closure of the branches $Q_1, Q_2, Q_3$ are all affine subspaces in the parameter space. Moreover, $Q_1, Q_2$ intersects at $(\bara, \barw, 0, \barw)$ and $Q_1, Q_3$ intersects at $(0, \barw, \bara, \barw)$.}
    \label{Figure Illust of Q*}
\end{figure}

Observe that $\overline{Q_1}, \overline{Q_2}, \overline{Q_3}$ are all affine subspaces in $\bR^6$ with different dimensions. In general, when sample size $n\geq 6$, we expect no imperfect global minima, thus the NN achieves perfect generalization when it converges to zero loss. However, at overparameterization, i.e., when sample size $n<6$, imperfect global minimum generally exists, which sparkles the question of whether each branch of perfect global minima are enclosed by the imperfect global minima. If not, the training dynamics clearly has chance to converge to some perfect global minima, thus achieving perfect generalization.  
% To study the generalization in such cases, We ask how likely are gradient flows to converge to different branches of $Q^*$, thus achieving perfect generalization. 
Inspired by this example, we notice that in general, to understand the problem of achieving perfect generalization for a two-layer neural network ($m$, $m_0$ and $n$ are arbitrary), we must investigate the following questions: 
\begin{itemize}
    \item [(a)] \textbf{Geometric structure:} How can we describe the perfect global minima $Q^*$ geometrically? Does it consist of branches as for the example above? 
    \item [(b)] \textbf{Separation of minima:} How is $Q^*$ related to $\Rmuzero$, in particular to the imperfect global minima $\Rmuzero \cut Q^*$? Can it be ``separated" from $\Rmuzero \cut Q^*$? How does this depend on samples? 
    \item [(c)] \textbf{Convergent Gradient dynamics:} What are the convergence, convergence rate, convergence direction, etc., of gradient flows near $\Rmuzero$? Moreover, what can we say about the stability of these gradient flows? 
\end{itemize}
In \cite{BSimsek, KFukumizu}, Simsek and Fukumizu already have an answer to question (a), i.e., the geometry of $Q^*$. In the papers they show that $Q^*$ is a set with lots of symmetry. Moreover, it is the union of finitely many branches of different dimensions, the closure of each branch being an affine subspace. See also Section \ref{Subsection geometry of Minfty}. Based on the geometry of $Q^*$, we provide detailed answers to questions (b) and (c) in Sections \ref{Section Loss landscape near Minfty} and \ref{Section Dynamics of gradient flow near Minfty}. A summary of the results can be found in the following Section \ref{Subsection Main results}, where the informal theorems (Theorem \ref{Thm Hess Rmu, informal} to question (b) and Theorem \ref{Thm gradient flow properties local, informal} to question (c)) are presented. Then we apply these theorems to this example. 

\subsection{Main Results}\label{Subsection Main results}

The main results of this paper are theorems resolving the local recovery of two-layer neural networks. Let us summarize and discuss them informally as follows. For the separation of branches of $Q^*$ we have the following. 

\vspace{0.5em}

\begin{thm}[separation of branches in $Q^*$]\label{Thm Hess Rmu, informal}
    Let $\{Q_t\}_{t=1}^N$ be the branches of $Q^*$. Each branch $Q_t$ corresponds to a sample size threshold $N_t \le m(d+1)$ (and if $m > m_0$, we have $N_t < m(d+1)$), such that when sample size $n \ge N_t$, $Q_t$ is ``separated" from the imperfect global minima. Moreover, by rearranging the indices of $Q_t$ if necessary, there is a partition 
    \[
        Q^* := \bigcup_{t=1}^N Q_t = \left(\bigcup_{t=1}^{N'} Q_t\right) \bigcup \left(\bigcup_{t=N'+1}^{N} Q_t\right)
    \]
    such that whenever $t \le N'$ and $n \ge N_t$, $R$ is not Morse--Bott anywhere at $Q_t$, while for $t > N'$ and $n \ge N_t$, $R$ is Morse--Bott a.e. at $Q_t$. 
\end{thm}

By saying that $Q_t$ is separated we mean there is an open $U \siq \bR^{(d+1)m}$ such that $U \cap \Rmuzero = U \cap Q_t$. For definition of a Morse--Bott function $f$ we mean the $\mathrm{Hess}\,f$ is non-degenerate along the normal bundle of a manifold in $(\nabla f)^{-1}\{0\}$ (Definition \ref{Morse--Bott function defn}). The details of Theorem \ref{Thm Hess Rmu, informal} will be shown in Lemmas \ref{separated branch} and \ref{underparametrized system} in Section \ref{Section Loss landscape near Minfty}, which relies on the theory of real analytic functions contained in Section \ref{Subsection Theory of real analytic functions}. \\

Finally, any gradient flow near $\Rmuzero$ has the following properties. 

\begin{thm}[gradient flow near global minima]\label{Thm gradient flow properties local, informal}
    Following the hypotheses and notations in Theorems \ref{Thm Hess Rmu, informal}, any gradient flow sufficiently close to $\Rmuzero$ converges. On the other hand, any point in $\Rmuzero$ is the limit of some gradient flow. The following results hold. 
    \begin{itemize}
        \item [(a)] Whenever $t \le N'$ and sample size $n \ge N_t$, a generic gradient flow sufficiently close to $Q_t$ converges to a point $\theta^* \in Q_t$. The convergence does not have linear rate and the curve is ``biased towards" $\ker \mathrm{Hess} R(\theta^*)$. Moreover, any small perturbation of it still converges to $Q_t$. 

        \item [(b)] Given $t > N'$. When sample size $n \ge N_t$, any gradient flow sufficiently close to $Q_t$ converges to points $Q_t$ at linear rate. Similar to (a), any small perturbation of it still converges to $Q_t$. 
    \end{itemize} 
\end{thm}

In short, we characterize the convergence (Theorem \ref{convergence rate of analytic or GMB function}), limiting set (Proposition \ref{Prop Converse of gradient flow convergence}), limiting direction and convergence rate of gradient flows near global minimum, especially near $Q^*$ (Theorem \ref{Thm Convergence rates of gradient flow}). Meanwhile we develop a concept called ``generalization stability" which discusses how the limiting model (under gradient flow) changes at perturbation of gradient flow (see Definition \ref{Defn Recovery stability} and the remark below). \\

The theorems above exhibit comprehensively the structural change of geometry and local dynamics of global minima at the overparameterized regime, with the perfect global minima  $Q^*$ as its backbone. In particular, the branch separation and convergence results guarantees the local recovery capability of two-layer neural networks, i.e., the target function will be recovered when initialized near separated branches.  Furthermore, our results suggest the following mechanism of generalization at overparameteration that deserve further study: with proper generic initialization and hyperparameter tuning, the gradient dynamics can be globally guided to a neighbourhood of the separated branches of $Q^*$, thus recovering the target function at convergence.  In the following, we illustrate above results in a simple example.  \\

\textbf{Example (Continued). } Using Theorems \ref{Thm Hess Rmu, informal} and \ref{Thm gradient flow properties local, informal}, together with some calculation, we now answer questions about local recovery theorem for our two-neuron model example.  
\begin{itemize}
    \item [(a)] From \cite{BSimsek, KFukumizu}, $Q^*$ is a union of several subsets of $\bR^4$ whose closures are affine subspaces. Specifically, $Q_1$ is a one-dimensional affine subspace, while $Q_2, Q_3$ are both two-dimensional affine subspaces minus a point. This coincides with our observation and Figure \ref{Figure Illust of Q*} above. 

    \item [(b)] By Theorem \ref{Thm Hess Rmu, informal}, each $Q_t$, $1 \le t \le 3$, corresponds to a sample size $N_t$ making it separated. Specifically, by Lemma \ref{separated branch} we have $N_1 = 5$ and $N_2 = N_3 = 4$. Moreover, when sample size $n \ge N_t$, $R$ is not Morse--Bott anywhere at $Q_t$ if $t = 1$, and $R$ is Morse--Bott a.e. at $Q_t$ if $t = 2,3$. 

    \item [(c)] By Theorem \ref{Thm gradient flow properties local, informal}, any gradient flow sufficiently near $\Rmuzero$ converges to it, and any point in $\Rmuzero$ is the limit of some gradient flow. According to this theorem and (b), a generic gradient flow $\gamma$ does not converge to a point $\theta^* \in Q_1$ at linear rate and is biased towards $\ker \mathrm{Hess} R(\theta^*)$, when sample size $n \ge 5$. However, any gradient flow $\gamma$ converging to a point in $Q_2 \cup Q_3$ has linear rate, when $n \ge 4$. 
\end{itemize}

The following figure illustrates these properties of $Q^*$ and gradient flow nearby.

\begin{figure}[H]
    \centering
    \includegraphics[trim={0 2cm 0 2cm}, clip, width=0.7\textwidth]{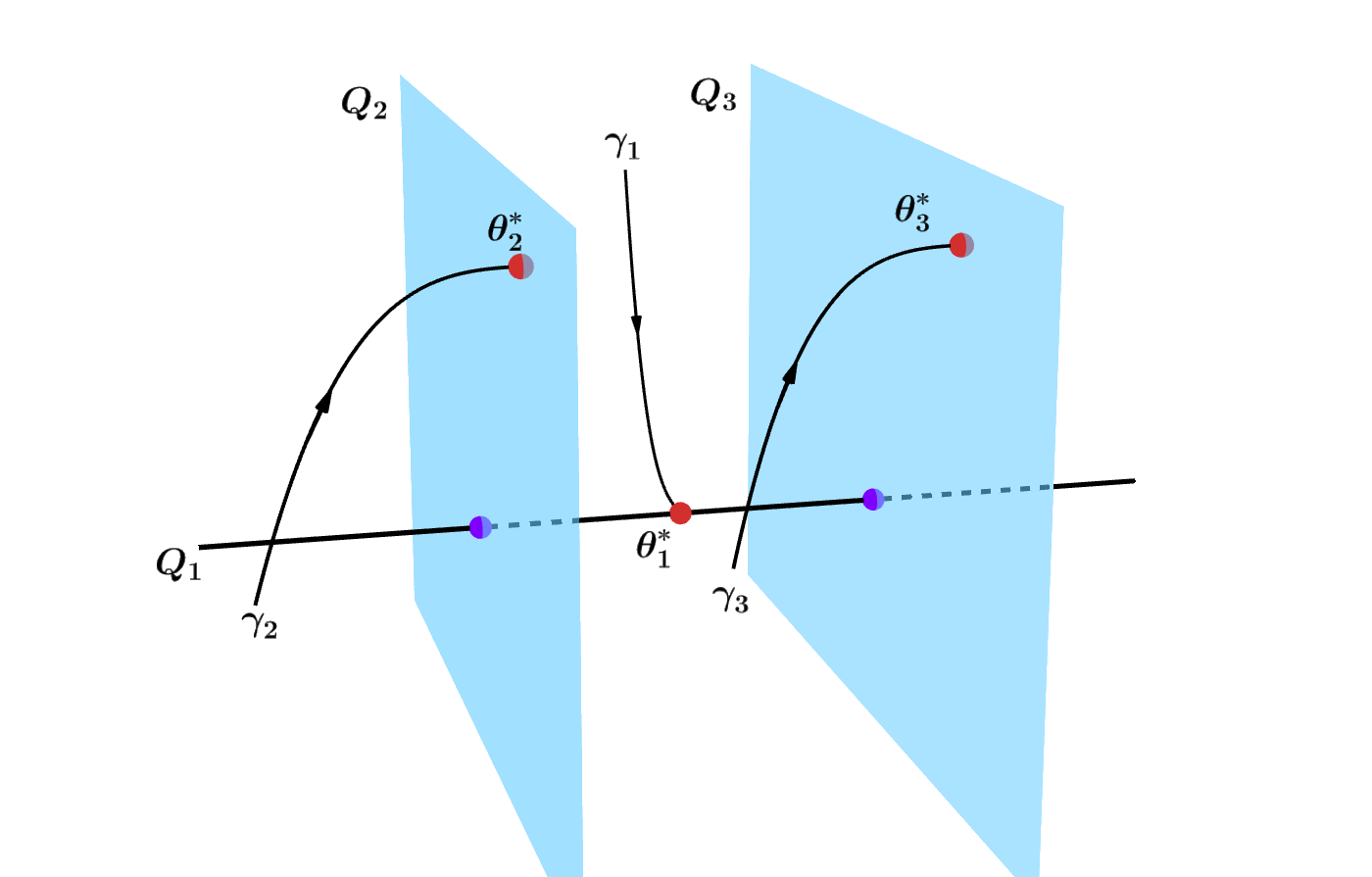}
    \caption{Illustration of the example for two-neuron model fitting a one-neuron network. As shown in part (a) of example, $Q^*$ consists of three sets whose closures are one-dimensional affine subspaces. By (b), the loss $R$ is not Morse--Bott near any point in $Q_1$, whence by (c) a gradient flow with limit in $Q_1$ ($\theta_1^*$ in the figure) is in general ``biased towards" $\ker\,\mathrm{Hess}\, R(\theta_1^*)$. On the other hand, $R$ is Morse--Bott a.e. at $Q_2$ and $Q_3$, whence a gradient flow with limit in $Q_2 \cup Q_3$ ($\theta_2^*, \theta_3^*$ in the figure) in general converges at linear rate. Finally, note that $Q_{12} = (\bara, \barw, 0, \barw)$ and $Q_{13} = (0, \barw, \bara, \barw)$ are the points of intersections $\overline{Q_1} \cap \overline{Q_2}$ and $\overline{Q_1} \cap \overline{Q_3}$, respectively.}
    %\label{fig:enter-label}
\end{figure}

\section{Preparing Lemmas and Propositions}\label{Section Preparing Lemmas and Propositions}

To prove the main results we shall do some preparation in this section. We will first show that, as aforementioned, the linear independence of neurons with a generic activation. Then we introduce separating inputs (Definition \ref{Defn Separating inputs}) based on how the choice of samples affect the rank of certain matrices. Then we present some basic results about the zero set of real analytic functions, some of which will be used in Section \ref{Subsection Separating inputs are almost everwhere}. After that, we summarize and rephrase the results about geometry of $Q^*$ given by \cite{BSimsek, KFukumizu}. \\

Let's start with a lemma about power series. 

\begin{lemma}[characterization of $\vep$-polynomial]\label{Lem Lin comb of stretched sigma and their derivatives}
    Let $\sum_{j=0}^\infty c_j \vep^j$ be a power series of real or complex coefficients $\{c_j\}_{j=0}^\infty$ such that for any $N \in \bN$ there are some odd number $j_{\mathrm{odd}} > N$ and even number $j_{\mathrm{even}} > N$ with $c_{j_{\mathrm{odd}}} \neq 0, c_{j_{\mathrm{even}}} \neq 0$. Then for any $m, l\in \bN$ and any distinct $p_1, ..., p_r \in \bR\cut\{0\}$, the power series in $\vep$:
    \[
        \sum_{j\geq l} c_j \sum_{k=1}^r \left[ \alpha_{0k} + j \alpha_{1k} + j(j-1)\alpha_{2k} + ... + \frac{j!}{(j-l)!} \alpha_{lk} \right] p_k^{j-l} \vep^{j-l}
    \]
    is a polynomial if and only if $\alpha_{tk} = 0$ for all $0 \le t \le l$ and $1 \le k \le r$. 
\end{lemma}
\begin{proof}
    First note that for any sufficiently small $\vep$ the power series in question converges. Without loss of generality, assume that $|p_1| \geq ... \geq |p_r|$; in particular, $p_1$ has the largest absolute value among the $|p_k|$'s. We consider two cases. 
    \begin{itemize}
        \item [(a)] $|p_1| > |p_2|$. Assume that the $\alpha_{t1}$'s are not all zero; otherwise we work on a smaller $r$. Thus, there is a largest $t_1 \in \{0, ..., l\}$ with $\alpha_{t_1 1} \neq 0$. Then 
        \[
            \sum_{k=1}^r \left[ \alpha_{0k} + j \alpha_{1k} + ... + \frac{j!}{(j-l)!} \alpha_{lk} \right] p_k^{j-l} \sim p_1^{j-l} 
        \]
        as $j \to \infty$. By hypothesis, there is a subsequence $\{c_{j_s}\}_{s=1}^\infty \siq \bR(\bC) \cut \{0\}$ of $\{c_j\}_{j=1}^\infty$. Therefore, 
        \[
            c_{j_s} \sum_{k=1}^r \left[ \alpha_{0k} + j_s \alpha_{1k} + ... + \frac{j_s!}{(j_s - l)!} \alpha_{lk} \right] p_k^{j_s-l} \sim p_1^{j_s-l}
        \]
        as $s \to \infty$, which shows that the power series has infinitely many non-zero coefficients, whence not a polynomial. 

        \item [(b)] $|p_1| = |p_2|$. Because the $p_k$'s are distinct, we must have $p_1 + p_2 = 0$ and $|p_2| > |p_3|$. Assume that $t_1 \in \{0, ..., l\}$ is the largest number such that $\alpha_{t_1 1}$ or $\alpha_{t_1 2}$ is non-zero. Then, similar as (a) above, we have 
        \[
            \sum_{k=1}^r \left[ \alpha_{0k} + j \alpha_{1k} + ... + \frac{j!}{(j-l)!} \alpha_{lk} \right] p_k^{j-l} \sim \left(\alpha_{t_1 1} + (-1)^{j-l} \alpha_{t_1 2}\right) p_1^{j-l} 
        \]
        as $j \to +\infty$, provided that $\alpha_{t_1 1} + (-1)^{j-l}\alpha_{t_1 2} \ne 0$. By hypothesis, either i) $\alpha_{t_1 1} = 0$ or $\alpha_{t_2 1} = 0$, or both are non-zero, there is an odd or even sequence $\{j_s\} \siq \bN$ such that $\alpha_{t_1 1} + (-1)^{j_s - l}\alpha_{t_1 2} \ne 0$ and $c_{j_s} \ne 0$ for all $s \in \bN$. It follows that 
        \[
            c_{j_s} \sum_{k=1}^r \left[ \alpha_{0k} + j_s \alpha_{1k} + ... + \frac{j_s!}{(j_s - l)!} \alpha_{lk} \right] p_k^{j_s-1} \sim p_1^{j_s-1}
        \]
        as $s \to +\infty$, so the power series has infinitely many non-zero coefficients, whence not a polynomial. 
    \end{itemize}
    In either case, we have shown that $\alpha_{01} = \alpha_{11} = ... = \alpha_{l1} = 0$ must hold if $h$ is a polynomial. By repeating this procedure for $r$ times we can see that $\alpha_{tk} = 0$ for all $0 \le t \le l$ and $1 \le k \le r$. 
\end{proof}

\subsection{Linear Independence of Neurons}\label{Subsection Linear independence of neurons}

\begin{cor}[linear independence of neurons] \label{independent activation}
    Let $d$ be any positive integer. Given a real analytic function $\sigma: \bR \to \bR$, the following two statements about $\sigma$ are equivalent. 
    \begin{itemize}
        \item [(a)] $\sigma$ is a generic activation, namely, it satisfies Assumption \ref{generic activation}. 
        
        \item [(b)] For any $r > 0$ and any distinct vectors $w_1, ..., w_r \in \bR^d$, the functions $\sigma(w_1 \cdot x), ..., \sigma(w_r \cdot x)$ are linearly independent. 
    \end{itemize}
\end{cor}
\begin{proof}
First suppose that (a) holds. Let $1 \leq k < j \leq r$. The set 
\begin{equation}
    A_{k,j} = \{x \in \bR^d: \<x, w_k - w_j\> = 0\}
\end{equation} 
is a subspace of dimension $d-1$, whence $\bigcup_{1\leq k < j \leq r} A_{k,j}$ has $\lambda_d$-measure zero. This, together with linearity, implies that we can find some $e \in \partial B(0,1) \siq \bR^d$ with $p_k := \<w_k, e\> \neq \<w_j, e\> =: p_j$ whenever $k\neq j$. For any $|\vep| < (\max_j |p_j|)^{-1} R$ and any $k$ we have 
\begin{equation*}
    \sigma(w_k\cdot \vep e) = \sum_{j=0}^\infty c_j \sigma(w_k \cdot \vep e) = \sum_{j=0}^\infty (c_j p_k^j) \vep^j. 
\end{equation*} 
Now let $\alpha_1, ..., \alpha_r$ be constants such that $x \mapsto \sum_{k=1}^r \alpha_k \sigma(w_k \cdot x)$ is a zero map. It follows that for $|\vep| < (\max_j |p_j|)^{-1} R$, 
\begin{equation}\label{eq 1 for neuron linear independence}
\begin{aligned}
    \sum_{k=1}^r \alpha_k \sigma(w_k \cdot \vep e) 
    &= \sum_{k=1}^r \alpha_k \sum_{j=0}^\infty (c_j p_k^j) \vep^j \\ 
    &= \sum_{j=0}^\infty c_j \left( \sum_{k=1}^r \alpha_k p_k^j \right) \vep^j \\
    &= \sum_{j=1}^\infty c_{j-1} \left( \sum_{k=1}^r \alpha_k p_k^{j-1} \right) \vep^{j-1} = 0. 
\end{aligned}
\end{equation} 
Since $\sigma$ is a generic activation, the sequence $\{c_{j-1}\}_{j=1}^\infty$ satisfies the hypothesis of Lemma \ref{Lem Lin comb of stretched sigma and their derivatives}. Since the $p_k$'s are also distinct, by Lemma \ref{Lem Lin comb of stretched sigma and their derivatives} we have $\alpha_1 = ... = \alpha_r = 0$. Therefore, $\sigma(w_1 \cdot x), ..., \sigma(w_r \cdot x)$ are linearly independent. 

Conversely, assume that (a) does not hold. First suppose there is some $N \in \bN$ such that for any odd $j > N$ we have $c_j = 0$. Then $\sigma$ is the sum of a polynomial and an even function. Let $\tilde{\sigma}: \bR \to \bR$ be defined by $\tilde{\sigma}(x) = \sigma(x) - \sigma(-x)$, so $\tilde{\sigma}$ is a polynomial of degree at most $N$. Thus, the dimension of $\Spans{\tilde{\sigma}(\tw_1\cdot x), ..., \tilde{\sigma}(\tw_{r'}\cdot x)}$ is bounded by $(N+1)^d$. This implies that for any $r > 2(N+1)^d$ the functions 
\begin{equation*}
    \sigma(w_1 \cdot x), \sigma(-w_1 \cdot x), ..., \sigma(w_r\cdot x), \sigma(-w_r \cdot x) 
\end{equation*} 
can never be linearly independent. Similarly, if there is some $N \in \bN$ such that for any even $j > N$ we have $c_j = 0$. Then $\sigma$ is the sum of a polynomial and an odd function. Let $\tilde{\sigma}: \bR \to \bR$ be defined by $\tilde{\sigma}(x) = \sigma(x) + \sigma(-x)$, so $\tilde{\sigma}$ is a polynomial. Thus, for sufficiently large $r$ the functions 
\begin{equation*}
    \sigma(w_1 \cdot x), \sigma(-w_1 \cdot x), ..., \sigma(w_r \cdot x), \sigma(-w_r \cdot x)
\end{equation*} 
can never be linearly independent. 
\end{proof}
\begin{remark}
    In fact, the proof of Corollary \ref{independent activation} shows that if $\sigma$ is analytic and not a polynomial, $\sigma(w_1\cdot x), ..., \sigma(w_r\cdot x)$ are linearly independent for any $r \in \bN$, whenever the distinct $w_1, ..., w_r \in \bR^d$ satisfy $w_k + w_j \neq 0$ for all $1 \leq k,j \leq r$. Similarly, Corollary \ref{Cor separating sample II} hold under similar requirements. In particular, these results hold for $\sigma(x) = \frac{1}{1+e^{-x}}$, $\sigma(x) = \tanh(x)$ or $\sigma(x) = \log(1 + e^{-x})$. Interestingly, we also observe that current analysis of loss landscapes of neural network focus on polynomial and non-polynomial activations separately. For example, Venturi \cite{LVenturi} has shown that for sufficiently wide one-hidden-layer neural network with polynomial activation, the corresponding loss landscape has no spurious valley, while in \cite{BSimsek} the analysis of critical points are valid only for neural networks with certain non-polynomial activations. 
\end{remark}

Corollary \ref{independent activation} proves the (linear) neuron independence of analytic neurons, which is the main object we concern in this paper. For completeness, we also present a version of neuron independence result without requiring the neurons to be analytic. Instead, it considers other important properties of an activation function, which we hope could be of its own interest. More precisely, we make the following assumption. 

\begin{asmp}\label{generic activation - aleternative}
    Assume that $\sigma: \bR \to \bR$ is an $s$-times continuously differentiable function ($s \geq 0$) with the following properties:
    \begin{itemize}
        \item [(a)] (rapid decreasing) $\frac{|\sigma^{(s)}(y)|}{|\sigma^{(s)}(x)|} \to +\infty$ as $|x|, |y| \to +\infty$ and $|x| - |y| \to +\infty$. 
        \item [(b)] (non-asymptotic symmetry) There is some $c>1$ such that either
        \[
            \varliminf_{x\to +\infty} \frac{|\sigma^{(s)}(x)|}{|\sigma^{(s)}(-x)|} \geq c
        \]
        or 
        \[
            \varliminf_{x\to-\infty} \frac{|\sigma^{(s)}(x)|}{|\sigma^{(s)}(-x)|} \geq c.
        \]
    \end{itemize}
\end{asmp}
\begin{remark}
    Notice that the rapid decreasing property of $\sigma^{(s)}$ requires that $\sigma^{(s)}(x) \neq 0$ for $|x|$ sufficiently large. Moreover, this property implies that $\lim_{x\to +\infty} \sigma^{(s)}(x) = 0$. To see this, let $\{x_n\}_{n=1}^\infty$ be a sequence with $\limftyn |x_n| = +\infty$. By passing to a subsequence of $\{x_n\}_{n=1}^\infty$, we can assume that $\limftyn (|x_{n+1}| - |x_n|) = +\infty$ as well. Given $A > 1$, any sufficiently large $N \in \bN$ gives 
    \[
        \frac{|\sigma^{(s)}(x_N)|}{|\sigma^{(s)}(x_{N+1})|} \geq A.
    \]
    Thus, 
    \[
        |\sigma^{(s)}(x_N)| \geq A |\sigma^{(s)}(x_{N+1})| \geq ... \geq A^n |\sigma^{(s)}(x_{N+n})|. 
    \]
    Since $A > 1$, $\limftyn A^n = +\infty$, which shows that $\limftyn \sigma^{(s)}(x_n) = \limftyn \sigma^{(s)}(x_{N+n}) = 0$. \\ 
    
    We shall give some examples to illustrate the two properties in Assumption \ref{generic activation - aleternative}. For the rapid decreasing property we have the following examples, which takes the commonly-seen activations into consideration.
    \begin{itemize}
        \item [(a)] $\sigma(x) = \exp(-x^2)$ with $s = 0$. Note that for any $x, y \in \bR$, $\sigma(y)/\sigma(x) = \exp(x^2 - y^2)$, which obviously tends to infinity as $|x| - |y|$ tends to infinity.
    
        \item [(b)]  $\sigma(x) = \exp(-|x|)$ with $s = 0$. For any $x, y \in \bR$, $\sigma(y)/\sigma(x) = \exp(|x| - |y|)$, which obviously tends to infinity as $|x| - |y|$ tends to infinity.
    
        \item [(c)] $\sigma(x) = \frac{1}{1 + e^{-x}}$ with $s = 1$. Recall that $\sigma'(x) = \frac{e^{-x}}{(1 + e^{-x})^2}$. Thus, In other words, $\lim_{x\to\pm\infty} \frac{\sigma'(x)}{e^{-|x|}} = 1$. Then by (b) we see that for any $|x|, |y|$ sufficiently large, $\sigma'(y)/\sigma'(x) \approx \exp(|x| - |y|)$, which also tends to infinity as $|x| - |y|$ tends to infinity.
    
        \item [(d)] $\sigma(x) = \log (1 + e^x)$ with $s = 2$. This is because $\sigma'(x) = \frac{1}{1 + e^{-x}}$. Similarly, $\sigma(x) = \tanh(x)$ also decreases rapidly.
    \end{itemize}
    It is easy to see that in any case above, $\sigma^{(s)}$ does not satisfy the non-asymptotic symmetry property, because $\sigma^{(s)}$ is an even function. However, in any case a horizontal shift of $\sigma^{(s)}$ by $b \in \bR$, $x \mapsto \sigma^{(s)}(x - b)$, which is just the $s$-th derivative of $\sigma(\cdot -b)$, satisfies both the rapid decreasing property and the non-asymptotic symmetry property. \\
\end{remark}

\begin{prop}\label{independent activation - alternative}
    Let $d$ be a positive integer. Let $\sigma$ satisfy Assumption \ref{generic activation - aleternative}. For any $r > 0$ and any distinct vectors $w_1, ..., w_r \in \bR^d\cut\{0\}$, the functions $\sigma(w_1\cdot x), ..., \sigma(w_r\cdot x)$ are linearly independent. 
\end{prop}
\begin{proof}
    We prove in a similar way as that of Corollary \ref{independent activation}. As we have shown above, there is some $e \in \partial B(0,1) \siq \bR^d$ with $p_k := \<w_k, e\> \neq \<w_j, e\> =: p_j \neq 0$ whenever $k, j \in \{1, ..., r\}$ are different. By rearranging the indices if necessary, we can assume that $|p_1| \geq |p_2| \geq ... \geq |p_r| > 0$. 
    
    Let $\alpha_1, ..., \alpha_r$ be constants such that $x \mapsto \sum_{k=1}^r \alpha_k \sigma(w_k\cdot x)$ is a zero map. In particular, this means 
    \[
        g: \bR \to \bR, g(\vep) = \sum_{k=1}^r \alpha_k \sigma((w_k\cdot e)\vep) = \sum_{k=1}^r \alpha_k \sigma(p_k \vep)
    \] 
    is a zero map. Then the $s$-th derivative (when $s = 0$, it is just $g$ itself) of $g$ in $\vep$ is given by 
    \[
        g^{(s)}(\vep) = \sum_{k=1}^r \alpha_k p_k^s \sigma^{(s)}(p_k \vep) = 0.
    \]
    We start by showing that $\alpha_1 = 0$. When $r = 1$ this clearly holds. When $r \geq 2$, we consider two cases. 
    \begin{itemize}
        \item [i)]  $|p_1| > |p_k|$ for all $2 \leq k \leq r$. Then $p_1^s \neq 0$. For $\sigma(p_1 \vep) \neq 0$ (which holds when $|\vep|$ is sufficiently large), we can rewrite $g^{(s)}(\vep)$ as 
        \[
            \alpha_1 = - \sum_{k=2}^r \alpha_k \left(\frac{p_k}{p_1}\right)^s \frac{\sigma^{(s)}(p_k\vep)}{\sigma^{(s)}(p_1\vep)}. 
        \]
        For any $2 \leq k \leq r$, $|p_1| > |p_k|$, so $|p_1\vep|, |p_k\vep|$ and $|p_1\vep| - |p_k\vep|$ all tend to infinity as $\vep \to \pm\infty$. Therefore, using the rapid decreasing property of $\sigma^{(s)}$ we have 
        \[
            \lim_{\vep \to \pm\infty} \frac{\sigma^{(s)}(p_k\vep)}{\sigma^{(s)}(p_1\vep)}. 
        \]
        In particular, it follows that as $\vep \to +\infty$, 
        \[
            \alpha_1 = \lim_{\vep \to +\infty} \alpha_1 = - \sum_{k=2}^r \alpha_k \left(\frac{p_k}{p_1}\right)^s \lim_{\vep \to +\infty} \frac{\sigma^{(s)}(p_k\vep)}{\sigma^{(s)}(p_1\vep)} = 0. 
        \]
        \item [ii)] $|p_1| = |p_2|$. Then we must have $p_1 = - p_2 \neq 0$ and $|p_1| = |p_2| > |p_k|$ for all $k \neq 1,2$. Again, for $\sigma(p_1 \vep) \neq 0$, we can rewrite $g^{(s)}(\vep)$ as 
        \begin{align*}
            \alpha_1 
            &= -\alpha_2 \left( \frac{p_2}{p_1} \right)^s \frac{\sigma^{(s)}(p_2\vep)}{\sigma^{(s)}(p_1\vep)} - \sum_{k>2} \alpha_k \left(\frac{p_k}{p_1}\right)^s \frac{\sigma^{(s)}(p_k\vep)}{\sigma^{(s)}(p_1\vep)} \\ 
            &= (-1)^{s+1} \alpha_2 \frac{\sigma^{(s)}(-p_1\vep)}{\sigma^{(s)}(p_1\vep)} - \sum_{k>2} \alpha_k \left(\frac{p_k}{p_1}\right)^s \frac{\sigma^{(s)}(p_k\vep)}{\sigma^{(s)}(p_1\vep)}, 
        \end{align*}
        where we use $p_1 = -p_2$. Same as in i), letting $\vep \to \pm\infty$, we still have 
        \[
            \lim_{\vep \to \pm\infty} \sum_{k>2} \alpha_k \left(\frac{p_k}{p_1}\right)^s \frac{\sigma^{(s)}(p_k\vep)}{\sigma^{(s)}(p_1\vep)} = 0. 
        \]
        For the other term, we use the non-asymptotic symmetry property of $\sigma^{(s)}$. Without loss of generality, assume that $\varliminf_{x\to +\infty} \frac{|\sigma^{(s)}(-p_1\vep)|}{|\sigma^{(s)}(p_1\vep)|} \geq c$ for some $c > 1$. Then $\varlimsup_{x\to-\infty} \frac{|\sigma^{(s)}(-p_1\vep)|}{|\sigma^{(s)}(p_1\vep)|} \leq c^{-1}$. These two inequalities yield lower and upper bounds for $\alpha_1$: 
        \begin{align*}
            |\alpha_1| 
            &= \left| \varliminf_{x\to +\infty} (-1)^{s+1} \alpha_2 \frac{\sigma^{(s)}(-p_1\vep)}{\sigma^{(s)}(p_1\vep)} - \varliminf_{x\to +\infty} \sum_{k>2} \alpha_k \left(\frac{p_k}{p_1}\right)^s \frac{\sigma^{(s)}(p_k\vep)}{\sigma^{(s)}(p_1\vep)} \right| \\ 
            &= |\alpha_2| \varliminf_{x\to +\infty} \frac{|\sigma^{(s)}(-p_1\vep)|}{|\sigma^{(s)}(p_1\vep)|} \\ 
            &\geq c|\alpha_2|
        \end{align*}
        and similarly, 
        \[
            |\alpha_1| = \left| \varlimsup_{x\to-\infty} (-1)^{s+1} \alpha_2 \frac{\sigma^{(s)}(-p_1\vep)}{\sigma^{(s)}(p_1\vep)} - \varlimsup_{x\to-\infty} \sum_{k>2} \alpha_k \left(\frac{p_k}{p_1}\right)^s \frac{\sigma^{(s)}(p_k\vep)}{\sigma^{(s)}(p_1\vep)} \right| \leq c^{-1}|\alpha_2|. \\ 
        \]
        Unless $\alpha_1 = \alpha_2 = 0$, we would get a contradiction. 
    \end{itemize}
    In either case, we show that $\alpha_1 = 0$. By repeating this argument (at most) $r$ times, we can show that $\alpha_1 = ... = \alpha_r = 0$ and thus the linear independence of $\sigma(w_1\cdot x), ..., \sigma(w_r\cdot x)$ follows. 
\end{proof}

Note that the non-asymptotic symmetry property of $\sigma^{(s)}$ is only used to deal with $p_k = -p_j$ for some $k, j\in \{1, ..., r\}$ in the proof, which is unavoidable only when $w_k + w_j = 0$ for some $w_k, w_j \neq 0$. Thus, if $w_1, ..., w_r \in \bR^d\cut\{0\}$ are distinct vectors satisfying $w_k + w_j \neq 0$ for all $1 \leq k,j \leq r$, any $\sigma \in C^s$ which has the rapid decreasing property would give the linear independence result. By our examples above, this holds for lots of commonly seen neurons, including $\sigma(x) = \frac{1}{1+e^{-x}}$, $\sigma(x) = \tanh(x)$, $\sigma(x) = \exp(-x^2)$, $\sigma(x) = \exp(-|x|)$ and $\sigma(x) = \log(1 + e^x)$. \\

\begin{remark}[proof techniques]
    In both cases (analytic and non-analytic), we use two important properties of neurons to show their linear independence. The first one is that for $w\in \bR^d$ the mapping $\bR^d \ni x \mapsto w\cdot x$ reduces high-dimensional problems to 1-dimensional ones, as it induces a mapping $\bR \ni \vep \mapsto w\cdot(\vep x)$. Therefore, many 1-dimensional results of neurons can be applied to higher-dimensional cases. Another technique is to use the fact that distinct weights of neurons, no matter how small the difference is, could distinguish them with significant different behaviors. In Proposition \ref{independent activation}, we observe this by taking higher derivative, while in Proposition \ref{independent activation - alternative}, we observe this by comparing their asymptotic behaviors at $\pm\infty$. The treatment of $w_k + w_j = 0$ for some $k,j$ is more technical, and we make different assumptions in dealing with it: indeed, the function 
    \[
        \sigma(x) = e^{-x^2} + \frac{d^2}{dx^2}\frac{1}{1+e^{-x}} 
    \] 
    satisfies Assumption \ref{generic activation} and is rapidly decreasing, but $\lim_{x\to\pm\infty} \sigma(x)/\sigma(-x) = -1$. 
\end{remark}

We then study the rank of some special matrices related to $\sigma$. Recall the loss function $R(\theta) = \sum_{i=1}^n |\sum_{k=1}^m a_k \sigma(w_k\cdot x_i) - f(x_i)|^2$. For each $1 \le i \le n$, the second-order partial derivative of $|\sum_{k=1}^m a_k \sigma(w_k\cdot x_i) - f(x_i)|^2$ is given by the tensor product $v_{i, \sigma}^\TT v_{i, \sigma}$ where 
\begin{align*}
    v_{i,\sigma} := (\sigma(w_1\cdot x_i), ..., \sigma(w_m\cdot x_i), \sigma'(w_1\cdot x_i)x_i^\TT, ..., \sigma'(w_m\cdot x_i)x_i^\TT). 
\end{align*}
Since $R$ is a sum of the $|\sum_{k=1}^m a_k \sigma(w_k\cdot x_i) - f(x_i)|^2$'s, it motivates us to study the matrices whose rows are these vectors $v_{i,\sigma}$ or parts of them.

\begin{cor}[separating inputs]\label{Cor separating sample II} 
    Let $\sigma$ be a generic activation. For any $r \in \bN$ and any distinct $w_1, ..., w_r \in \bR^d \cut \{0\}$, there are $x_1, ..., x_{(d+1)r}$ such that the matrix
    \begin{equation*}
        \begin{pmatrix}
            \sigma(w_1\cdot x_1) &... &\sigma(w_r\cdot x_1) &\sigma'(w_1\cdot x_1)x_1^\TT &... &\sigma'(w_r\cdot x_1)x_1^\TT \\ 
            \vdots               &\ddots &\vdots            &\vdots                       &\ddots   &\vdots \\ 
            \sigma(w_1\cdot x_{(d+1)r}) &... &\sigma(w_r\cdot x_{(d+1)r}) &\sigma'(w_1\cdot x_{(d+1)r})x_{(d+1)r}^\TT &... &\sigma'(w_r\cdot x_{(d+1)r})x_{(d+1)r}^\TT
        \end{pmatrix}
    \end{equation*}
    has full rank. 
\end{cor}
\begin{proof}
For simplicity, denote $v_{i, \sigma}$ as the vector $(\sigma(w_1\cdot x_i), ..., \sigma(w_r\cdot x_i), \sigma'(w_1\cdot x_i)x_i^\TT, ..., \sigma'(w_r\cdot x_i)x_i^\TT)$, so $v_{i, \sigma} \in \bR^r \times \bR^{rd}$ is the $i$-th row of the matrix in question. We shall prove the result by inductively showing that there are $x_1, ..., x_{(d+1)r}$ such that for every $1 \le i \le (d+1)r$, the matrix $(v_{1, \sigma}^\TT, ..., v_{i, \sigma}^\TT)^\TT$ has rank $i$.  

Since $\sigma(0) \neq 0$, we can select a sufficiently small $x_1$ not orthogonal to $w_1$ so that $\sigma(w_1 \cdot x_1) \neq 0$. This proves the desired result for $i = 1$. Suppose the result holds for some $i-1$, where $2 \le i \le (d+1)r$. Fix $x_1, ..., x_{i-1}$. To find $x_i$ we do the following. First, choose $e \in \bR^d$ with the three properties below: 
\begin{itemize}
    \item [(a)] $p_j := \<w_j, e\> \neq \<w_i, e\> =: p_i$ for all distinct $i,j \in \{1, ..., r\}$. 

    \item [(b)] $p_1, ..., p_r \ne 0$. 

    \item [(c)] For any non-zero vector $b := ((b_1^0, ..., b_r^0), (b_1^1, ..., b_r^1)) \in \bR^r \times \bR^{rd}$ in the orthogonal complement of $\{v_{j, \sigma}:1 \le j < i\}$, we have that $b_1^0, ..., b_r^0, \<b_1^1, e\>, ..., \<b_1^r, e\>$ are not all zero. 
\end{itemize}
Note that this holds for almost all $e \in \bR^d$. 

We claim that $\vep \mapsto \sum_{k=1}^r b_k^0 \sigma(\vep p_j) + \sum_{j=1}^r \<b_j^1, e\> \sigma'(\vep p_j)\vep $ is not constant zero. Thus, for almost all $\vep$, by setting $x_i := \vep e$ we would have $v_{i,\sigma} \notin \Spans{v_{j, \sigma}: 1 \le j < i}$, and thus the matrix $(v_{1, \sigma}^\TT, ..., v_{i, \sigma}^\TT)^\TT$ has rank $i$. So suppose that this function is constant zero; equivalently, 
\begin{align*}
    &\sum_{k=1}^r b_k^0 \sigma(\vep p_k) + \sum_{k=1}^r \<b_k^1, e\> \sigma'(\vep p_k) \vep \\    
    =:&\, c_0 \sum_{k=1}^r b_k^0 + \sum_{j\geq 1} c_j \sum_{k=1}^r [\alpha_{1k} + j\alpha_{2k}] p_k^{j-1} \vep^j \\ 
    =&\, c_0 \sum_{k=1}^r b_k^0 + \vep \left( \sum_{j\ge 1} c_j \sum_{k=1}^r [\alpha_{1k} + j\alpha_{2k}] p_k^{j-1} \vep^{j-1} \right) \\ 
    =&\, 0, 
\end{align*}
where for each $1 \le k \le m$, 
\[
    \alpha_{1k} = b_k^0 p_k, \quad \alpha_{2k} = \<b_k^1, e\>. 
\]
Similar as in the proof of Corollary \ref{independent activation}, we may apply Lemma \ref{Lem Lin comb of stretched sigma and their derivatives} to see that $\sum_{k=1}^r \beta_k = 0$ and $\alpha_{1k} = \alpha_{2k} = 0$ for all $1 \le k \le m$. But this implies that $b_k^0 = \<b_k^1, e\> = 0$ for all $k$, contradicting our construction of $e$. Thus, this function is not constant zero, completing the induction step. 
\end{proof}

Inspired by the lemma above, we define separating inputs below. As we will see in Section \ref{Section Loss landscape near Minfty}, separating inputs allows us to determine the rank of Hessian of $R$ on $\Rmuzero$, and thus help us study the geometry of $R$ (e.g., separation of branches, see also Definition \ref{fragmentation/stratification of M^infty} for a description of geometry of branches) near its perfect global minima $Q^*$. This in turn guarantees that locally any gradient flow finds a solution with zero generalization error. We will prove this in Section \ref{Section Dynamics of gradient flow near Minfty}.  

\begin{defn}[separating inputs]\label{Defn Separating inputs}
    Let $r \in \{1, ..., m\}$. Given $n \le (d+1)m$ and distinct $w_1, ..., w_r \in \bR^d \cut \{0\}$, we call $\{x_i\}_{i=1}^n$ separating inputs (for $w_1, ..., w_r$) if the matrix 
    \begin{equation}\label{eq 1 of Defn Sepaating inputs}
        \begin{pmatrix}
            \sigma(w_1\cdot x_1) &... &\sigma(w_r\cdot x_1) &\sigma'(w_1\cdot x_1)x_1^\TT &... &\sigma'(w_r\cdot x_1)x_1^\TT \\ 
            \vdots               &\ddots &\vdots            &\vdots                       &\ddots   &\vdots \\ 
            \sigma(w_1\cdot x_n) &... &\sigma(w_r\cdot x_n) &\sigma'(w_1\cdot x_n)x_n^\TT &... &\sigma'(w_r\cdot x_n)x_n^\TT 
        \end{pmatrix}
    \end{equation}
    has full rank. For any $\theta = (a_k, w_k)_{k=1}^m \in \bR^{(d+1)m}$ with non-zero $w_k$'s, we say $\{x_i\}_{i=1}^n$ are separating inputs for $\theta$ if they are separating inputs for the distinct elements among $\{w_1, ..., w_m\}$.  
\end{defn}

Note that Corollary \ref{Cor separating sample II} shows the existence of separating inputs for $\theta$. Since $\sigma$ is analytic, this implies that for almost every $(x_1, ..., x_n) \in \bR^{dn}$, $\{(x_i, f^*(x_i))\}_{i=1}^n$ are separating inputs for $\theta$, see Lemma \ref{open dense set of separating inputs}. 

\subsection{Theory of Real Analytic Functions}\label{Subsection Theory of real analytic functions}

To prove such corollaries in this subsection, we introduce some properties of real analytic functions. such tools will also be used in the analysis of the geometry of the loss landscape of $R$ near its global-min $\Rmuzero$. 

\begin{lemma}[analytic implicit function theorem]\label{analytic version of rank theorem}
    Let $\calM_1$ and $\calM_2$ be smooth manifolds of dimension $s_1$ and $s_2$, respectively, and each one has a coordinate representation given by an analytic diffeomorphism. Suppose that $\calF: \calM_1 \to \calM_2$ is a smooth map with constant rank $s$ ($\text{rank}\,D\calF(p) = s$ for all $p$). For each $p \in \calM_1$ and any coordinate ball $U$ around $p$, there exist coordinate maps $\varphi: U \to \bR^{s_1}$, $\psi: \calF(\varphi(U)) \to \bR^{s_2}$, such that 
    \begin{equation}
        \psi \circ \calF \circ \varphi^{-1} (\zeta_1, ..., \zeta_{s_1}) = (\zeta_1, ..., \zeta_s, 0, ..., 0) 
    \end{equation}
    for $(\zeta_1, ..., \zeta_{s_1}) \in \varphi(U)$. 
\end{lemma}
\begin{proof}
By \cite{SKrantz}, the inverse of a (real) analytic diffeomorphism is analytic. Following the proof of the Rank theorem in \cite{JLee}, we can see that all the functions constructed can be made analytic, whence the desired result holds. 
\end{proof}

\begin{lemma}\label{zero set of multivariable analytic function}
Let $\Omega \siq \bR^{(d+1)m}$ be a connected open set and let $f: \Omega \to \bR$ be a (real) analytic function. The following results hold. 
\begin{itemize}
    \item [(a)] Either $f = 0$ on $\Omega$ or $f^{-1}\{0\}$ has $\lambda_{(d+1)m}$-measure zero. 
    
    \item [(b)] Suppose that $f$ is not a zero map. For any $z_0 \in f^{-1}\{0\}$ and any compact $K \siq \Omega$, $K \cap f^{-1}\{0\}$ is contained in a finite union of (analytic) smooth manifolds of dimension $(d+1)m-1$, each one having a coordinate representation by an analytic diffeomorphism. 
    
    \item [(c)] Let $f^{-1}\{0\}$ be locally contained in the union of connected embedded smooth manifolds $\calM_1, ..., \calM_N$ of the same dimension, with $\calM_1 \siq f^{-1}\{0\}$. Then there is an open $U \siq \Omega$ such that $U \cap f^{-1}\{0\} = U \cap \calM_1 \neq \emptyset$. 
\end{itemize}
\end{lemma}
\begin{proof}
\begin{itemize}
    \item [(a)] We will show that given $n \in \bN$ and any set $\Omega$ open in $\bR^n$, the zero set of any non-constant analytic function $f: \Omega \to \bR$ has $\lambda_n$-measure zero. Indeed, when $n=1$ the zeros of $f$ are (possibly empty isolated points), which means $f^{-1}\{0\}$ has measure zero. Suppose that the result holds for $n-1$. Let $E \siq \bR$ be the projection of $\Omega$ onto $\Spans{e_1}$. For any $z_1 \in E$, $f(z_1, \cdot)$ is an analytic function in $n-1$ variables, whence by our assumption its zero set has $\lambda_{n-1}$-measure zero. It follows that 
    \begin{equation*}
        \lambda_n (f^{-1}\{0\}) = \int_\Omega \chi_{f^{-1}\{0\}} = \int_E \lambda_{n-1}(f(z_1,\cdot)^{-1}\{0\}) dz_1 = 0, 
    \end{equation*}
    where $\chi$ is the characteristic function. This completes our induction and shows that $f^{-1}\{0\}$ has measure zero. 
    
    \item [(b)] It suffices to show that for any $y \in f^{-1}\{0\}$, $y$ has a bounded neighborhood $U$ such that $U \cap f^{-1}\{0\}$ is contained in a finite union of (analytic) submanifolds of codimension $1$. Since $f$ is not a zero map on $\Omega$, there is some $N \in \bN$ such that $D^t f(y) = 0$ for all multi-index $t$ with $|t| \leq N$ and there are some $j \in \{1, ..., (d+1)m\}$, multi-index $\alpha$ with $|\alpha| = N$ and $\partial_{z_j} D^\alpha f(y) \neq 0$. Thus, 
    \begin{equation*}
        y \in \{z \in \bR^{(d+1)m}: D^\alpha f(z) = 0 \text{ and } \partial_{z_j} D^\alpha f(z) \neq 0\}, 
    \end{equation*} 
    which shows that $y$ is in a smooth manifold of dimension $(d+1)m-1$, and by Lemma \ref{analytic version of rank theorem}, it has a coordinate representation given by an analytic diffeomorphism. Let $U$ be a bounded neighborhood of $y$ on which $\partial_{z_j} D^\alpha f \neq 0$. It follows that $U \cap f^{-1}\{0\}$ is contained in
    \begin{equation}
         \bigcup_{t, 0 \leq |t| \leq N} \bigcup_{i=1}^{(d+1)m} 
        \{z \in \bR^{(d+1)m}: D^t f(z) = 0 \text{ and } \partial_{z_i} D^t f(z) \neq 0\}. 
    \end{equation}
    Just note that the right side is a finite union of (possibly empty) embedded, analytic smooth manifolds, as desired. 
    
    \item [(c)] By shrinking the domain of $f$, $\Omega$ if necessary, we may for simplicity assume that $f^{-1}\{0\} \siq \bigcup_{j=1}^N \calM_i$. First suppose that $N = 2$. Since 
    \begin{equation*}
        \calM_1 = (\calM_1 \cut \calM_2) \cup (\calM_1 \cap \calM_2), 
    \end{equation*} 
    one of the sets on the right side of the equation must contain an open submanifold $E$ of $\calM_1$. Consider first the case in which $E \siq \calM_1 \cut \calM_2$ and $x \in E$. If for any $\delta > 0$ the $(d+1)m$-dimensional ball $B(x, \delta)$ intersects $\calM_2$, then $x$ is in the closure of $\calM_2$, which implies that $x \in \calM_2$, as $\calM_2$ is embedded in $\bR^{(d+1)m}$, a contradiction. Now suppose that $x \in E \siq \calM_1 \cap \calM_2$. Since $\dim \calM_1 = \dim \calM_2$, $E$ is also an open submanifold of $\calM_2$; furthermore, it is the intersection of an open subset of $\bR^{(d+1)m}$ with $\calM_1$ ($\calM_2$). In either case, there is some open $U \siq \bR^{(d+1)m}$ with $U \cap f^{-1}\{0\} = U \cap \calM_1$. When $N \geq 2$, we argue by induction to obtain the same desired result. 
\end{itemize}
\end{proof}

\begin{cor}\label{R mu is real analytic}
    Let $\mu = \sum_{i=1}^n \delta(\cdot -x_i)$ as in Assumption \ref{finite-feature setting}. If $\sigma$ is analytic, then $R$ is analytic, in which case all the results in Lemma \ref{zero set of multivariable analytic function} apply to $R$. 
\end{cor}
\begin{proof}
By our hypothesis, $R(\theta) = \sum_{i=1}^n |g(\theta, x_i) - f^*(x_i)|^2$ for some $x_1, ..., x_n \in \bR^d$. Because the composition of two analytic functions is again analytic, each $a_k \sigma(w_k \cdot x_i)$ is analytic (in $a_k, w_k$); because adding and/or multiplying two analytic functions produces analytic functions, $R$ is analytic. 
\end{proof}

\begin{cor}[common zeros of parametrized analytic functions]\label{square-loss type perturbation}
    Let $f: E \times \bR^d \to \bR$ be analytic, where $E \siq \bR^s$. Suppose that for any $z \in E$, $f(z, \cdot)$ is not constant-zero. Given $n > s$, for any prescribed $x_1, ..., x_n \in \bR^d$ and any $\vep > 0$, we can perturb each of them with no more than $\vep$-distance to obtain $x_1', ..., x_n'$, so that 
    \begin{equation*}
        \calA_E := \{z \in E: f(z, x_1') = ... = f(z, x_n') = 0\}
    \end{equation*}
    is empty. In particular, $\calA_E = \emptyset$ for almost all $(x_1', ..., x_n') \in \bR^{nd}$. 
\end{cor}
\begin{proof}
By hypothesis, $f$ is not constant-zero, whence there is some $x_1' \in \bR^d$ with $|x_1' - x_1| < \vep$ such that $f(\cdot, x_1')|_E$ is not constant-zero. By Lemma \ref{zero set of multivariable analytic function} (b), for each $N \in \bN$ and $x \in \bR^d$ with $f(\cdot, x)$ not constant-zero, $f(\cdot, x)^{-1}\{0\} \cap (\overline{B(0,N)} \cap E)$ is contained in a finite union of codimension-1 submanifolds of $\bR^s$. Thus, $f(\cdot, x_1')^{-1}\{0\}$ is contained in a union of countably many submanifolds of $\bR^s$, which we denote by $\{\calM_{1n}: n \in \bN\}$. 

For the choice of $x_2'$ we consider two cases. Suppose first that there are (possibly finite) $\calM_{1n_1}, ..., \calM_{1n_k}, ...$ such that $f(\cdot, x_2)$ restricted to each of them is a zero map. Select points $\{z_k' \in \calM_{1n_k}: k \in \bN\}$. For each $k$, there is an open, dense $O_k' \siq \bR^d$ such that $\lambda_d (\bR^d\cut O_k') = 0$ and $f(z_k', x) \neq 0$ for all $x \in O_k'$. Thus, $\bigcap_{k=1}^\infty O_k'$ is dense and have (Lebesgue measure) zero in $\bR^d$. Therefore $f(z_k', x) \neq 0$ for all $k \in \bN$, whenever $x \in \bigcap_{k=1}^\infty O_k'$. Let $x_2' \in \bigcap_{k=1}^\infty O_k'$ with $|x_2' - x_2| < \vep$. Otherwise, $f(z, x_2')$ is not constant-zero on $\bigcup_{k=1}^\infty \calM_{1n}$; in this case simply let $x_2' = x_2$. By Lemma \ref{zero set of multivariable analytic function}, the set on which $f(\cdot, x_1')|_E = f(\cdot, x_2')|_E = 0$ is contained in a countable union of codimension-2 (analytic) submanifolds of $\bR^s$. 

Repeat this procedure to perturb $x_1, ..., x_s$. We see that in general, for $s' \leq s$, 
\begin{equation*}
    \calA_{E,s'} := \{z \in E: f(z, x_1') = ... = f(z, x_{s'}') = 0\}
\end{equation*}
is contained in a countable (possibly finite) union of manifolds of codimension $s'$. In particular, $\calA_{E,s}$ is a countable (possibly finite) set. For simplicity, denote $\calA_{E,s} := \{z_k: k \in \bN\}$. For each $k$, there is an open $O_k \siq \bR^d$ such that $\lambda_d (\bR^d\cut O_k) = 0$ and $f(z_k, x) \neq 0$ for all $x \in O_k$. Let $x_{s+1}' \in \bigcap_{k=1}^\infty O_k$ with $|x_{s+1}' - x_{s+1}| < \vep$. Finally, let $x_{s+2}' = x_{s+2}$, ..., $x_n' = x_n$. The parameters $x_1' ,..., x_{s+1}', ..., x_n'$ satisfy our requirements. 
\end{proof}
\begin{remark}
    This technique enables us to control the size of the global minimum of $R$. See also Lemmas \ref{separated branch} (b) and \ref{underparametrized system}. 
\end{remark}

\subsection{Separating Inputs are Almost Everywhere}\label{Subsection Separating inputs are almost everwhere}

\begin{cor}[separating inputs works almost everywhere]\label{Cor separating inputs works ae}
    Suppose that Assumptions \ref{generic activation} and \ref{finite-feature setting} hold. Given separating inputs $\{x_i\}_{i=1}^n$ for some fixed distinct weights $w_1^*, ..., w_r^* \in \bR^d \cut \{0\}$. Given any $1 \leq j \leq r$, there is an open dense subset of $\bR^{jd}$ such that for any $(w_1, ..., w_j)$ in it, $\{x_i\}_{i=1}^n$ are separating inputs for $(w_1, ..., w_j, w_{j+1}^*, ..., w_r^*)$.
\end{cor} 
\begin{proof}
Let $A = A((w_k)_{k=1}^r, (x_i)_{i=1}^n)$ be the matrix (\ref{eq 1 of Defn Sepaating inputs}) in Definition \ref{Defn Separating inputs}. We may assume that $n \le (d+1)r$, because when $n > (d+1)r$ we can always find distinct $i_1, ..., i_{(d+1)r} \in \{1, ..., (d+1)r\}$ such that $A((w_k^*)_{k=1}^r, (x_{i_j})_{j=1}^{(d+1)r})$ has full rank. \\

Since $n \le (d+1)r$, the number of columns of $A$ is no less than the number of rows of $A$, so the fact that it is full rank means there is some $n\times n$ submatrix of it, denoted by $B((w_k^*)_{k=1}^r, (x_i)_{i=1}^n)$, such that $\det B = B((w_k^*)_{k=1}^r, (x_i)_{i=1}^n) \ne 0$. Notice that when $w_{j+1}^*, ..., w_r^*$ and $x_1, ..., x_n$ are fixed, we can view $\det B$ as an analytic function in $w_1, ..., w_j$: 
\[
    \bR^{jd} \ni (w_1, ..., w_j) \mapsto \det B((w_1, ..., w_j, w_{j+1}^*, ..., w_r^*), (x_i)_{i=1}^n). 
\]
Therefore, our proof above implies that it is not constant zero on $\bR^{jd}$, whence by Lemma \ref{zero set of multivariable analytic function} (a), its zero set, which is clearly closed, has $\lambda_{jd}$-measure zero. Thus, $\det B \ne 0$, and thus $A$ has full rank for $(w_1, ..., w_j)$ an open dense subset of $\bR^{jd}$. 
\end{proof}

Argue in the same way as above, we can see that separating inputs are almost everywhere: 

\begin{lemma}[separating inputs are almost everywhere]\label{open dense set of separating inputs}
    Given $\theta = (a_k, w_k)_{k=1}^m \in \bR^{(d+1)m}$ with $w_1,..., w_m \ne 0$. Almost all choices of $(x_1, ..., x_n) \in \bR^{nd}$ are separating inputs for $\theta$. 
\end{lemma}

\subsection{Geometry of \texorpdfstring{$Q^*$}{M}}\label{Subsection geometry of Minfty} 

In this part, we investigate the geometry of $Q^*$ under the assumption that $\sigma$ is a generic activation (Assumption \ref{generic activation}). In this case, Refs. \cite{BSimsek, KFukumizu} have shown that the global minima of the loss function is a union of subsets of $\bR^{(d+1)m}$ whose closure are affine subspaces which we call ``branches". For completeness of our study for local recovery, we present this result and prove it with our notations (see Proposition \ref{prop fragmentation/stratification}). Based on this, \cite{BSimsek} shows that $Q^*$ is connected. We further investigate how different branches intersect with each other. First let's recall partition and refinement. \\

\begin{defn}[partition and refinement]
    A partition of $\{1, ..., m\}$ is denoted by $P := (q_0, ..., q_r)$ for some $r \in \bN$, where $0 = q_0 < q_1 < ... < q_r = m$. Given another partition $P' = (q_0', ..., q_{r'}')$, we say $P'$ is a refinement of $P$ if $\{q_0, ..., q_r\} \siq \{q_0', ..., q_{r'}'\}$. 
\end{defn}

We then partition $Q^*$ into symmetric subsets based on the number of distinct $w_k$'s. 

\begin{defn}[fragmentation/stratification of $Q^*$]\label{fragmentation/stratification of M^infty}

    Given $m_0 \leq r \leq m$, a partition $P = (q_0, ..., q_r)$ and a permutation $\pi$ of $\{1, ..., m\}$. We define the following subsets of $Q^*$. 
    \begin{itemize}
        \item [(a)] $Q^r := \{\theta^* \in Q^*: \card \{w_k^*: 1 \leq k \leq m\} = r\}$, namely, it consists of points $(a_k^*, w_k^*)_{k=1}^m \in Q^*$ which has precisely $r$ distinct $w_k^*$'s. 

        \item [(b)] $Q_P^r$ consists of $\theta^* = (a_k^*, w_k^*)_{k=1}^m \in Q^*$ such that $w_k^* = w_{q_t}^*$ for all $q_{t-1} < k \leq q_t$ and $1 \leq t \leq r$, and $w_{q_t}^* = \bar{w}_t$ for all $1 \leq t \leq m_0$. For example, if $m_0 = 2$, $r = 3$, $m = 6$ and $P = (0, 2, 4, 6)$ then $\theta^* \in Q_P^3$ if and only if it has the form 
        \[
        \theta^* = (a_1^*, \bar{w}_1, a_2^*, \bar{w}_1, a_3^*, \bar{w}_2, a_4^*, \bar{w}_2, a_5^*, u, a_6^*, u)
        \]
        for some $u \in \bR^d$. 

        \item [(c)] $Q_{P,\pi}^r := \{(a_{\pi(k)}^*, w_{\pi(k)}^*)_{k=1}^m: (a_k^*, w_k^*)_{k=1}^m \in Q_P^r\}$. We call each $Q_{P,\pi}^r$ a brunch of $Q^*$. 
        
    \end{itemize}
\end{defn}

The notations $Q_P^r$ and $Q_{P,\pi}^r$ make sense only when $P$ partitions $\{1, ..., m\}$ into $r$ subsets. Thus, when we write these notations we implicitly assume that $P$ satisfies this requirement. Then, clearly, $Q_{P,\pi}^r$ and $Q_P^r$ are both subsets of $Q^r$. Also, notice that for any $(a_k^*, w_k^*)_{k=1}^m \in Q_P^{m_0}$ ($Q_{P,\pi}^{m_0}$), $w_k^* \in \{\bar{w}_j: 1 \leq j \leq m\}$ for each $k$, while for any $(a_k^*, w_k^*)_{k=1}^m \in Q_P^m$ ($Q_{P,\pi}^m$), $w_k^* \neq w_j^*$ whenever $k \neq j$. \\

An immediate observation is that for any permutation $\pi$, $Q_{P,\pi}^r$ is the image of $Q_P^r$ under a coordinate transformation, so they have the same geometric properties. To see this, let $\{e_i: 1 \leq i \leq (d+1)m\}$ be the standard basis of $\bR^{(d+1)m}$. Clearly, $\pi$ induces a permutation $\tau$ on $\{1, ..., (d+1)m\}$ that maps $i$ to $\tau(i) = (d+1)(\pi(i)-1) + k$ if $i = k \bmod{(d+1)}$, $1 \leq k \leq d+1$. Let $\rho$ be the linear map on $\bR^{(d+1)m}$ such that $\rho(e_i) = e_{\tau(i)}$ for all $1 \leq i \leq (d+1)m$. Then $Q_{P,\pi}^r = \rho(Q_P^r)$. \\

Since $Q_{P,\pi}^r$ is isometric to $Q_P^r$, $Q^*$ has a strong symmetry property. In particular, to study the structure of $Q_{P,\pi}^r$ we need only study that of $Q_P^r$. Similarly, given $r'$, partition $P'$ and permutation $\pi'$, to investigate $\overline{Q_{P,\pi}^r} \cap \overline{Q_{P',\pi'}^{r'}}$ it suffices to investigate $\overline{Q_P^r} \cap \overline{Q_{P',\tau}^{r'}}$, where $\tau = \pi^{-1} \circ \pi'$. More importantly, to study the behavior of $R$ near $Q_{P,\pi}^r$ we need only study its behavior near $Q_P^r$. 

\begin{prop}[fragmentation/stratification of $Q^*$, Theorem 3.1 in \cite{BSimsek}]\label{prop fragmentation/stratification}
    Suppose that Assumption \ref{generic activation} and Assumption \ref{finite-feature setting} hold. Fix $m_0 \leq r < r' \leq m$. 
    \begin{itemize}
        \item [(a)] $Q^* = \bigcup_{r''=m_0}^m Q^{r''}$. 
        
        \item [(b)] $Q^r \cap Q^{r'} = \emptyset$, and $Q^r = \bigcup_{P,\pi} Q_{P,\pi}^r$, the union taken over all possible partition $P$'s and possible permutation $\pi$'s. Thus, $Q^* = \cup_{r''=m_0}^m \cup_{P,\pi} Q^{r''}$. 
        
        \item [(c)] For any partition $P = (q_0, ..., q_r)$, $\overline{Q_P^r}$ is an affine subspace of dimension $(m-r) + (r-m_0)d$. Furthermore, for any $\theta^* \in Q_P^r$, $\theta^*$ has a neighborhood $U \siq \bR^{(d+1)m}$ such that $U \cap Q_P^r$ equals $U \cap \overline{Q_P^r}$. 
    \end{itemize}
\end{prop} 
\begin{remark}
    By (b) and (c) and our remark above, the Hausdorff dimension of $Q^r$ is $(m-r) + (r-m_0)d$. 
\end{remark} 
\begin{proof}
\begin{itemize}
    \item [(a)] If $\theta^* \in Q^*$, we have 
    \begin{equation*}
        g(\theta^*, x) - f^*(x) = \sum_{k=1}^m a_k^* \sigma(w_k^* \cdot x) - \sum_{k=1}^{m_0} \bara_k \sigma(\barw_k \cdot x)
    \end{equation*}
    for all $x \in \bR^d$. Suppose that $\card\{w_1^*, ..., w_m^*\} = r''$. By Corollary \ref{independent activation}, $\{w_1^*, ..., w_m^*\} \soq \{\barw_1, ..., \barw_{m_0}\}$ and thus $r'' \geq m_0$. Furthermore, if $w_{k_1}^*, ..., w_{k_j}^* = \barw_k$ for some $k \in \{1, .., m_0\}$, then we must have $\sum_{t=1}^j a_{k_t}^* = \bara_k$; otherwise, we must have $\sum_{t=1}^j a_{k_t}^* = 0$. This shows that $\theta^* \in Q^{r''}$. Since $m_0 \leq r'' \leq m$, it follows that $Q^* = \bigcup_{r''=m_0}^m Q^{r''}$. 
    
    \item [(b)] By definition, $Q^r \cap Q^{r'} = \emptyset$. Fix $\theta^* \in Q^r$. We can find distinct weights $w_{k_1}^*, ..., w_{k_r}^*$ such that 
    \begin{equation*}
        \{w_{k_1}^*, ..., w_{k_r}^*\} \soq \{\barw_1, ..., \barw_{m_0}\}. 
    \end{equation*} 
    By reordering them if necessary, we may further assume that for any $t \in \{1, ..., m_0\}$, $w_{k_t}^* = \barw_t$. For each $t \in \{1, ..., r\}$, let $w_{t(1)}^*, ..., w_{t(\delta_t)}^*$ be all the $w_k^*$'s that are equal to $w_{k_t}^*$. Set $q_t' = \sum_{j=1}^t \delta_j$; also set $q_0' = 0$. Let $P' = (q_0', q_1', ..., q_r')$ and $\pi'$ be the permutation satisfying $\pi'(q_{t-1}+j) = t(j)$, $1 \leq j \leq \delta_t$, for all $t \in \{1, ..., r\}$. But then $\theta^* \in Q_{P',\pi'}^r$, so $Q^r$ is the union of $Q_{P,\pi}^r$'s. Since $r$ is arbitrary, we conclude from (a) that $Q^* = \cup_{r''=m_0}^m \cup_{P,\pi} Q_{P,\pi}^{r''}$. 
    
    \item [(c)] Clearly, $\theta^* \in Q_P^r$ if and only if it satisfies 
        \begin{itemize}
            \item [i)] $\sum_{k=q_{t-1}+1}^{q_t} a_k^* = \bara_t$ for $1 \leq t \leq m_0$; 
            \item [ii)] $\sum_{k=q_{t-1}+1}^{q_t} a_k^* = 0$ for $m_0 < t \leq r$; 
            \item [iii)] $w_{q_{t-1}+1}^* = ... = w_{q_t}^* = \barw_t$ for $1 \leq t \leq m_0$; 
            \item [iv)] $w_{q_{t-1}+1}^* = ... = w_{q_t}^* \notin \{w_{q_1}^*, ..., w_{q_{t-1}}^*\}$ for $m_0 < t \leq r$, 
        \end{itemize}
    from which we can see that the closure of $Q_P^r$ consists of $\theta^* \in \bR^{(d+1)m}$ with 
        \begin{itemize}
            \item [i)] $\sum_{k=q_{t-1}+1}^{q_t} a_k^* = \bara_t$ for $1 \leq t \leq m_0$; 
            \item [ii)] $\sum_{k=q_{t-1}+1}^{q_t} a_k^* = 0$ for $m_0 < t \leq r$; 
            \item [iii)] $w_{q_{t-1}+1}^* = ... = w_{q_t}^* = \barw_t$ for $1 \leq t \leq m_0$; 
            \item [vi)] $w_{q_{t-1}+1}^* = ... = w_{q_t}^* \in \bR^d$ for $m_0 < t \leq r$. 
        \end{itemize}
    Fix $\ttheta^* \in Q_P^r$. Given $t \in \{1, ..., r\}$, let $A_t$ be the subspace of $\bR^{q_t-q_{t-1}}$ such that for any $(z_1, ..., z_{q_t-q_{t-1}}) \in A_t$, $z_1 + ... + z_{q_t-q_{t-1}} = 0$, and $W_t$ be the subspace of $\prod_{k=1}^{q_t-q_{t-1}} \bR^d = \bR^{(q_t-q_{t-1})d}$, such that for any $(z_1, ..., z_{q_t-q_{t-1}}) \in W_t$, $z_1 = z_2 = ... = z_{q_t-q_{t-1}} \in \bR^d$. By making the identification 
    \begin{align*}
        \bR^{(d+1)m} &\simeq \prod_{t=1}^r \prod_{k=1}^{q_t-q_{t-1}} (\bR \times \bR^d) \\
        &\simeq \prod_{t=1}^r (\bR^{(q_t-q_{t-1}} \times \bR^{(q_t-q_{t-1})d}) \\
        &= \prod_{t=1}^r \{(a_{q_{t-1}+1}, ..., a_{q_t}, w_{q_{t-1}+1}, ..., w_{q_t}) \in \bR^{q_t-q_{t-1}} \times \bR^{(q_t-q_{t-1})d} \}, 
    \end{align*}
    we can see that
    \begin{equation*}
        \overline{Q_P^r} \simeq \ttheta^* + \prod_{t=1}^{m_0} (A_t \times \{0\}) \times \prod_{t=m_0+1}^r (A_t \times W_t). 
    \end{equation*}
    Thus, $\overline{Q_P^r}$ is an affine subspace of dimension 
    \begin{equation*}
        \sum_{t=1}^{m_0} (q_t - q_{t-1}-1) + \sum_{t=m_0+1}^r [(q_t - q_{t-1}-1) + d] = (m - r) + (r - m_0)d. 
    \end{equation*}
    Since $\overline{Q_P^r}\cut Q_P^r$ consists of $\theta^* \in \overline{Q_P^r}$ such that for some $t\in \{m_0 +1, ..., r\}$ we have $w_{q_t}^* \in \{w_{q_1}^*, ..., w_{q_{t-1}}^*\}$, it is a close subset of $\bR^{(d+1)m}$. Fix $\theta^* \in Q_P^r$. Thus, any neighborhood of $\theta^*$ that intersects trivially with $\overline{Q_P^r}\cut Q_P^r$ satisfies $U \cap Q_P^r = U \cap \overline{Q_P^r}$. 
\end{itemize}
\end{proof}

\vs{0.5em}

\begin{lemma}[branch intersection]\label{intersection of brances of M infty I}
    Suppose that Assumptions \ref{generic activation} and \ref{finite-feature setting} hold. Let $m_0 \leq r \leq m$. Suppose that $P = (q_0, ..., q_r)$ and $P' = (q_0', ..., q_{r'}')$ are two partitions of $\{1, ..., m\}$. Also let $\pi$ be a permutation of $\{1, ..., m\}$. If there are indices $\{t_1, ..., t_{m'}\}$ such that $\{\pi(k): q_{t_{i-1}}' < k \leq q_{t_i}'\}$ is not contained in $(q_{t-1}, q_t]$ for all $t$, then $\overline{Q_P^r} \cap \overline{Q_{P',\pi}^{r'}} \siq \bigcup_{k=m_0}^{r-m'} Q^k$. In particular, if $\theta^* \in Q_P^r \cap \overline{Q^{r'}}$, then there is a partition $\tilde{P} = (\tilde{q}_0, ..., \tilde{q}_r)$ finer than $P$ and a permutation $\tilde{\pi}$ of $\{1, ..., m\}$ such that $\theta^* \in Q_P^r \cap \overline{Q_{\tilde{P},\tilde{\pi}}^{r'}}$. 
\end{lemma} 
\begin{proof}
Suppose that $\theta^* \in \overline{Q_P^r} \cap \overline{Q_{P',\pi}^{r'}}$. Then there are sequences $\{ \theta_n\}_{n=1}^\infty$ in $Q_P^r$ and $\{\theta'_n\}_{n=1}^\infty$ in $Q_{P',\pi}^{r'}$ that converge to $\theta^*$, respectively. Fix $1 \leq i \leq m'$. By hypothesis there are distinct $j_1, j_2 \in \{1, ..., r\}$ and $k_1, k_2 \in \{q_{t_{i-1}+1}', ..., q_{t_i}'\}$ such that 
\begin{equation*}
    q_{j_1-1} < \pi(k) \leq q_{j_1}; \quad q_{j_2-1} < \pi(k_2) \leq q_{j_2}. 
\end{equation*} 
It follows that 
\begin{equation*}
    \limftyn (w_{q_{j_1}})_n = \limftyn (w_{\pi(k_1)})_n' = \limftyn (w_{\pi(k_2)})_n' = \limftyn (w_{q_{j_2}})_n 
\end{equation*} 
because for each $n$ we have $(w_{\pi(k_1)})_n' = (w_{\pi(k_2)})_n'$. Therefore, 
\begin{equation}
    \limftyn (w_k)_n = w_t^* = \limftyn (w_s)_n
\end{equation} 
for all $q_{j_1-1} < k \leq q_{j_1}$, $q_{j_2-1} < s \leq q_{j_2}$ and all $t \in \{q_{j_1-1}+1,..., q_{j_1}\} \cup \{q_{j_2-1}+1,..., q_{j_2}\}$. This means any occurrence of such indices $t_i$ ``reduces" the number of distinct $w_t^*$'s of $\theta^*$ by 1; since there are $m'$ such occurrences, the number of distinct $w_t^*$'s of $\theta^*$ is at most $r-m'$. It is also clear that (the possibly empty set) $\overline{Q_P^r} \cap \overline{Q_{P',\pi}^{r'}} \siq Q^*$. Therefore, $\overline{Q_P^r} \cap \overline{Q_{P',\pi}^{r'}} \siq \bigcup_{k=m_0}^{r-m'} Q^k$. 

Now assume that $\theta^* \in Q_P^r \cap \overline{Q^{r'}}$, so there is some partition $\tilde{P}$ and permutation $\tilde{\pi}$ with $\theta^* \in Q_P^r \cap \overline{Q_{\tilde{P}, \tilde{\pi}}^{r'}}$. Also, since $\theta^* \in Q_P^r$, $\theta^* \in Q^r$. Then our proof above implies that for any index $j \in \{1, ..., r'\}$, there is some $t \in \{1, ..., r\}$ with 
\begin{equation*}
    \{\tilde{\pi}(k): \tilde{q}_{j-1} < k \leq \tilde{q}_j\} \siq (q_{t-1}, q_t]. 
\end{equation*} 
This means there is a permutation $\tau$ on $\{1, ..., m\}$ with 
\begin{equation*}
    \{\tau\tilde{\pi}(k): \tilde{q}_{j-1} < k \leq \tilde{q}_j\} \siq (q_{t-1}, q_t] 
\end{equation*}
and $\tau\tilde{\pi}(\tilde{q}_{j-1}) - \tau\tilde{\pi}(\tilde{q}_j) = 1$ for all $1 \leq j \leq r'$. Note that we always have $\tau\tilde{\pi}(\tilde{q}_0) = 0$ and $\tau\tilde{\pi}(\tilde{q}_{r'}) = m$. Replacing $\tilde{P}$ with the partition $(\tau\tilde{\pi}(\tilde{q}_0), ..., \tau\tilde{\pi}(\tilde{q}_{r'}))$ and replacing $\tilde{\pi}$ with $\tau^{-1}$ concludes the proof. \\
\end{proof}

\textbf{Example. } Consider a three-neuron model $g(\theta, x) = \sum_{k=1}^3 a_k \sigma(w_k x)$ a one-neuron target function $f^*(x) = \bara \sigma(\barw x)$. Up to permutations of $\{0, 1, 2, 3\}$, $Q^*$ consists of the following branches: 
\begin{itemize}
    \item [(a)] $Q_{\{0,3\}}^1$: points of the form $(a_1, \barw, a_2, \barw, \bara - (a_1 + a_2), \barw)$. This set has Hausdorff dimension $2$. 

    \item [(b)] $Q_{\{0,2,3\}}^2$: points of the form $(a_1, \barw, \bara - a_1, \barw, 0, w_3)$. This set has Hausdorff dimension $1 + d$. 

    \item [(c)] $Q_{\{0,1,3\}}^2$: points of the form $(\bara, \barw, a_2, w_2, -a_2, w_2)$. This set has Hausdorff dimension $1 + d$. 

    \item [(d)] $Q_{\{0,1,2,3\}}^3$: points of the form $(\bara, \barw, 0, w_2, 0, w_3)$. This set has Hausdorff dimension $2d$. 
\end{itemize}
Clearly, the closure of each such branch is an affine subspace of $\bR^{3(d+1)}$. Moreover, we have $\overline{Q_{\{0,2,3\}}^2} \cap \overline{Q_{\{0,1,3\}}^2} \siq Q_{\{0,3\}}^1$, $\overline{Q_{\{0,2,3\}}^2} \cap \overline{Q_{\{0,1,2,3\}}^3} \siq Q_{\{0,2,3\}}^2$, and $\overline{Q_{\{0,1,3\}}^2} \cap Q_{\{0,1,2,3\}}^3 \siq Q_{\{0,1,3\}}^2$. See also the figures below for how different branches are related to one another. \\

\begin{figure}[H]
    \centering
    \begin{minipage}[b]{0.49\textwidth}
        \includegraphics[width=\textwidth]{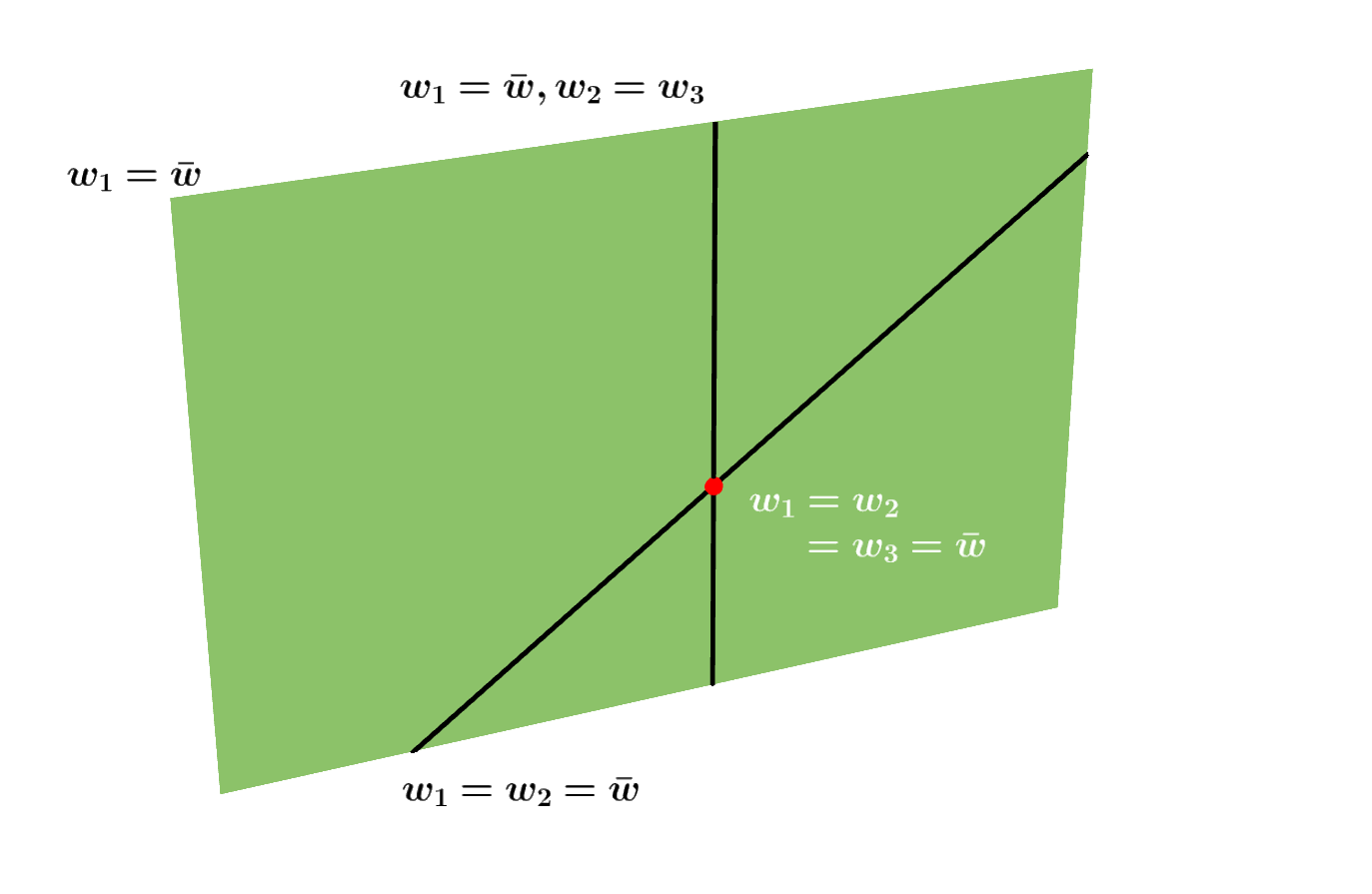}
    \end{minipage} %hfill
    \begin{minipage}[b]{0.49\textwidth}
        \includegraphics[width=\textwidth]{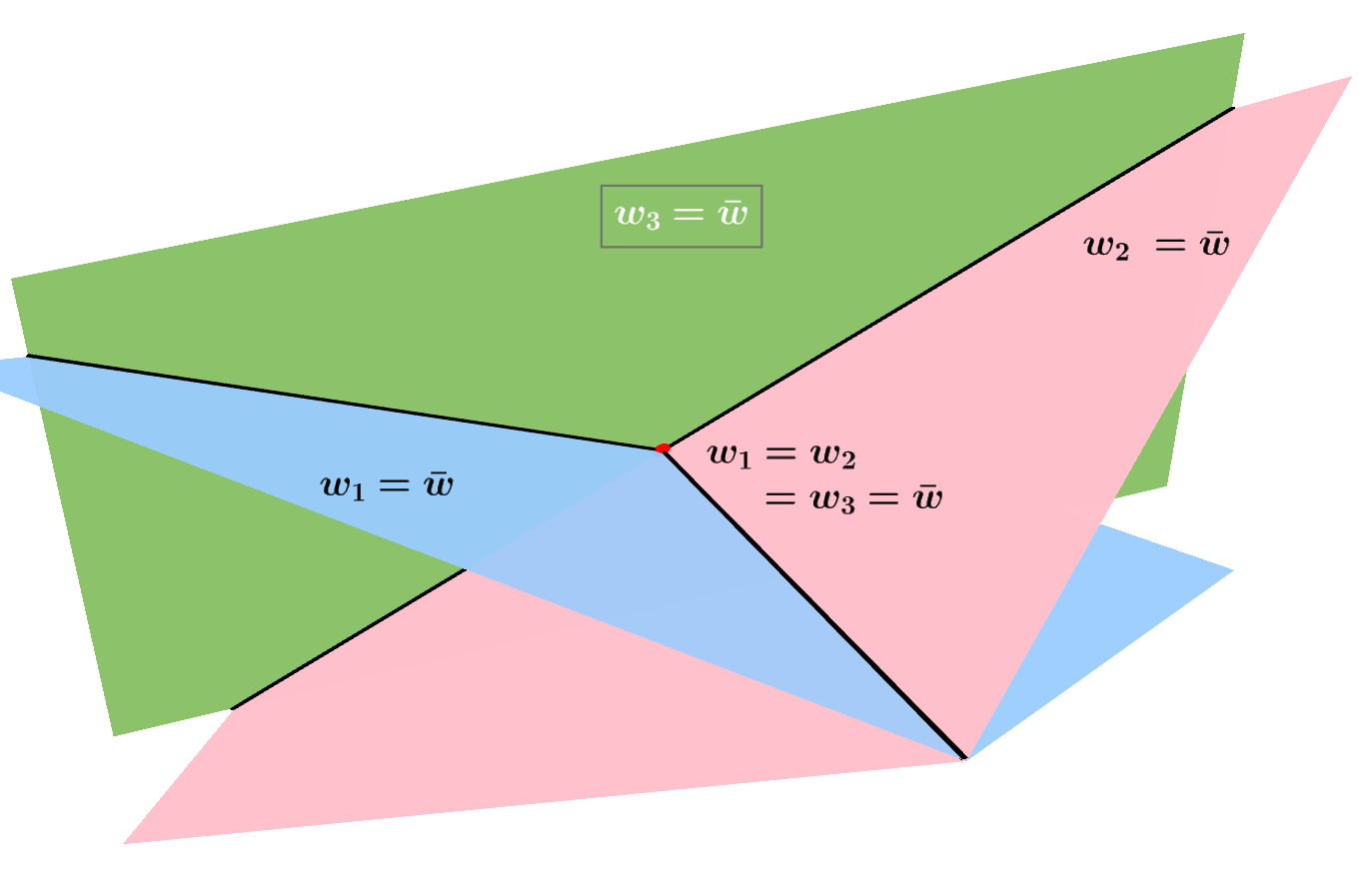}
    \end{minipage}
    \caption{Intersection of different $Q_P^r$'s, view from ``$w$-space". The left one shows the intersection of $Q_{\{0,1,2,3\}}^3$ (green surface), $Q_{\{0,1,3\}}^2$ (black line) and $Q_{\{0,3\}}^1$ (red dot). The right one shows the intersection of $Q_P^2$'s. Clearly $Q^2$ consists of three (geometrically) identical branches (green surface) with same $r$ but different permutation. Their intersections are blue lines and the red dot, which are also identical up to permutation.}
    \label{Fig: BranchIntersection_wSpace}
\end{figure}

\begin{figure}[thbp]
    \centering
    \includegraphics[width=0.5\textwidth]{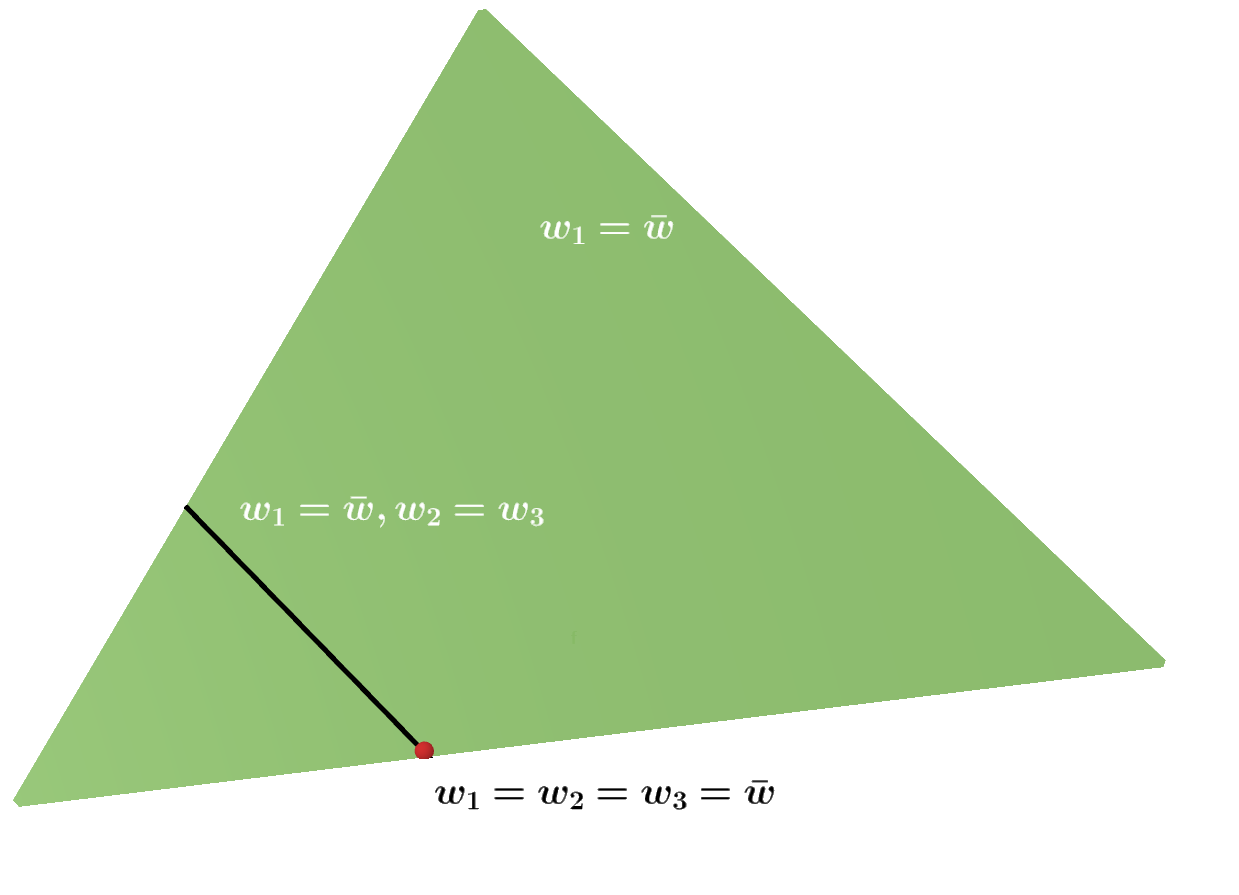}
    \caption{Intersection of branches $Q_{\{0,1,2,3\}}^3$ (green surface), $Q_{\{0,1,3\}}^2$ (tilted black line) and $Q_{\{0,3\}}^1$ (red dot), view from ``$a$-space".}
    \label{fig BranchIntersection_DistinctDim_aSpace}
\end{figure}

\begin{defn}[deficient number]
    Let $P = (q_0, ..., q_r)$ be a partition of $\{1, ..., m\}$. We say that the deficient number of $P$ is $l$ ($l \geq 0$) if there are precisely $t_1, ..., t_l \in \{m_0+1, ..., r\}$ such that $q_{t_j} - q_{t_j-1} = 1$ for all $1 \leq j \leq l$.
\end{defn} 

\vs{0.5em}

Clearly, if $P$ has deficient number $l > 0$ and $\theta^* = (a_k^*, w_k^*)_{k=1}^r\in Q_P^r$, then there are exactly $l$ distinct $w_k^*$'s whose corresponding $a_k^* = 0$. As we shall see in Section 3 and 4, the deficient number together with separating inputs provide information about the rank of the Hessian of $R$ at $Q_P^r$. 

\begin{lemma}\label{Lem range of deficient number}
    Let $P = (q_0, ..., q_r)$ be a partition of $\{1, ..., m\}$. If $r \le \frac{m+m_0}{2}$, the smallest possible deficient number of $P$ is $l = 0$. If $r > \frac{m+m_0}{2}$, the smallest possible deficient number of $P$ is $l = 2r - m - m_0$. 
\end{lemma}
\begin{proof}
    When $r \le \frac{m+m_0}{2}$ we can set $q_t = t$ for all $0 \le t \le m_0$ and $q_t - q_{t-1} \ge 2$ otherwise. Then clearly $l = 0$. \\

    When $r > \frac{m+m_0}{2}$, we must have $l \ge 1$. To get the smallest possible $l$, we need to make $q_{m_0}$ as small as possible in order to ``leave room for" non-unique $w_k^*$'s. Thus, still set $q_t = t$ for $0 \leq t \leq m_0$. This gives $q_k - q_{m_0} = m - m_0$. By definition of deficient number of $P$, there are $r - m_0 - l$ distinct $t > m_0$ satisfying $q_t - q_{t-1} \geq 2$, whence 
    \[
        2(k - m_0 -l) + l \leq m - m_0. 
    \]
    Solving the inequality, we see that $l \geq 2k - m - m_0$. This completes the proof. 
\end{proof}

\section{Loss Landscape Near \texorpdfstring{$Q^*$}{M}}\label{Section Loss landscape near Minfty}

\begin{lemma}[$R$ is analytic]\label{Lem Rmu is analytic}
    Let $\sigma$ be an analytic activation. Then $R$ defined as equation (\ref{Eq Definition of Rmu}) is analytic. 
\end{lemma}

Recall the definition of Morse--Bott functions, which are summarized from \cite{PFeehan}. They play an essential role in our study of the local convergence of gradient flow near each branch of $Q^*$. 

\begin{defn}[Morse--Bott function, rephrased from \cite{PFeehan}]\label{Morse--Bott function defn}
    Let $f: \bR^{(d+1)m} \to \bR$ be a smooth function and let $M \subseteq (\nabla f)^{-1}\{0\}$ be a non-empty submanifold of $\bR^{(d+1)m}$. Let $NM \to M$ be the normal bundle of $M$. Given $p \in M$, let $\operatorname{Hess}_M f(p): N_x M \times N_x M \to \bR$ be the restriction of $\operatorname{Hess} f(p)$ to $N_x M$. $M$ is called a non-degenerate critical manifold of $f$ if $\operatorname{Hess}_M f(p)$ is non-degenerate for each $p \in M$, and such $f$ is called Morse--Bott (at $M$) if there is a neighborhood of $M$ restricted to which $(\nabla f)^{-1}\{0\} = M$. 
\end{defn}

For our notation, we view $\mathrm{Hess}_M f(x)$ as a map in Definition \ref{Morse--Bott function defn}. Since each bilinear map corresponds to a matrix, we may also view $\mathrm{Hess}_M f(x)$ as a matrix, as we usually do for $\operatorname{Hess}f(x)$. In this case, $\mathrm{Hess}_M f(x)$ is simply the restriction of $\operatorname{Hess}f(x)$ to a subspace of $\bR^{(d+1)m}$. \\

Morse--Bott functions are not rare, the most trivial examples being $f(x,y) = x^n$, $n = 2, 3, 4, ...$. Another important example is analytic function whose set of critical points are submanifold(s) and whose Hessian are non-degenerate at the normal spaces of it. Below we give more examples of them with the focus on neural network type functions.
\begin{itemize}
    \item [(a)] The linear regression model with $L^2$ loss
    \[
        R (\theta) = R(w) = \sum_{i=1}^n \left| g(w, x_i) - y_i \right|^2, \quad g(w, x_i) = w \cdot x
    \]
    is clearly a Morse--Bott function. This is because $\operatorname{Hess}R$ is a constant matrix-valued function, which implies that $(\nabla R)^{-1}\{0\}$ is a union of affine spaces which are all orthogonal to $\ker(\operatorname{Hess}R)$. 

    \item [(b)] Consider $\sigma(x) = e^x$ and $m = n =1$, namely, 
    \[
        R(\theta) = (g(\theta, x) - y)^2, \quad g(\theta, x) = a e^{w\cdot x}, 
    \]
    where $x$ and $y$ are chosen and fixed. Since 
    \[
        \parf{R}{a} = 2 \left( a e^{w\cdot x} - y \right) e^{w\cdot x}, \quad \parf{R}{w} = 2a (a e^{w\cdot x} - y)e^{w\cdot x} x
    \]
    we can see that $\nabla R(\theta) = 0$ if and only if $\theta \in \Rmuzero$, if and only if $a = y e^{-w\cdot x}$. This shows that $M :=(\nabla R)^{-1}\{0\} = \Rmuzero$ and is a submanifold of codimension 1 in $\bR^{d+1}$. Then we can calculate the Hessian of $R$ on it: 
    \[
        \operatorname{Hess}R(\theta) = 2
        \begin{pmatrix}
            e^{2w\cdot x}       & a e^{2w\cdot x} x_1   &... & ae^{2w\cdot x} x_d \\ 
            a e^{2w\cdot x} x_1 & a^2 e^{2w\cdot x}x_1^2  &... & a^2 e^{2w\cdot x}x_1 x_d \\ 
            \vdots              &\vdots                 &\ddots &\vdots \\ 
            a e^{2w\cdot x} x_d & a^2 e^{2w\cdot x}x_1 x_d &... & a^2 e^{2w\cdot x}x_d^2
        \end{pmatrix}. 
    \]
    To simplify, if $v$ denotes the first column of $R(\theta)$, then $\operatorname{Hess}R(\theta) = (v, ax_1 v, ..., ax_d v)$, whence it has constant rank 1 on $M$. Since $T_\theta M \siq \ker \operatorname{Hess}R(\theta)$, we see that $\operatorname{Hess}_M R(\theta)$ must be non-degenerate. 

    \item [(c)] In fact, as we shall see in Lemma \ref{separated branch}, under our finite-feature setting of 2-layer neural networks (see Assumption \ref{finite-feature setting}), for sufficiently many samples (number depending on $r \in \{m_0, ..., m\}$), $R$ is Morse--Bott on a dense, relatively open subset of $Q_P^r$ for $r \geq (m+m_0)/2$ and many choices of the partition $P$. \\
\end{itemize}

\begin{figure}[H]
    \centering
    \begin{minipage}[b]{0.5\textwidth}
        \includegraphics[width=\textwidth]{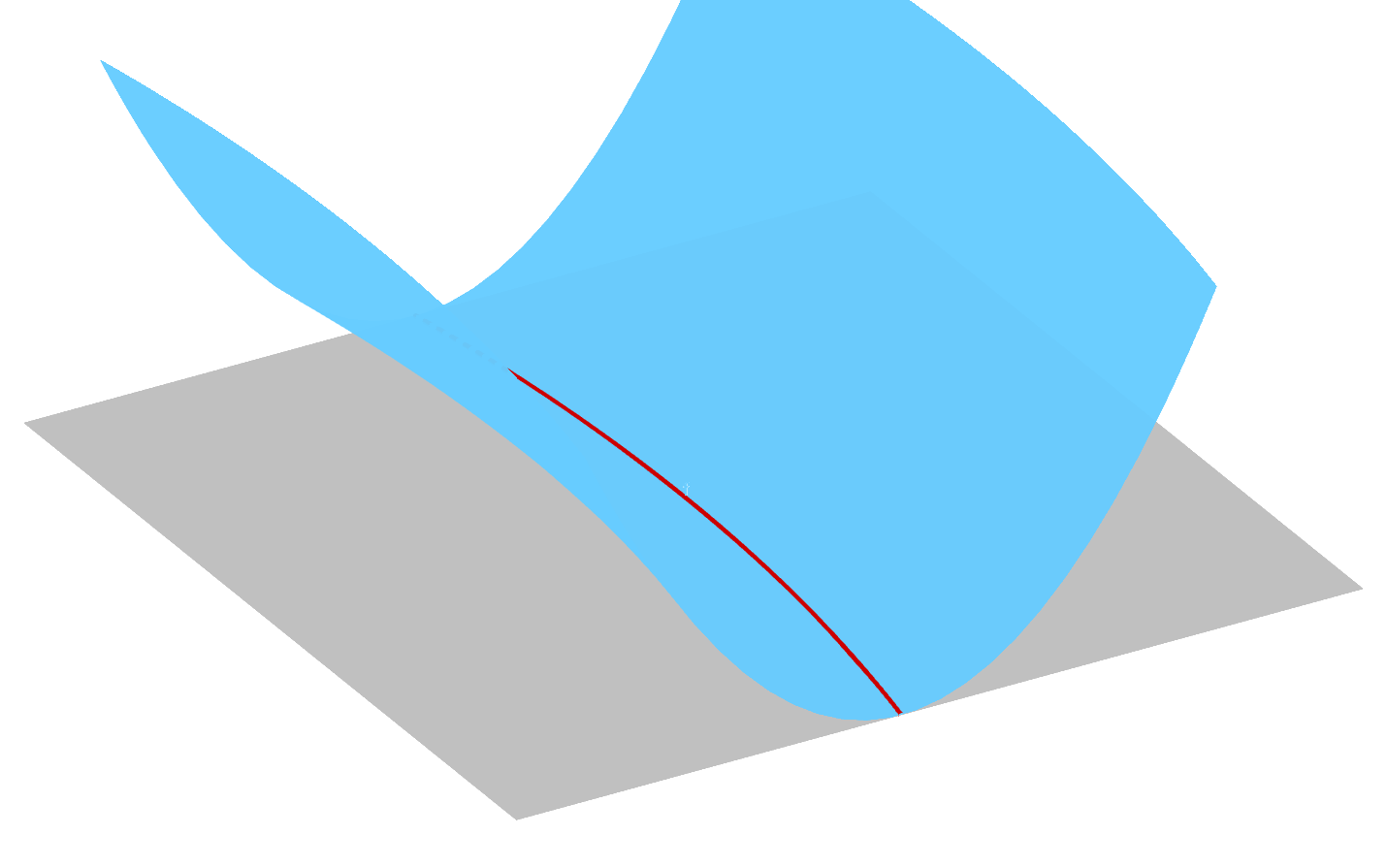}
    \end{minipage} %hfill
    \begin{minipage}[b]{0.4\textwidth}
        \includegraphics[width=\textwidth]{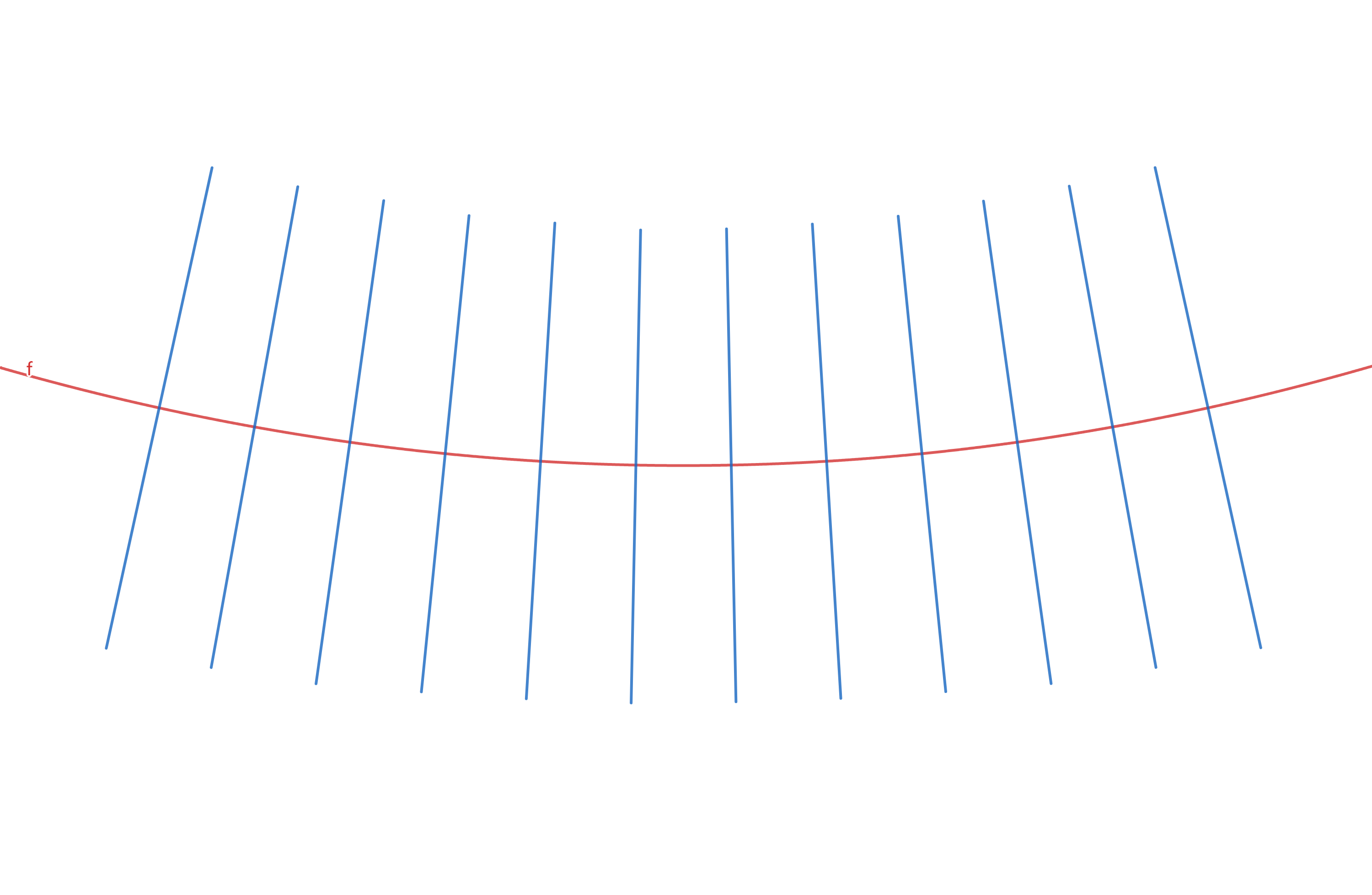}
    \end{minipage}
    \caption{Illustration of $f(x,y) = (x^2 - y)^2$. The left one shows the graph of this function as a submanifold in $\bR^3$ (blue manifold) and its zero set (red curve). The right one shows $f^{-1}\{0\} \siq \bR^2$, the blue lines are the normal bundles along it, along which $\mathrm{Hess} f$ is non-degenerate.}
\end{figure}

We then investigate how $\operatorname{Hess}R$ near $Q^*$ depends on the choice of samples and the branches $Q_P^r$. In fact, we will give sample size thresholds (i) for branches to be separated from the imperfect global minima $\Rmuzero \cut Q^*$, and (ii) for certain branches at which the loss $R$ becomes Morse--Bott. Meanwhile, we show that for other branches, they are provably not Morse--Bott once they are separated from the imperfect global minima. Therefore, we obtain a hierarchical (with respect to sample size) characterization of $\Rmuzero$ near $Q^*$ when the system is overparametrized.

\begin{lemma}[separation of $Q^*$ -- overparametrized case]\label{separated branch}
    Suppose that Assumptions \ref{generic activation} and \ref{finite-feature setting} hold. Let $n \in \bN$. Suppose that $\{x_i\}_{i=1}^{(d+1)m}$ are separating inputs for $(\bar{w}_1, ..., \bar{w}_m)$ for an arbitrary choice of distinct $\barw_{m_0}, ..., \barw_m \in \bR^d\cut\{0, \barw_k\}_{k=1}^{m_0}$. Given $m_0 \leq r \leq m$ and $1 \le n \leq (d+1)m$, define $R(\theta) = \sum_{i=1}^n |g(\theta, x_i) - f^*(x_i)|^2$ as before. Then the following results hold.
    \begin{itemize}
        \item [(a)] When $n \leq (d+1)m_0$, any $\theta^* \in Q^*$ has an open neighborhood $U$ such that $R$ is Morse--Bott at $U \cap \Rmuzero$, namely, if $M:= U \cap \Rmuzero$ then for any $\ttheta^* \in U \cap \Rmuzero$, $\mathrm{Hess}_M R(\ttheta^*)$ is non-degenerate. 
        
        \item [(b)] Let $P = (q_0, ..., q_r)$ be a partition with deficient number $l$. When $n \leq r + (r-l)d$, there is an open $U \siq \bR^{(d+1)m}$ such that $R$ is Morse--Bott at $U \cap \Rmuzero$, with $U \cap Q_P^r \siq U \cap \Rmuzero$. When $n \ge r + (m + m_0 - r)d$, up to an arbitrarily small perturbation of samples, we have an open $U \siq \bR^{(d+1)m}$ such that $R$ with $U \cap \Rmuzero = U \cap Q_P^r$. 
        
        \item [(c)] Let $P$ be a partition with deficient number $l$. Suppose that $n \geq r + (m+m_0-r)d$, the samples $\{x_i\}_{i=1}^n$, and open $U\siq \bR^{(d+1)m}$ are chosen so that (b) holds. When $r < (m+m_0)/2$, $R$ is not Morse--Bott at $U \cap Q_P^r$. When $r \geq (m+m_0)/2$, the samples and $U$ can be chosen so that $R$ becomes Morse--Bott at $U \cap Q_P^r$ if and only if $l = 2r - m - m_0$. 
    \end{itemize}
\end{lemma}
\begin{proof}~
\begin{itemize}
    \item [(a)] By Proposition \ref{prop fragmentation/stratification}, $Q^* = \bigcup_{r=m_0}^m \bigcup_{P,\pi} Q_{P,\pi}^r$, so there are $k_1, ..., k_{m_0} \in \{1, ..., m\}$ such that $a_{k_t}^* \neq 0$ and $w_{k_t}^* = \barw_t$ for all $1 \leq t \leq m_0$. By hypothesis, $\{x_i\}_{i=1}^{(d+1)m}$ are separating inputs for $\barw_1, ..., \barw_m$. Thus, for each $n \le (d+1)m$, $\{x_i\}_{i=1}^n$ are separating inputs for $\barw_1, ..., \barw_m$; in particular, they are separating inputs for $\barw_1, ..., \barw_{m_0}$ as well. This implies that for any $1 \leq i \leq n$, 
    \[
        \nabla_\theta (g(\theta^*,x_i) - f^*(x_i)) =: h_i(\theta^*) \neq 0
    \] 
    and thus 
    \[
        \mathrm{Hess}\,(g(\theta^*, x_i) - f^*(x_i))^2 = h_i(\theta^*) h_i(\theta^*)^\TT
    \]
    has rank 1, with $h_i(\theta^*)$ being an eigenvector for the only non-zero eigenvalue of this Hessian (because $\theta^* \in \Rmuzero$; in fact, this holds for any $\ttheta^* \in \Rmuzero$). Also, it implies that
    \begin{align*}
        &\,\,\,\,\,\,\, \operatorname{rank} 
        \begin{pmatrix}
            \sigma(w_{k_1}^*\cdot x_1) &... &\sigma(w_{k_{m_0}}^*\cdot x_1) &\sigma'(w_{k_1}^*\cdot x_1)x_1^\TT &... &\sigma'(w_{k_{m_0}}^*\cdot x_1)x_1^\TT\\ 
            \vdots                 &\ddots &\vdots                  &\vdots                         &\ddots &\vdots \\ 
            \sigma(w_{k_1}^*\cdot x_n) &... &\sigma(w_{k_{m_0}}^*\cdot x_n) &\sigma'(w_{k_1}^*\cdot x_n)x_n^\TT &... &\sigma'(w_{k_{m_0}}^*\cdot x_n)x_n^\TT
        \end{pmatrix} \\ 
        &= \operatorname{rank}  
       \begin{pmatrix}
            \sigma(\barw_1\cdot x_1) &... &\sigma(\barw_{m_0}\cdot x_1) &\sigma'(\barw_1\cdot x_1)x_1^\TT &... &\sigma'(\barw_{m_0}\cdot x_1)x_1^\TT\\ 
            \vdots                 &\ddots &\vdots                  &\vdots                         &\ddots &\vdots \\ 
            \sigma(\barw_1\cdot x_n) &... &\sigma(\barw_{m_0}\cdot x_n) &\sigma'(\barw_1\cdot x_n)x_n^\TT &... &\sigma'(\barw_{m_0}\cdot x_n)x_n^\TT
        \end{pmatrix} \\ 
        &= n. 
    \end{align*}
    Therefore, $\{h_i(\theta^*)\}_{i=1}^n$ is a linearly independent set. By continuity of the $h_i$'s, there is an open neighborhood $U$ of $\theta^*$ such that $\{h_i(\theta)\}_{i=1}^n$ is linearly independent for any $\theta \in U$. This implies that $M := U \cap \Rmuzero$ is the transverse intersection of $n$ codimension-1 submanifolds of $U$: 
    \[
        (g(\cdot, x_i) - f^*(x_i))|_{U}^{-1}\{0\}, \quad 1 \leq i \leq n. 
    \]
    Therefore, $M$ has codimension $n$. On the other hand, when $\ttheta^* \in U \cap \Rmuzero$, each $h_i(\ttheta^*)$ is an eigenvector for the only non-zero eigenvalue of $\mathrm{Hess}\,(g(\ttheta^*, x_i) - f^*(x_i))^2$. Since 
    \[
        \mathrm{Hess}\, R(\ttheta^*) = \sum_{i=1}^n \mathrm{Hess}\,(g(\ttheta^*, x_i) - f^*(x_i))^2
    \] 
    the linear independence of $h_i(\ttheta^*)$'s yields $\operatorname{rank}\left(\mathrm{Hess} R(\ttheta^*)\right) = n$. It follows that $\mathrm{Hess}_M R(\ttheta^*)$ is non-degenerate for each $\ttheta^* \in M$. Thus, $R$ is Morse--Bott at $M = U \cap \Rmuzero$. 
    
    \item [(b)] First assume that $n \leq r + (r-l)d$. Without loss of generality, we may let $q_t = t$ for all $m - l \leq t \leq m$; the general case can be reduced to this one by a rearrangement of indices. Since $r\le m$ and $r-l \le m$ and since the $\{x_i\}_{i=1}^n$ are separating inputs for $\barw_1, ..., \barw_m$, $\{x_i\}_{i=1}^n$ are both separating inputs for $\barw_1, ..., \barw_r$ and $\barw_1, ..., \barw_{r-l}$. 
    
    Let $\theta^* = (a_k^*, w_k^*)_{k=1}^m \in Q_P^r$ be such that
    \begin{itemize}
        \item [i)] $w_{q_t}^* = \barw_{q_t}$ for all $1 \leq t \leq m$. 
        \item [ii)] For any $1 \leq t \leq m - l$, there is some $k_t \in \{q_{t-1}+1, ..., q_t\}$ with $a_{k_t}^* \neq 0$. 
    \end{itemize}
    For simplicity of our computation below, we may also set $k_t = t$ when $m - l \leq t \leq m$. Then $a_{k_t}^* \neq 0$ if and only if $t \leq m - l$. Thus, by our hypothesis on the separating inputs, we have 
    \begin{align*}
        &\,\,\,\,\,\,\, \operatorname{rank}
        \begin{pmatrix}
            \sigma(w_1^*\cdot x_1) &... &\sigma(w_m^*\cdot x_1) &a_1^*\sigma'(w_1^*\cdot x_1)x_1^\TT &... &a_m^*\sigma'(w_m^*\cdot x_1)x_1^\TT\\ 
            \vdots                 &\ddots &\vdots                  &\vdots                         &\ddots &\vdots \\ 
            \sigma(w_1^*\cdot x_n) &... &\sigma(w_m^*\cdot x_n) &a_1^*\sigma'(w_1^*\cdot x_n)x_n^\TT &... &a_m^*\sigma'(w_m^*\cdot x_n)x_n^\TT
        \end{pmatrix} \\ 
        &= \operatorname{rank}
        \begin{pmatrix}
            \sigma(\barw_1\cdot x_1) &... &\sigma(\barw_r\cdot x_1) &a_{k_1}^*\sigma'(\barw_1\cdot x_1)x_1^\TT &... &a_{k_r}^*\sigma'(\barw_r\cdot x_1)x_1^\TT\\ 
            \vdots                 &\ddots &\vdots                  &\vdots                         &\ddots &\vdots \\ 
            \sigma(\barw_1\cdot x_n) &... &\sigma(\barw_r\cdot x_n) &a_{k_1}^*\sigma'(\barw_1\cdot x_n)x_n^\TT &... &a_{k_r}^*\sigma'(\barw_r\cdot x_n)x_n^\TT
        \end{pmatrix} \\ 
        &= \operatorname{rank}
        \begin{pmatrix}
            \sigma(\barw_1\cdot x_1) &... &\sigma(\barw_r\cdot x_1) &a_{k_1}^*\sigma'(\barw_1\cdot x_1)x_1^\TT &... &a_{k_{r-l}}^*\sigma'(\barw_{r-l}\cdot x_1)x_1^\TT\\ 
            \vdots                 &\ddots &\vdots                  &\vdots                         &\ddots &\vdots \\ 
            \sigma(\barw_1\cdot x_n) &... &\sigma(\barw_r\cdot x_n) &a_{k_1}^*\sigma'(\barw_1\cdot x_n)x_n^\TT &... &a_{k_{r-l}}^*\sigma'(\barw_{r-l}\cdot x_n)x_n^\TT
        \end{pmatrix} \\ 
        &= \operatorname{rank}
        \begin{pmatrix}
            \sigma(\barw_1\cdot x_1) &... &\sigma(\barw_r\cdot x_1) &\sigma'(\barw_1\cdot x_1)x_1^\TT &... &\sigma'(\barw_{r-l}\cdot x_1)x_1^\TT\\ 
            \vdots                 &\ddots &\vdots                  &\vdots                         &\ddots &\vdots \\ 
            \sigma(\barw_1\cdot x_n) &... &\sigma(\barw_r\cdot x_n) &\sigma'(\barw_1\cdot x_n)x_n^\TT &... &\sigma'(\barw_{r-l}\cdot x_n)x_n^\TT
        \end{pmatrix} \\ 
        &= n. 
    \end{align*} 
    Argue in the same way as we do in (a), we show that $\operatorname{rank}\left(\mathrm{Hess} R(\theta^*)\right) = n$. Also, note that this is the largest possible rank of $\operatorname{Hess}R(\theta^*)$ for $\theta^* \in Q_P^r$. By continuity, $\theta^* \in Q_P^r$ has an (open) neighborhood $U \siq \bR^{(d+1)m}$ such that $\operatorname{rank}\left(\mathrm{Hess} R(\theta)\right) \geq n$ for all $\theta \in U$. On the other hand, the computation above shows that $M := U \cap \Rmuzero$ is the transverse intersection of $n$ submanifolds in $\bR^{(d+1)m}$, whence a submanifold having codimension $n$. By the proof in (a), it follows that for any $\ttheta^* \in M$, $\mathrm{Hess}_M R(\ttheta^*)$ is non-degenerate, namely, $R$ is Morse--Bott at $U \cap \Rmuzero$. 

    Now assume that $n \geq r + (m + m_0 - r)d$. Following the proof of Corollary \ref{square-loss type perturbation} with $E = Q^*$ and $s' = r + (m+m_0-r)d$, we can perturb $x_1, ..., x_{s'}$ arbitrarily small so that $\calA_{Q^*, s'}$ is contained in a countable (possibly finite) union of analytic submanifolds $\calM_1, \calM_2, ...$ of $\bR^{(d+1)m}$ with codimension $s' = r + (m+m_0-r)d$. By Proposition \ref{prop fragmentation/stratification} (c), $\overline{Q_P^r}$ is an affine subspace of codimension 
    \[
        (d+1)m - [(m-r) + (r-m_0)d] = r + (m+m_0-r)d, 
    \]
    and each $\theta^* \in Q_P^r$ has an open neighborhood $V$ such that $V \cap Q_P^r = V \cap \overline{Q_P^r}$. Thus, for any such $V$, $V \cap Q_P^r$ is a submanifold of codimension $r + (m+m_0-r)d$. In other words, the zero set of $R|_{V}$, the restriction of $R$ to $V$, is contained in a countable union of submanifolds of codimension $r + (m+m_0-r)d$. Now by Lemma \ref{zero set of multivariable analytic function} (c), we can find an open $U \siq V$ such that $U \cap \Rmuzero = U \cap Q_P^r \neq \emptyset$. 
    
    \item [(c)] Because $R$ is constant zero at $U \cap Q_P^r$, $R$ is Morse--Bott at $U \cap Q_P^r$ if and only if $\operatorname{rank}\left(\mathrm{Hess} R\right)= \operatorname{codim}(U \cap Q_P^r)$, which is also the largest possible rank of $\operatorname{Hess}R$. We know that $\operatorname{rank}\left(\mathrm{Hess} R(\theta^*)\right) \leq r + (r-l)d$. Thus, if $R$ is Morse--Bott at $U \cap Q_P^r$, we must have 
    \[
        r + (r - l)d \geq \operatorname{codim} (U \cap Q_P^r) = r + (m + m_0 - r)d, 
    \]
    which yields $l \leq 2r - m - m_0$. Since $l \geq 2r - m - m_0$ (recall Lemma \ref{Lem range of deficient number}) and obviously $l \geq 0$, we conclude that $l = 2r - m - m_0$ and $r \geq (m+m_0)/2$. In other words, when $r < (m+m_0)/2$ or $l > 2r - m - m_0$, we can never make $R$ to be Morse--Bott at $U \cap Q_P^r$, no matter what samples and what $U$ we choose. 

    To prove the remaining part of (c), let $r \geq (m+m_0)/2$ and let $l = 2r - m - m_0$ and $n \geq r + (m + m_0 - r)d = r + (r-l)d$. By (b), there is some $\theta^* = (a_k^*, w_k^*)_{k=1}^m \in Q_P^r$ such that the separating inputs $\{x_i\}_{i=1}^{n'}$, $n' = r + (r-l)d$, satisfies 
    \begin{align*}
    &\,\,\,\,\,\,\,\operatorname{rank} 
    \begin{pmatrix}
        \sigma(w_1^*\cdot x_1) &... &\sigma(w_m^*\cdot x_1) &a_1^*\sigma'(w_1^*\cdot x_1)x_1^\TT &... &a_m^*\sigma'(w_m^*\cdot x_1)x_1^\TT\\ 
        \vdots                 &\ddots &\vdots                  &\vdots                         &\ddots &\vdots \\ 
        \sigma(w_1^*\cdot x_n) &... &\sigma(w_m^*\cdot x_n) &a_1^*\sigma'(w_1^*\cdot x_{n'})x_{n'}^\TT &... &a_m^*\sigma'(w_m^*\cdot x_{n'})x_{n'}^\TT
    \end{pmatrix} \\ 
    &= r + (m+m_0-r)d. 
    \end{align*}
    Since $\Rmuzero$ is contained in the zero set of the function $\theta \mapsto \sum_{i=1}^{n'} |g(\theta, x_i) - f(x_i)|^2$, we must have 
    \begin{align*}
    &\,\,\,\,\,\,\,\operatorname{rank} 
    \begin{pmatrix}
        \sigma(w_1^*\cdot x_1) &... &\sigma(w_m^*\cdot x_1) &a_1^*\sigma'(w_1^*\cdot x_1)x_1^\TT &... &a_m^*\sigma'(w_m^*\cdot x_1)x_1^\TT\\ 
        \vdots                 &\ddots &\vdots                  &\vdots                         &\ddots &\vdots \\ 
        \sigma(w_1^*\cdot x_n) &... &\sigma(w_m^*\cdot x_n) &a_1^*\sigma'(w_1^*\cdot x_n)x_n^\TT &... &a_m^*\sigma'(w_m^*\cdot x_n)x_n^\TT
    \end{pmatrix} \\ 
    &= r + (r - l)d = r + (m + m_0 - r)d. 
    \end{align*}
    Then, as in (b), there is an open neighborhood $U$ of $\theta^*$ such that $U \cap \Rmuzero = U \cap Q_P^r$ and $R$ is Morse--Bott at $U \cap \Rmuzero$. 
\end{itemize}
\end{proof}

\begin{remark}
    All the three parts of lemma \ref{separated branch} can be strengthened. First, the results in (a) can be strengthened as: $R$ restricted to some open $U \siq \bR^{(d+1)m}$ containing $Q^*$ is Morse--Bott at $U \cap \Rmuzero$. This can be proved by simply taking the union of the neighborhoods we construct for points in $Q^*$. For (b) and (c), we note that  $U \siq \bR^{(d+1)m}$ can be chosen so that $U \cap Q_P^r$ is a dense subset of $Q_P^r$. Indeed, the set of $\theta^* \in Q_P^r$ such that $a_{k_t}^* \neq 0$ for some $q_{t-1} < k_t \leq q_t$ whenever $t \leq m_0$ or $q_t - q_{t-1} > 1$ is dense in $Q_P^r$, and each such point has a neighborhood on which (b) and/or (c) hold (for this we call $Q_P^r$ \textit{separated/Morse--Bott a.e.}). Therefore, we can simply let $U$ be the union of these neighborhoods. Moreover, the results in (b) and (c) also hold on $Q_{P,\pi}^r$, for any permutation $\pi$. This is because the geometry of $Q_{P,\pi}^r$ is the same as that of $Q_P^r$.\\
\end{remark}

\begin{figure}[H]
    \centering
    \includegraphics[trim={0 0 0 4cm}, clip, width=0.65\textwidth]{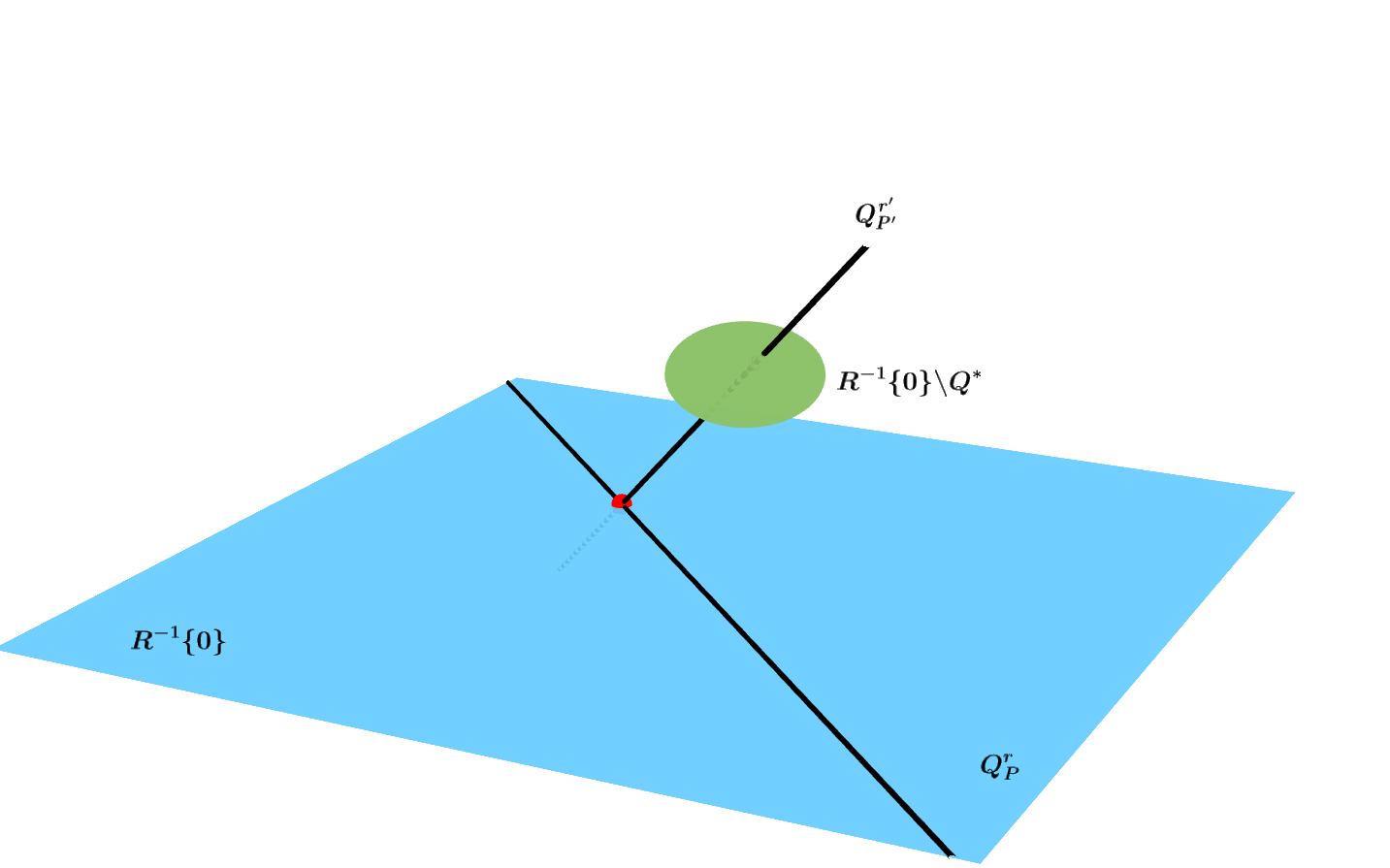}
    \caption{Separation of branches in overparameterized regime. In the figure, the $r$ and $r'$ are given. As illustrated above, there is some open $U \siq \bR^{(d+1)m}$ with $U \cap \Rmuzero = U \cap Q_{P'}^{r'}$, while this does not hold for $Q_P^r$. Also note that when $n \ge r' + (m+m_0-r')d$, $Q_{P'}^{r'}$ we can only guarantee that it is separated a.e. (e.g., not at points covered by the green disk).}
    %\label{fig:enter-label}
\end{figure}

\begin{lemma}[separation of $Q^*$ -- underparametrized case]\label{underparametrized system} 
    Suppose that Assumptions \ref{generic activation} and \ref{finite-feature setting} hold. 
    \begin{itemize}
        \item [(a)] For any $N \in \bN$ and any open $O \soq Q^*$, there is a finite collection of samples such that $\overline{B(0,N)} \cap \Rmuzero$ is contained in $O$. 
        
        \item [(b)] For any collection of inputs $\{x_i\}_{i=1}^n$ with $n > (d+1)m$ and any $\vep > 0$, we can perturb each $x_i$ with no more than $\vep$-distance to obtain a collection of new inputs $\{x_i'\}_{i=1}^n$, so that $\Rmuzero = Q^*$. In particular, for any $\theta = (a_k, w_k)_{k=1}^m \in \bR^{(d+1)m}$ with $w_1, ..., w_m \ne 0$ and any $n > (d+1)m$, almost all separating inputs $\{x_i'\}_{i=1}^n$ for $\theta$ make $\Rmuzero = Q^*$. 
    \end{itemize}
\end{lemma}
\begin{remark}
    Obviously, (a) follows from (b), but to demonstrate different techniques we use two ways to prove them. 
\end{remark}
\begin{proof}
\begin{itemize}
    \item [(a)] Fix $\theta \in \overline{B(0,N)} \cut O$, which is a compact subset of $\bR^{(d+1)m}$. Since $\theta \notin Q^*$, the set $\{x \in \bR^d: g(\theta, x) = f(x)\}$ has $\lambda_d$-measure zero. Thus, we can find some $x_\theta$ such that $g(\theta, x_\theta) \neq f(x_\theta)$; it follows by continuity that $g(\ttheta, x_\theta) \neq f(x_\theta)$ on an open subset $B_\theta$ of $\overline{B(0,N)}$. By compactness, we can find $\theta_1, ..., \theta_n \in \overline{B(0,N)}\cut O$ such that $\overline{B(0,N)}\cut O \siq \bigcup_{i=1}^n B_{\theta_i}$. Set $R = \sum_{i=1}^n |g(\theta, x_{\theta_i}) - f(x_{\theta_i})|^2$, then $\overline{B(0,N)} \cap \Rmuzero$ is contained in $O$. 
    
    \item [(b)] Let $E := \bR^{(d+1)m} \cut Q^*$. The function $(\theta, x) \mapsto g(\theta, x) - f(x)$ satisfies: for any $\theta \in E$, $x \mapsto g(\theta, x) - f(x)$ is not constant-zero. By Corollary \ref{square-loss type perturbation}, for almost all inputs $\{x_i'\}_{i=1}^n$ and $R(\theta) = \sum_{i=1}^n |g(\theta, x_i') - f^*(x_i')|^2$, we have $E \cap \Rmuzero = \emptyset$, or $\Rmuzero = Q^*$. Given such $\theta = (a_k, w_k)_{k=1}^m \in \bR^d$, the set of separating inputs for it is an open dense full-measure subset of $\bR^{dm}$, whence its intersection with the inputs making $\Rmuzero = Q^*$ is also a full-measure subset of $\bR^{dm}$. In particular, this means when $n \ge (d+1)m$, almost all separating inputs $\{x_i'\}_{i=1}^n$ makes $\Rmuzero = Q^*$.
\end{itemize}
\end{proof}

The discussion above are all based on the assumption that we know exactly what the target function is. In general, given samples $\{(x_i, y_i)\}_{i=1}^n$, the target functions $f^*$ with $f^*(x_i) = \sum_{k=1}^{m_0} \bara \sigma(\barw \cdot x_i) = y_i$ for all $i$ may not be unique. Luckily, the following proposition guarantees the uniqueness of target function $f^*$ for a dense set of separating inputs at overparameterization. Thus, in general we only need to deal with one fixed target function. 

\begin{prop}[uniqueness of representation]\label{isolated global minimum} 
    Suppose that Assumptions \ref{generic activation} and \ref{finite-feature setting} hold. Given $n > (d+1)m_0$, for almost all separating inputs $\{x_i'\}_{i=1}^n$ for $\barw_1, ..., \barw_{m_0}$, $\bar{\theta} = (\bara_k, \barw_k)_{k=1}^{m_0}$ is the unique global minimum of
    \begin{equation*}
        \bR^{(d+1)m_0} \ni (a_k, w_k)_{k=1}^{m_0} \mapsto \sum_{i=1}^n \left|\sum_{k=1}^{m_0} a_k \sigma(w_k \cdot x_i) - f^*(x_i) \right|^2. 
    \end{equation*}
\end{prop}
\begin{proof}
    Let $\{x_i\}_{i=1}^n$ be separating inputs for $\barw_1, ..., \barw_{m_0}$, using Lemma \ref{separated branch}. Then use Lemma \ref{underparametrized system} to perturb the inputs $\{x_i\}$ arbitrarily small to obtain $\{x_i'\}_{i=1}^n$ so that $\bar{\theta}$ becomes the unique global minimum of this function. Arguing in the same way as in Lemma \ref{underparametrized system}, almost all separating inputs have this property. 
\end{proof}

\section{Dynamics of Gradient Flow Near \texorpdfstring{$Q^*$}{M}}\label{Section Dynamics of gradient flow near Minfty} 

Based on the geometry of $Q^*$ and functional properties of $R$ which we characterized above, we are now able to give the complete characterization of gradient flows near $Q^*$. In this section, we apply Lojasiewicz type inequalities to show the convergence of gradient flow near the critical points of a real analytic function $f$. Then we discuss whether each point in $(\nabla f)^{-1}\{0\}$ is the limit of some gradient flow of $f$. By Assumption \ref{finite-feature setting} and/or our discussion in Section \ref{Section Loss landscape near Minfty}, all of them hold for $R$. Based on these results and Lemma \ref{separated branch}, we further characterize the convergence rate, limiting direction, and generalization stability (whether $g(\lim_{t\to+\infty} \gamma(t),\cdot )$ is stable under perturbation of $\lim_{t\to+\infty} \gamma(t)$, see also Definition \ref{Defn Recovery stability}) of gradient flow near $Q^*$. Thus we prove Theorem \ref{Thm gradient flow properties local, informal}, with a detailed understanding of the behavior of training dynamics near $\Rmuzero$. 

\subsection{Limiting Set of Gradient Flow} 

The following theorems, which can be seen as different types of Lojasiewicz inequality are summarized from \cite{PFeehan} and \cite{PAbsil}. They show that for an analytic function, any gradient flow near a local minimum converges. 

\begin{thm}[Theorem 1 of $\text{\cite{PFeehan}}$]\label{Lojasiewicz inequality for analytic and GMB function}
    Let $f: \bR^{(d+1)m} \to \bR$ be a real analytic function. For any critical point $p$ of $f$, there is a neighborhood $U$ of $p$ and constants $C > 0$, $\mu \in [1/2, 1)$ such that 
    \begin{equation}
        |\nabla f(q)| \geq C |f(q) - f(p)|^\mu 
    \end{equation}
    for any $q \in U$. 
\end{thm}

\begin{thm}[rephrased from Theorem 2.2 in $\text{\cite{PAbsil}}$]\label{convergence rate of analytic or GMB function}
    Let $f: \bR^{(d+1)m} \to \bR$ be a real analytic function. Suppose that $p$ is a local minimum of $f$. Then $p$ has a neighborhood $U$ such that any non-constant gradient flow with initial value $x_0 \in U$ converges to $f^{-1}\{f(p)\}$. Moreover, any such gradient flow converges with some rate $0 < \beta < 1$ depending only on $p$. Namely, the curve length 
    \begin{equation}
        l(\gamma[t, +\infty)) = O(|f(\gamma(t)) - f(p)|^\beta)
    \end{equation}
    as $t \to +\infty$, where $l(\gamma[t, +\infty))$ is the curve length of $\gamma[t, +\infty)$. 
\end{thm} 
\begin{proof}
    By Theorem 2.2 in \cite{PAbsil}, there is some neighborhood $U \siq \bR^{(d+1)m}$ around $p$ such that for any gradient flow $\gamma: [0, +\infty) \to \bR^{(d+1)m}$ with any $\gamma(0) = x_0 \in U$, $\gamma$ satisfies: there are some $c, \beta > 0$ such that for any $0 < t_1 < t_2 < +\infty$, 
    \[
        \int_{t_1}^{t_2} \left| \dot{\gamma}(t)\right| dt \leq c |f(\gamma(t_1)) - f(p)|^\beta. 
    \]
    Taking $t_2 \to +\infty$, the monotonic convergence theorem gives $l(\gamma[t, +\infty)) \leq c|f(\gamma(t_1) - f(p)|^\beta$, as desired. This shows the length of $\gamma$ is bounded, whence $\limftyt \gamma(t)$ exists. In particular, there is some neighborhood $V \siq U$ of $p$ such that gradient flows with initial value in $V$ lies in $U$ eventually. Since $p$ is a local minimum of $f$, the neighborhood $U$ can be chosen so that $f(q) \ge f(p)$, and (by Theorem \ref{Lojasiewicz inequality for analytic and GMB function}) $|\nabla f(q)| \ge C|f(q) - f(p)|^\mu$. Thus, $\nabla f(q) = 0$ is possible only for $q \in f^{-1}\{f(p)\}$, which implies that $\limftyt \gamma(t) \in f^{-1}\{f(p)\}$ whenever $\gamma$ has initial value $\gamma(0) \in V$. 
\end{proof}

Thus, if $f$ is non-negative and $f^{-1}\{0\}$ is non-empty, then there is some open $U \siq \bR^{(d+1)m}$ such that any gradient flow with initial value in $U$ converges to $f^{-1}\{0\}$. We then show the converse of the previous result, i.e., if $f$ is analytic near $N$, then any point in $N$ is the limit of a gradient flow, in other words the limiting set of gradient flow contains $N$. 

\begin{prop}[converse of Theorem \ref{convergence rate of analytic or GMB function}]\label{Prop Converse of gradient flow convergence}
    Let $f: \bR^{(d+1)m} \to [0, +\infty)$ be continuously differentiable. Suppose that each point $x^* \in f^{-1}\{0\}$ has a neighborhood $U$ satisfying 
    \begin{itemize}
        \item [(a)] For any $p \in U$, the gradient flow starting at $p$, $\gamma_p$, converges to a point in $f^{-1}\{0\}$, 
        \item [(b)] There are some $C, \alpha > 0$ such that for any $p \in U$, the curve length of $\gamma_p$ is bounded above by $C \dist{p}{f^{-1}\{0\}}^\alpha$. 
    \end{itemize}
    Then for any $x^* \in f^{-1}\{0\}$, there is a non-constant gradient flow converging to $x^*$ as $t \to +\infty$. In particular, we result holds when $f$ is analytic. 
\end{prop}
\begin{proof}
    By hypothesis, there is a sequence $\{x_j^*\}_{j=1}^\infty$ and a sequence of non-constant gradient flow $\{ \gamma_j \}_{j=1}^\infty$ such that $\lim_{t\to+\infty} \gamma_j(t) = x_j^* \in f^{-1}\{0\}$ and $\limftyj x_j^* = x^* \in f^{-1}\{0\}$ (this can be proved in the same way as we do in the remark above). Choose a compact neighborhood $V \siq U$ of $x^*$ (so $\bar{V} \siq U$). For each $j$, there is a largest $t_j \in \bR$ such that $p_j := \gamma_j (t_j) \in \partial V \cap \gamma_j$. Since $\partial V$ is compact, the sequence $\{p_j\}_{j=1}^\infty$ has an accumulation point $p$ in $V$. Moreover, hypothesis (b) implies that $p \notin f^{-1}\{0\}$. Since $p \in U$, the gradient flow $\gamma_p: [0, +\infty) \to \bR^{(d+1)m}$ converges to a point in $f^{-1}\{0\}$ and its curve length is bounded by $C\dist{p}{f^{-1}\{0\}}^\alpha$. 

    Let $l_j$ be the curve length of $\gamma_j[t_j, +\infty)$ and $l$ be the curve length of $\gamma_p$. For each $j \in \bN$, define $u_j: [0, +\infty) \to \bR^{(d+1)m}$ by 
    \begin{align*}
        u_j(0)       &= p_j; \\
        \dot{u}_j(t) &= - \frac{\nabla f(u_j(t))}{|\nabla f(u_j(t))|}, \quad 0 < t < l_j; \\
        u_j(t)       &= u_j (l_j), \quad\quad\quad\quad\,\,\, t \geq l_j. 
    \end{align*}
    Similarly, define $u: [0,+\infty) \to \bR^{(d+1)m}$ by 
    \begin{align*}
        u(0)       &= p; \\
        \dot{u}(t) &= - \frac{\nabla f(u(t))}{|\nabla f(u(t))|}, \quad 0 < t < l; \\
        u(t)       &= u(l), \quad\quad\quad\quad\,\,\,\, t \geq l. 
    \end{align*}  
    Note that the $u_j$'s and $u$ are exactly the trajectories of their corresponding gradient flows. 

    Fix $\vep > 0$. Choose any $k \in \bN$ with $|p - p_k| < \vep$. There is some $T > 0$ such that for any $t > T$, we have $\dist{u_j(t)}{f^{-1}\{0\}} < \vep^{1/\alpha}$ an $\dist{u(t)}{f^{-1}\{0\}} < \vep^{1/\alpha}$. Then the Grownwall's inequality and hypothesis (b) yield 
    \begin{align*}
        \left| \lim_{s\to+\infty} u(s) - \lim_{s\to+\infty} u_k(s) \right| 
        &\leq | \lim_{s\to+\infty} u(s) - u(t) | + |u(t) - u_k(t)| + |u_k(t) - p_k| \\ 
        &\leq C\dist{u(t)}{f^{-1}\{0\}}^\alpha + \exp(t) \vep + C\dist{u_k(t)}{f^{-1}\{0\}}^\alpha \\ 
        &\leq (2C + \exp(t)) \vep. 
    \end{align*}
    Since $\{p\} \cup \{p_j\}_{j=1}^\infty$ is a subset of the bounded $\partial V$, it follows that $\sup\{l, l_1, l_2, ...\} < +\infty$. In particular, there is some $T' > 0$ such that the $u_j$'s and $u$ are all constant on $T'$. Thus, we actually have 
    \[
        \left| \lim_{s\to+\infty} u(s) - \lim_{s\to+\infty} u_k(s) \right| \leq (2C + \exp(T')) \vep. 
    \]
    Letting $\vep \to 0$, we see that $\lim_{s\to+\infty} u(s) = \limftyk x_k^* = x^*$, which means $u$, and thus $\gamma_p$, converges to $x^*$. This shows the first part of the proposition. 

    Now suppose that $f$ is analytic. Let $x^* \in f^{-1}\{0\}$. By Theorem \ref{convergence rate of analytic or GMB function}. there is a bounded neighborhood $U$ of $x^*$ and some $\beta > 0$ such that for any $p \in U$, $\gamma_p$ converges at rate $\beta$. Since $U$ is bounded and $f$ is smooth, $f$ is Lipschitz on $U$, so there is some $L > 0$ with $|f(z_1) - f(z_2)| \leq L|z_1 - z_2|$ for any $z_1, z_2 \in U$. It follows that the curve length $l(\gamma_p)$ of any such $\gamma_p$ can be estimated by 
    \[
        l(\gamma_p) \leq \tilde{C} |f(p)|^\beta \leq \tilde{C} |f(p) - 0|^\beta \leq \tilde{C} L^\beta |p - q|^\beta, 
    \]
    where $\tilde{C} > 0$ is some constant, and $q \in f^{-1}\{0\}$ is any point satisfying $|p - q| = \dist{p}{f^{-1}\{0\}}$ ($q$ exists because $f^{-1}\{0\}$ closed). This shows $f$ satisfies the hypotheses and the desired result follows.
\end{proof} 

We now apply the results above to our loss function $R$. By Lemma \ref{Lem Rmu is analytic}, $R$ is analytic whenever $\sigma$ is a generic activation, whence Theorems \ref{Lojasiewicz inequality for analytic and GMB function}, \ref{convergence rate of analytic or GMB function} and Proposition \ref{Prop Converse of gradient flow convergence} hold for $R$. \\

\subsection{Convergence Rate and Limiting Direction of gradient flow}\label{Subsection Conv rate and lim direction of gradient flow}

In this part we investigate the convergence rate and limiting direction of a gradient flow near any (fixed) $Q_{P,\pi}^r$ for $r \in \{m_0, ..., m\}$, given that $Q_{P,\pi}^r$ is separated. As we shall see, the behavior of gradient flows depends largely on $r$ and partition $P$. \\

Let's begin with two general lemmas about a ``parameterized version of center manifold theory". Consider a dynamical system of the form 
\begin{equation}\label{Eq Generalized system concerning center manifold}
\begin{aligned}
    \dot{x}(t) &= f(x(t), y(t), z(t)); \\
    \dot{y}(t) &= H(z(t)) y(t) + g(x(t), y(t), z(t)); \\
    \dot{z}(t) &= m(x(t), y(t), z(t)), 
\end{aligned}
\end{equation}
where $(x, y, z) \in \bR^c \times \bR^s \times \bR^p$ and $f,g,H,m$ are $C^3$ functions, and there is an open $U_p \siq \bR^p$ such that for any $z \in U_p$, we have i) $H(z)$ is negative definite, ii) $f(0,0,z) = g(0,0,z) = m(0,0,z) = 0$, and iii) when $x,y \to 0$ the maps $f, g, m$ can be estimated as 
\begin{align*}
    f(x,y,z) &= O(|x|^2 + |y|^2); \\
    g(x,y,z) &= O(|x|^2 + |y|^2); \\
    m(x,y,z) &= O(|x|^2 + |y|^2). 
\end{align*} 
Specifically, when $c = 0$, we identify the space $\bR^0 \times \bR^s \times \bR^p$ with $\bR^s \times \bR^p$, so that the system (\ref{Eq Generalized system concerning center manifold}) becomes
\begin{align*}
    \dot{y}(t) &= H(z(t)) y(t) + g(y(t), z(t)); \\
    \dot{z}(t) &= m(y(t), z(t)). 
\end{align*}
Moreover, for each $z \in U_p$, We also consider the following simplified system 
\begin{equation}\label{eq for parametrized center manifold}
\begin{aligned}
    \dot{x}(t) &= f(x(t), y(t), z); \\
    \dot{y}(t) &= H(z) y(t) + g(x(t), y(t), z).
\end{aligned}
\end{equation}

\begin{lemma}\label{convergence direction}
    Consider the system (\ref{Eq Generalized system concerning center manifold}). Let $\gamma = (x, y, z): [0,+\infty) \to \bR^{c+s+l}$ be a solution curve to (\ref{Eq Generalized system concerning center manifold}) that converges to some $(0,0,z^*)\in \bR^{c+s+l}$ with $z^* \in U_p$. 
    \begin{itemize}
        \item [(a)] If $c = 0$ then $|z(t) - z^*| = O(|y(t)|^2)$ as $t \to +\infty$. 
        
        \item [(b)] If $c > 0$ and $\varliminf_{t \to +\infty, x(t) \neq 0} \frac{|y(t)|}{|x(t)|^2} < +\infty$, then $|y(t)| = O(|x(t)|^2)$ as $t \to +\infty$. 
    \end{itemize} 
\end{lemma}
\begin{proof}
\begin{itemize}
\item [(a)]  Since $\dot{y}(t) = H(z(t)) y(t) + O(|y(t)|^2)$ when $c = 0$, for each $1 \leq j \leq s$, $|y_j(t)|$ decreases to 0 at exponential rate. Thus, there is some $T > 0$ such that for any $\zeta \geq t > T$ and for any $1 \leq j \leq s$, we have $|y_j(t)| < 1$, $|\dot{y}_j(t)| \geq \frac{|\lambda_j|}{2} |y_j(t)|$, and 
\begin{equation*}
    D_1 |y_j(t)| e^{-\mu_1(\zeta-t)} \leq |y_j(\zeta)| \leq D_2 |y_j(t)| e^{-\mu_2(\zeta-t)} 
\end{equation*}
for some $D_1, D_2 \geq 0, \mu_1, \mu_2 > 0$ depending only on $T$. These assumptions imply that $y_j$ is decreasing and $y_j = 0$ or the sign of $y_j$ does not change on $(T, +\infty)$. Thus, for any $t > T$, 
\begin{align*}
    \int_t^\infty y_j^2(\zeta) d\zeta 
    &\leq \sum_{k=\floor{t}}^\infty y_j^2(k) \\ 
    &\leq \left( \sum_{k=\floor{t}}^\infty y_j(k) \right)^2 \\ 
    &\leq \frac{4}{\lambda_j^2}\left( \int_{\floor{t}}^\infty \dot{y}_j(\zeta) d\zeta \right)^2 = \frac{4}{\lambda_j^2} y_j^2(\floor{t}), 
\end{align*} 
where $\floor{t}$ denotes the largest integer smaller than $t$. But then 
\begin{equation*}
    |y_j(t)| \geq D_1 |y_j(\floor{t})| e^{-\mu_1(t - \floor{t})} \geq D_1 |y_j(\floor{t})| e^{-\mu_1}. 
\end{equation*}
This means there is some $C_j > 0$ with $y_j^2(\floor{t}) \leq C_j y_j^2(t)$ whenever $t > T$. Since $m(x, y, z) = O(|y|^2)$ for $z \in U_p$, there is some $C_z > 0$ with $|\dot{z}(t)| \leq C_z |y(t)|^2$ for $t > T$. It follows that when $t > T$, 
\begin{align*}
    |z(t) - z^*| \leq \int_t^\infty |\dot{z}(\zeta)|d\zeta 
    &\leq C_z \sum_{j=1}^s \int_t^\infty y_j^2(\zeta) d\zeta \\ 
    &\leq C_z \sum_{j=1}^s \frac{4}{\lambda_j^2} C_j y_j^2(t) \\ 
    &\leq \left( C_z \sum_{j=1}^s \frac{4}{\lambda_j^2} C_j \right) |y(t)|^2, 
\end{align*}
which shows that $|z(t) - z^*| = O(|y(t)|^2$ as $t \to +\infty$. 
    
\item [(b)] Without loss of generality, assume that $H(z^*)$ is diagonal, and its eigenvalues are $\lambda_1, ..., \lambda_s$; denote $\lambda := \frac{1}{2}\max_{1\leq j\leq s} \lambda_j$. Let $C_y > 0$ be a constant such that $|g(x,y,z)| \leq C_y (|x|^2 + |y|^2)$. Consider the quotient 
\begin{equation*}
    Q(t) := \frac{|y(t)|^2}{|x(t)|^4} 
\end{equation*} 
for $t \in [0,+\infty)$ such that $x(t) \neq 0$. Suppose that (b) does not hold, then $\varlimsup_{t\to +\infty, x(t)\neq 0} Q(t) = +\infty$. By hypothesis, $\varliminf_{t\to +\infty, x(t) \neq 0} Q(t) < +\infty$, so by continuity of $Q$ we can find a $k_0 > 0$ and a sequence $\{t_n\}_{n=1}^\infty \to +\infty$ such that $|\lambda k_0| > 2C_y$, and $|y(t_n)| = k_0 |x(t_n)|^2$, $\dot{Q}(t_n) \geq 0$ for each $n \in \bN$. We will show that this gives a contradiction. A straightforward computation yields 
\begin{equation*}
    \dot{Q}(t) = \frac{2 \sum_{j=1}^s y_j(t)\dot{y}_j(t)}{|x(t)|^4} - \frac{4 |y(t)|^2 \sum_{j=1}^c x_j(t)\dot{x}_j(t)}{|x(t)|^6} 
\end{equation*} 
when $x(t) \neq 0$. As $t_n \to +\infty$, both $x(t_n), y(t_n) \to 0$ by hypothesis, thus for sufficiently large $t_n$ we have $|y(t_n)| \leq |x(t_n)|$ and $|(H(z(t))y(t))_j| \geq \frac{|\lambda_j|}{2} |y_j(t)|$ for all $j \in \{1, ..., s\}$. Then for such $t_n$, 
\begin{align*}
    \sum_{j=1}^s y_j(t_n)\dot{y}_j(t_n) 
    &= \<y(t_n), H(z(t)) y(t_n)\> + \< y(t_n), g(x(t_n), y(t_n), z(t_n))\> \\
    &\leq \lambda |y(t_n)|^2 + |y(t_n)| C_y(|x(t_n)|^2 + |y(t_n)|^2) \\
    &\leq \lambda |y(t_n)|^2 + 2C_y |y(t_n)||x(t_n)|^2 \\ 
    &\leq k_0 (\lambda k_0 + 2C_y) |x(t_n)|^4. 
\end{align*}
Since $\lambda < 0$, by our assumption we have $k_0 (\lambda k_0 + 2C_y) < 0$. On the other hand, for any $n \in \bN$, 
\begin{align*}
    |y(t_n)|^2 \left| \sum_{j=1}^s x_j(t_n) \dot{x}_j(t_n) \right| 
    &\leq k_0^2 |x(t_n)|^4 |x(t_n)||\dot{x}(t_n)| \\ 
    &\leq k_0^2 |x(t_n)|^5 m(x(t_n), y(t_n), z(t_n)) \\ 
    &= O(|x(t_n)|^7). 
\end{align*}
Therefore, 
\begin{equation*}
    \dot{Q}(t_n) \leq 2k_0(\lambda k_0 + 2C_y) - O(|x(t_n)|) 
\end{equation*} 
from which we can see that if $N \in \bN$ is chosen so that for any $n > N$, $O(|x(t_n)|) \leq k_0 (-\lambda k_0 - 2C_y)$, we would have $\dot{Q}(t_n) \leq k_0 (\lambda k_0 + 2C_y) < 0$, contradicting our assumption that $\dot{Q}(t_n) > 0$ for each $n$. 
\end{itemize}
\end{proof}
\begin{remark}\label{Rmk gradient flow converges at linear rate}
    When $c > 0$, we do not know much about the relationship between $|z(t) - z^*|$ and $|x(t)|$. The best estimate we know is
    \begin{align*}
        |z(t) - z^*| = O\left( \int_t^\infty |\dot{z}(s)| ds \right)
        &= O\left( \int_t^\infty \left(|x(s)|^2 + |y(s)|^2 \right) ds \right) \\
        &= O\left( \int_t^\infty |x(s)|^2 ds \right)
    \end{align*}
    as $t \to +\infty$. 
\end{remark}

Then we focus on systems with $c > 0$. Note that when $c = 0$, $|y(t)| = O(e^{-\beta t})$ for some $\beta > 0$; this is just similar to the Morse--Bott case. 

\begin{lemma}\label{convergence for degenerate case}
    Consider the systems (\ref{Eq Generalized system concerning center manifold}) and (\ref{eq for parametrized center manifold}) with $c > 0$. Suppose that for each $z \in U_p$, $(0,0,z)$ is an asymptotically stable equilibrium of (\ref{eq for parametrized center manifold}). Let $\gamma = (x, y, z): [0,+\infty) \to \bR^{c+s+p}$ be a solution curve to (\ref{Eq Generalized system concerning center manifold}) that converges to some $(0,0,z^*)\in \bR^{c+s+p}$ with $z^* \in U_p$. Then there is some $\beta > 0$ such that $|y(t)| = O(|x(t)|^2 + e^{-\beta t})$ as $t \to +\infty$. Moreover, $\beta$ can be made as close to the largest negative eigenvalue of $H$ as possible. 
\end{lemma}
\begin{proof}
Given $z \in U_p$, let $h_z$ be the center manifold of the equation (\ref{eq for parametrized center manifold}). In this way we obtain a family of center manifolds $\{h_z: z\in U_p\}$. Since $\lim_{t\to +\infty} \gamma(t) = (0,0,z^*)$ and since $z^* \in U_p$, there is some $T > 0$ such that for any $t \geq T$, $h_{z(t)}(x(t))$ exists. Define a map $\delta: [T, +\infty) \to \bR^s$, $\delta(t) = y(t) - h_{z(t)}(x(t))$. Clearly, 
\begin{equation*}
    \dot{\delta}(t) = H(z(t)) y(t) + g(x(t), y(t), z(t)) - Dh_{z(t)}(x(t)) f(x(t), y(t), z(t)). 
\end{equation*}
Since $h_{z(t)}$ is a center manifold for equation (\ref{eq for parametrized center manifold}) with $z = z(t)$, we have 
\begin{equation*}
    H(z(t)) h_{z(t)}(x(t)) + g(x(t), y(t), z(t)) = Dh_{z(t)}(x(t)) f(x(t), y(t), z(t)). 
\end{equation*}
Since $y(t), h_{z(t)}(x(t)) \to 0$ as $t \to +\infty$; this, together with $f(x,y,z(t)) = O(|x|^2 + |y|^2)$ and $g(x,y,z(t)) = O(|x|^2 + |y|^2)$, yields 
\begin{align*}
    &\,\,\,\,\,\,\,g(x(t), y(t), z(t)) - g(x(t), h_{z(t)}(x(t)), z(t)) \\
    &- Dh_{z(t)}(x(t)) [f(x(t), h_{z(t)}(x(t)), z(t)) - f(x(t), y(t), z(t))] = o(|\delta(t)|). 
\end{align*}
Therefore, 
\begin{align*}
    \dot{\delta}(t) 
    &= H(z(t)) \delta(t) + [g(x(t), y(t), z(t)) - g(x(t), h_{z(t)}(x(t)), z(t))] \\
    &\,\,\,\,\,\,\,- Dh_{z(t)}(x(t)) [f(x(t), h_{z(t)}(x(t)), z(t)) - f(x(t), y(t), z(t))] \\ 
    &= H(z(t)) \delta(t) + o(|\delta(t)|). 
\end{align*}
Since each $H(z(t))$ is negative definite and $\lim_{t\to +\infty} z(t) = z^*$, the continuity of $H$ implies that $\delta(t)$ decreases at exponential rate; in particular, there are $C, \beta > 0$ such that $|\delta(t)| \leq C_1 e^{-\beta t}$. 

Then we show that $\beta$ can be made as close to the largest negative eigenvalue of $H$ as possible. For each $z \in U_p$ the curve $(x(t), h_z(x(t)), z)$ is just the solution to the system 
\begin{equation}
\begin{aligned}
    \dot{x}(t) &= f(x(t), y(t), z(t)); \\
    \dot{y}(t) &= H(z(t))y(t) + g(x(t), y(t), z(t)); \\
    \dot{z}(t) &= 0, 
\end{aligned}
\end{equation}
with an initial value $(x(0), h_z(x(0)), z)$. Thus, by the smooth dependence of an autonomous system on initial value, we see that $h(x,z) := h_z(x)$ is twice continuously differentiable in both $x$ and $z$. Intuitively, this family of center manifolds deform smoothly. 

Since $f(x,y,z) = O(|x|^2 + |y|^2)$ and $g(x, y, z) = O(|x|^2 + |y|^2)$ for each fixed $z \in U_p$, if $\psi_i = \sum_{j,k} \alpha_{ijk} x_j x_k$ and $\psi = (\psi_1, ..., \psi_s)^\TT$, we have $\psi(x) = O(|x|^2)$ and $D\psi(x) = O(|x|)$, whence 
\begin{align*}
    &\,\,\,\,\,\,\,D\psi(x) f(x, \psi(x), z) - H\psi(x) - g((x, \psi(x), z) \\
    &= O(|x|) O(|x|^2 + O(|x|^4)) - O(|x|^2) + O(|x|^2 + O(|x|^4)) \\ 
    &= O(|x|^2). 
\end{align*} 
Thus, the approximation theory of center manifold (see e.g., Section 2.5 of \cite{JCarr}) implies that $h_z(x) = O(|x|^2)$. In particular, since $h$ is smooth in both $x$ and $z$, there is some $C_2 > 0$ such that for any $x$ close to $0$ and $z$ close to $z^*$, $|h_z(x)| \leq C_2|x|^2$. It follows that for $t$ large, 
\begin{equation}\label{partial result of convergence for degenerate case}
    |y(t)| \leq |h_{z(t)}(x(t))| + C_1 e^{-\beta t} \leq C_2|x(t)|^2 + C_1 e^{-\beta t}. 
\end{equation}
Let $\lambda$ be the largest eigenvalue of $H(z^*)$. Fix $\vep > 0$. In the proof above we have shown that $\dot{\delta}(t) = H(z(t))\delta(t) + o(|\delta(t)|)$ and $\delta(t) \to 0$, $z(t) \to z^*$ as $t \to +\infty$. Thus, there is some $T > 0$ such that for any $t \geq T$, $\< H(z(t))\delta(t), \delta(t)\> \geq (-\lambda - \vep)|\delta(t)|$ and the $o(|\delta(t)|) \leq \varepsilon |\delta(t)|$. It follows that 
\begin{equation*}
    |\dot{\delta}(t)| \geq ||H\delta(t)| - \vep|\delta(t)|| \geq (-\lambda - 2\vep) |\delta(t)|, 
\end{equation*}
and thus $\delta(t) = O(e^{(\lambda +2\vep) t})$ as $t \to +\infty$. 
\end{proof}

\begin{remark}\label{Rmk gradient flow may converge at sublinear rate}
    Observe that if $|x(t)| \geq \Omega (e^{-\frac{\mu}{2}t})$ for some $\mu > -\lambda$, then $|y(t)| = O(|x(t)|^2)$ as $t \to +\infty$. In general, we may not expect that the trajectory of $\gamma$ is ``biased" towards $h_{z(t)}$ in the sense that $|y(t)| = O(|x(t)|^2)$ as $t \to +\infty$, but if the tail length
    \begin{equation*}
        l_{xy}(\gamma)(t) := \int_t^\infty \sqrt{|\dot{x}(s)|^2 + |\dot{y}(s)|^2} ds = \Omega (e^{-\mu t}) 
    \end{equation*} 
    for some $\mu < -\lambda$, we have $|y(t)| = O(|x(t)|^2)$. To see this, choose any $\beta > 0$ as in Lemma \ref{convergence for degenerate case} with $\beta > \mu$. We have $|\dot{x}(s)|^2 + |\dot{y}(s)|^2 \leq C|x(s)|^4 + C'e^{-2\beta s}$ and thus 
    \begin{align*}
        De^{-\mu t} \leq l_{xy}(\gamma)(t) 
        &\leq \int_t^\infty \sqrt{C}|x(s)|^2 + \sqrt{C'} e^{-\beta s} ds \\ 
        &\leq \sqrt{C} \int_t^\infty |x(s)|^2 ds + \frac{\sqrt{C'}}{\beta} e^{-\beta t}  
    \end{align*} 
    for some constants $C, C', D > 0$. Thus, there is a sequence $\{t_n\}_{n=1}^\infty$ diverging to $\infty$ and some $D' > 0$ such that for any $n \in \bN$, $|x(t_n)| \geq D' e^{-\frac{\mu}{2} t_n}$. Using $|y(t)| = O(|x(t)|^2 + e^{-\beta t})$, we have $\varliminf_{t\to +\infty, x(t)\neq 0} \frac{|y(t)|}{|x(t)|^2} < +\infty$. By Lemma \ref{convergence direction}, this implies $|y(t)| = O(|x(t)|^2)$, as desired. 
\end{remark}

\begin{figure}[H]
    \centering
    \includegraphics[width=0.65\textwidth]{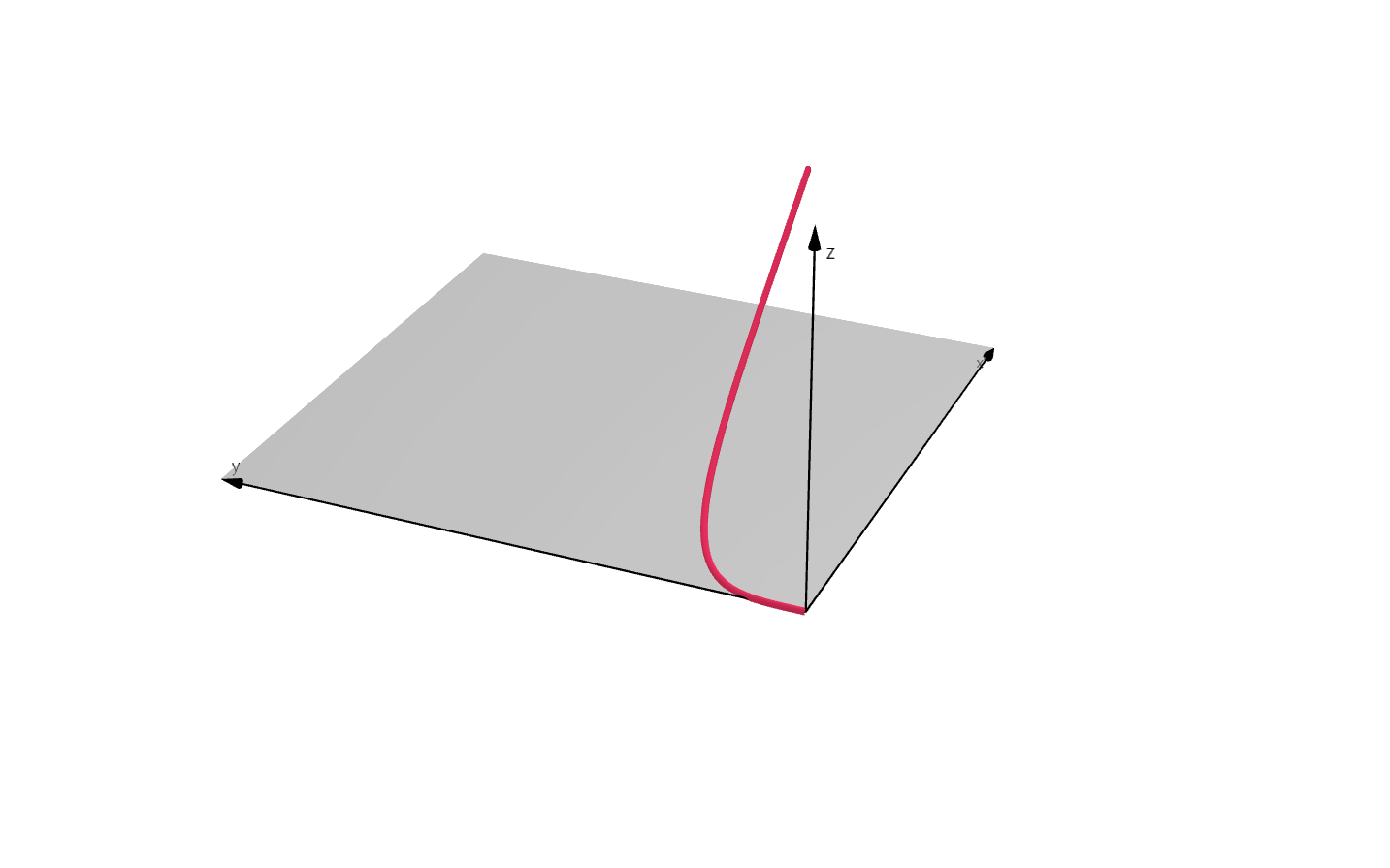}
    \caption{Limiting direction of the solution curve (blue curve) of system (\ref{Eq Generalized system concerning center manifold}) with $c > 0$. Viewed from $xy$-plane, this curve is ``biased towards" $y$-direction, which is non-degenerate.}
    %\label{fig:enter-label}
\end{figure}

Let $\theta \in \bR^{(d+1)m}$, consider $\operatorname{Hess}R(\theta)$. Let $V_s \siq \bR^{(d+1)m}$ be the largest subspace such that for any $v \in V_s \cut \{0\}$, $\<v, \operatorname{Hess}R(\theta) v\> \neq 0$. Let $V_p \siq (V_s)^\bot$ be the largest subspace such that $R|_{\theta + V_p}$ is locally constant near $\theta$. Let $V_c = (V_s + V_p)^\bot$. Define linear operators on $\bR^{(d+1)m}$: 
\begin{itemize}
    \item [i)]  $\pi_s$ be the orthogonal projection onto $V_s$. 
    \item [ii)] $\pi_c$ be the orthogonal projection onto $V_c$. 
    \item [iii)]$\pi_p$ be the orthogonal projection onto $V_p$. 
\end{itemize}
Finally, given the decomposition $\bR^{(d+1)m} := V_c \oplus V_s \oplus V_p$ as above, we write any element $\theta \in \bR^{(d+1)m}$ as $\theta := (x, y, z)$, where $x = \pi_c(\theta), y = \pi_s(\theta)$ and $z = \pi_p(\theta)$. \\
\textbf{Example. } Suppose that $R(x, y, z) = x^2 + y^4 + 0z$. Then 
\[
    \operatorname{Hess}R(x, y, z) = \maThree{2}{0}{0}{0}{12y^2}{0}{0}{0}{0}. 
\]
Thus, for $\operatorname{Hess}R(0,0,0)$, $V_s, V_c, V_p$ are the $x$-axis, $y$-axis and $z$-axis of $\bR^3$. \\

We are now ready to apply the lemmas above to the training dynamics of $R$. First we investigate the properties of the gradient flows near $R^{-1}\{0\}$, especially when it is near $Q^*$. Recall that $Q^* = \bigcup_{r=m_0}^m \left( \bigcup_{P,\pi} Q_{P,\pi}^r \right)$, whence by an appropriate coordinate transformation the gradient flow near some $Q_{P,\pi}^r$ becomes system (\ref{Eq Generalized system concerning center manifold}). Since $Q_{P,\pi}^r$ has the same geometry as $Q_P^r$, we will only prove the cases for $Q_P^r$'s; the results then apply to $Q_{P,\pi}^r$ for all permutation $\pi$'s on $\{1, ..., m\}$. 

\begin{thm}[convergence of gradient flow]\label{Thm Convergence rates of gradient flow} 
    For $\theta^* \in \Rmuzero$, let $\gamma$ denote any gradient flow converging to $\theta^*$. Let $\pi_s, \pi_c, \pi_p$ be defined as above. Given $m_0 \le r \le m$ and $n$ separating inputs of $\theta^*$. The following results hold. 
    \begin{itemize}
        \item [(a)] Suppose that $n \geq m + (m + m_0 - r)d$ and $\frac{m_0+m}{2} \le r \le m$. If $P$ has deficient number $l = 2r - m - m_0$, then for almost all $\theta^* \in Q_P^r$, $\gamma$ converges to $\theta^*$ at linear rate and satisfies $|\pi_p(\gamma(t) - \theta^*)| = O(|\pi_s(\gamma(t) - \theta^*)|)$ as $t \to +\infty$. When $n \le (d+1)m_0$, the same result holds for any $\theta^* \in \Rmuzero$ sufficiently close to $Q^*$.
        
        \item [(b)] Suppose that $n \geq m + (m + m_0 - r)d$. If $m_0 \leq r < \frac{m+m_0}{2}$ or the deficient number $l$ of $P$ satisfies $l < 2r - m - m_0$, then for almost all $\theta^* \in Q_P^r$, if 
        \begin{equation*}
            l(\gamma)(t) := \int_t^\infty \sqrt{|\pi_c(\gamma(\zeta) - \theta^*)|^2 + |\pi_s(\gamma(\zeta)-\theta^*)|^2} d\zeta = \Omega(e^{-\mu t}) 
        \end{equation*}  
        for some $\mu$ greater than the largest negative eigenvalue of $\operatorname{Hess} R(\theta^*)$, then $|\pi_s(\gamma(t) - \theta^*)| = O(|\pi_c(\gamma(t) - \theta^*)|^2)$ as $t \to +\infty$. 
        
        \item [(c)] In particular, (a) holds for all $\theta^* \in Q_P^m$ and (b) holds for all $\theta^* \in Q_P^{m_0}$. 
    \end{itemize}
\end{thm}
\begin{proof}
\begin{itemize}
    \item [(a)] First assume that $n \geq m + (m + m_0 - r)d$ and $\frac{m_0+m}{2} \le r \le m$. For almost all $\theta^* \in Q_P^r$ we have 
    \[
        s = \operatorname{rank}\left(\operatorname{Hess}R(\theta^*)\right) = \operatorname{codim}\,Q_P^r. 
    \]
    Fix any such $\theta^*$. Find some $U \ni \theta^*$ open, such that for any $\ttheta^* \in U \cap Q_P^r$ we have $\operatorname{rank} \operatorname{Hess}R(\ttheta^*) = \operatorname{rank}\left(\operatorname{Hess}R(\theta^*)\right)$. Let $\tau: U \to \bR^s \times \bR^{(d+1)m - s}$ be a (smooth) coordinate transformation such that $\tau(U \cap Q_P^r) \siq \{0\} \times \bR^p$, where $p = (d+1)m - s$, and let $(0, z^*) := \tau(\theta^*)$, where $0 \in \bR^{s}$ and $z^* \in \bR^{p}$. In this way we obtain the system (\ref{Eq Generalized system concerning center manifold}) with $c = 0$ and $U_p = \tau(U)$. By Remark \ref{Rmk gradient flow converges at linear rate}, $\tau(\gamma(t)) = (y(t), z(t))$ converges to $(0, z^*)$ at linear rate and by Lemma \ref{convergence direction} (a), $|z(t) - z^*| = O(|y(t)|^2)$. By transforming back to the original coordinate system, we see that $\gamma(t)$ converges at linear rate and $|\pi_p(\gamma(t) - \theta^*)| = O(|\pi_s(\gamma(t) - \theta^*)|)$. 

    Now assume that $n \le (d+1)m_0$. By Lemma \ref{separated branch} (a), there is some open $U \siq \bR^{(d+1)m}$ such that $Q^* \siq U \cap \Rmuzero$ and $R$ is Morse--Bott at $U \cap \Rmuzero$. In particular this means by applying a coordinate transformation we obtain the system (\ref{Eq Generalized system concerning center manifold}) with $c = 0$. Arguing in the same way as above, we can see that $\gamma(t) \to \theta^*$ at linear rate and $|\pi_p(\gamma(t) - \theta^*)| = O(|\pi_s(\gamma(t) - \theta^*)|)$, as $t \to +\infty$. 
    
    \item [(b)] Let $s$ be the maximum of $\operatorname{rank}\left(\operatorname{Hess}R(\theta^*)\right)$ for $\theta^* \in Q_P^r$. By Lemma \ref{separated branch}, for almost all $\theta^* \in Q_P^r$ we have $\operatorname{rank}\left(\operatorname{Hess}R(\theta^*)\right) = s$. Fix any such $\theta^*$. Find some $U \ni \theta^*$ open, such that for any $\ttheta^* \in U \cap Q_P^r$, $\operatorname{rank}\left(\operatorname{Hess}R(\ttheta^*)\right) = s$. Since $r < \frac{m_0+m}{2}$, by Lemma \ref{separated branch} $s < \operatorname{codim}\,Q_P^r$, so $c := \operatorname{codim}\,Q_P^r - s > 0$. Also define $p := \dim Q_P^r$. 

    At each $\ttheta^* \in Q_P^r$ there is a coordinate transformation $\tau_{\ttheta^*}: U \cap (\ttheta^* + (Q_P^r)^\bot) \to \bR^c \times \bR^s$, parametrized smoothly by $\ttheta^*$, such that 
    \begin{itemize}
        \item [i)]  $\tau(U \cap (\ttheta^* + V_s)) = \tau(\ttheta^*) + \{0\} \times \bR^s$, where $V_s \siq \bR^{(d+1)m}$ is the largest subspace such that for any $v \in V_s \cut \{0\}$, $\<v, \operatorname{Hess}R(\ttheta^*) v \> \neq 0$. 
        \item [ii)] $\tau(U \cap (\ttheta^* + V_c)) = \tau(\ttheta^*) + \bR^c \times \bR^s$, where $V_c = (V_s + Q_P^r)^\bot$.  
    \end{itemize}
    Since the eigenvectors of $\operatorname{Hess}R(\theta)$ depends smoothly on $\theta$, these $\tau_{\ttheta^*}$'s can be combined into an embedding $\tau: U \to \bR^c \times \bR^s \times \bR^p$. In this way we obtain system (\ref{Eq Generalized system concerning center manifold}) with $U_P = \tau(U)$. By Lemma \ref{convergence direction} (b) and Remark \ref{Rmk gradient flow may converge at sublinear rate}, we thus have $|y(t)| = O(|x(t)|^2)$. Now apply $\tau^{-1}: \ran \tau \to U$ to see that $|\pi_s(\gamma(t) - \theta^*)| = O(|\pi_c(\gamma(t) - \theta^*)|^2)$. 

    \item [(c)] When $r \in \{m_0, m\}$ and $n \geq r + (m+m_0-r)d$, the rank of $\operatorname{Hess} R$ is constant on $Q_P^r$.
\end{itemize}
\end{proof}

Recall that any gradient flow $\gamma$ sufficiently close to $\Rmuzero$ converges to $\Rmuzero$ (Theorem \ref{convergence rate of analytic or GMB function}). We then investigate the convergence rates of $R(\gamma)$. In particular, we show that even if $\gamma$ does not converge at linear rate, $R(\gamma)$ decreases to $0$ very quickly. 

\begin{cor}[convergence of loss]\label{Cor Convergence rates of loss}
    With the notations in Theorem \ref{Thm Convergence rates of gradient flow}, given $m_0 \le r \le m$ and $n$ separating inputs. The following results hold. 
    \begin{itemize}
        \item [(a)] Suppose that $n \geq m + (m + m_0 - r)d$ and $\frac{m_0+m}{2} \le r \le m$ and $P$ has deficient number $l = 2r - m - m_0$. For almost all $\theta^* \in Q_P^r$, we have $R(\gamma(t)) \to R(\theta^*) = 0$ at linear rate as $t \to +\infty$. When $n \le (d+1)m_0$, the same result holds for any $\theta^* \in \Rmuzero$ sufficiently close to $Q^*$. 

        \item [(b)] Suppose $m_0 \leq r < \frac{m+m_0}{2}$ and the deficient number $l$ of $P$ satisfies $l < 2r - m - m_0$. For almost all $\theta^* \in Q_P^r$, there is some $\beta > 0$ such that $R(\gamma(t)) = O(|\gamma(t) - \theta^*|^4 + e^{-\beta t})$ for all sufficiently large $t$. 
    \end{itemize}
\end{cor}
\begin{proof}~
\begin{itemize}
    \item [(a)] Since $R$ is analytic, it is in particular locally Lipschitz, so for any bounded open $U \siq \bR^{(d+1)m}$ containing $\theta^*$, there is a constant $c > 0$ with 
    \[
        R(\theta) = |R(\theta) - R(\theta^*)| \leq c|\theta - \theta^*|. 
    \]
    By Theorem \ref{Thm Convergence rates of gradient flow}, for almost all $\theta^* \in Q_P^r$ we have $|\gamma(t) - \theta^*| \leq c'e^{-\beta t}$ for some $c', \beta > 0$. Therefore, 
    \[
        R(\gamma(t)) \leq c|\gamma(t) - \theta^*| \leq cc' e^{-\beta t}. 
    \]
    In other words, $R(\gamma(t)) \to 0$ at linear rate as $t \to +\infty$. 

    Now assume that $n \le (d+1)m_0$. Then there is some open $U \siq \bR^{(d+1)m}$ such that $Q^* \siq U \cap \Rmuzero$ and $R$ is Morse--Bott at $U \cap \Rmuzero$. Thus, for any $\theta^* \in U \cap \Rmuzero$ we can apply Theorem \ref{Thm Convergence rates of gradient flow} and argue in the same way as above to deduce that $R(\gamma(t)) \to 0$ at linear rate as $t \to +\infty$. 

    \item [(b)] By Theorem \ref{Thm Convergence rates of gradient flow} (b), we actually have for almost all $\theta^* \in Q_P^r$, $|\pi_s(\gamma(t) - \theta^*)| = O(|\pi_c(\gamma(t) - \theta^*)|^2 + e^{-\mu t})$ as $t \to +\infty$. for some $\mu > 0$. Using the Talyor expansion of $R$ near $\theta^*$, we can write 
    \begin{align*}
        R(\theta) = R(\theta^*) &+ \< \nabla R(\theta^*), \theta - \theta^*\> \\ 
                                &+ \frac{1}{2}\< \operatorname{Hess}R(\theta^*)(\theta - \theta^*), \theta - \theta^*\> + O(|\theta - \theta^*|^3). 
    \end{align*}
    Since $R$ is analytic and $R \geq 0$, we can further write and simplify it as 
    \begin{align*}
        R(\theta) 
        &= \frac{1}{2}\< \operatorname{Hess}R(\theta^*)(\theta - \theta^*), \theta - \theta^*\> + O(|\theta - \theta^*|^4) \\ 
        &= \frac{1}{2}\< \operatorname{Hess}R(\theta^*)\pi_s(\theta - \theta^*), \pi_s(\theta - \theta^*)\> + O(|\theta - \theta^*|^4) \\ 
        &\leq \frac{1}{2} \norm{\operatorname{Hess}R(\theta^*)} |\pi_s(\theta - \theta^*)|^2 + O(|\theta - \theta^*|^4). 
    \end{align*}
    Since $\lim_{t\to+\infty} \gamma(t) = \theta^*$, it follows that for sufficiently large $t$ we have 
    \begin{align*}
        R(\gamma(t)) 
        &\leq \frac{1}{2} \norm{\operatorname{Hess}R(\theta^*)} |\pi_s(\gamma(t) - \theta^*)|^2 + O(|\gamma(t) - \theta^*|^4) \\ 
        &= \frac{1}{2} \norm{\operatorname{Hess}R(\theta^*)} O((|\pi_c(\gamma(t) - \theta^*)|^2 + e^{-\mu t})^2) +  O(|\gamma(t) - \theta^*|^4) \\
        &= \frac{1}{2} \norm{\operatorname{Hess}R(\theta^*)} O((|\gamma(t) - \theta^*|^2 + e^{-\mu t})^2) +  O(|\gamma(t) - \theta^*|^4) \\
        &= O(|\gamma(t) - \theta^*|^4 + e^{-\beta t}), 
    \end{align*}
    where $\beta > 0$ is some constant depending on $\mu$. This proves the desired result. 
\end{itemize}
\end{proof}

For clarity, we summarize the results of Theorem \ref{Thm Convergence rates of gradient flow} and Corollary \ref{Cor Convergence rates of loss} in the table below.

\begin{table}[H]\label{tab:2}
\centering
\begin{tabular}{|l|l|l|l|}
\multicolumn{4}{c}{Convergence Rates of Gradient Flows and Loss Near Different Branches\vs{0.4em}}\\
\hline
\textbf{Sample size}                  & \textbf{$Q_{P,\pi}^r$ condition}                                                      & \textbf{Convergence of GF} & \textbf{Convergence of Loss}                          \\ \hline
$\le (d+1)m_0$                        & arbitrary $Q_{P,\pi}^r$                                                               & Linear rate                & Linear rate                                           \\ \hline
\multirow{2}{*}{$\ge r + (m+m_0-r)d$} & \begin{tabular}[c]{@{}l@{}}$r \ge \frac{m+m_0}{2}$ and \\ $l = 2r-m-m_0$\end{tabular} & Linear rate                & Linear rate                                           \\ \cline{2-4} 
                                      & \begin{tabular}[c]{@{}l@{}}$r < \frac{m+m_0}{2}$ or \\ $l \ne 2r-m-m_0$\end{tabular}  & May not be linear rate     & $R(\gamma) = O(|\gamma - \theta^*|^4 + e^{-\beta t})$ \\ \hline
\end{tabular}
\caption{In the table above, ``GF" refers to gradient flow. The third column focuses on the convergence rate of gradient flow (near $Q_{P,\pi}^r$) and the last column indicates convergence of loss under gradient flow, i.e., convergence of $R(\gamma(t))$ for a gradient flow $\gamma$.}
\end{table}

\subsection{Local Recovery by Gradient Flow}

We end this section with a discussion about whether the points in $Q^*$ are stable under perturbation. Indeed, it is natural to ask by slightly perturbing a $\theta^* \in Q^*$ to some $\theta_0$, do we have $\limftyt \gamma_{\theta_0}(t) \in Q^*$? To study this problem, we first define ``recovery stability" below which works for a more general case. 

\begin{defn}[recovery stability]\label{Defn Recovery stability}
    Let $\theta^* \in \bR^{(d+1)m}$. We say $\theta^*$ is recovery stable if there is some $\delta > 0$ such that for any $\theta_0 \in B(\theta^*, \delta)$, the gradient flow $\gamma_{\theta_0}: [0,+\infty) \to \bR^{(d+1)m}$ with initial value $\theta_0$ converges and satisfies 
    \[
        g(\theta^*, x) = g\left( \lim_{t\to+\infty} \gamma_{\theta_0}(t), x \right)
    \] 
    for all $x \in \bR^d$. If this is not true, we say $\theta^*$ is recovery unstable. Given a subset $E \siq \bR^{(d+1)m}$, $E$ is called recovery stable if every $\theta^* \in E$ is recovery stable; otherwise, we say $E$ is recovery unstable. 
\end{defn} 

The following result is an immediate corollary of Theorem \ref{Thm Convergence rates of gradient flow}, Lemma \ref{separated branch} and Lemma \ref{underparametrized system}. It fully explains when a point in $Q^*$ is recovery (un)stable. \\

\begin{thm}[recovery stability]\label{Thm Recovery stability}
    Given $m_0 \le r \le m$, partition $P$ and permutation $\pi$ and separating inputs $\{x_i\}_{i=1}^n$. Then no point in $Q_{P,\pi}^r$ is recovery stable when $n \le r + (r-l)d$ ($l$ is the deficient number of $P$), and almost all points in $Q_{P,\pi}^r$ are recovery stable when $n \ge r + (m+m_0-r)d$. Moreover, all points in $Q^*$ are recovery stable when $n > (d+1)m$, namely, $Q^*$ is recovery stable. 
\end{thm}
\begin{proof}
    The desired result follows from the observation that a point $\theta^* \in Q^*$ is recovery stable if and only if it has a neighborhood $U \siq \bR^{(d+1)m}$ with $U \cap \Rmuzero = U \cap Q^*$. 
    
    So fix any $\theta^* \in Q^*$. First assume that this is not true, there is a sequence $\{\theta_n^*\}_{n=1}^\infty$ in $\Rmuzero \cut Q^*$ converging to $\theta^*$, whence by perturbing $\theta^*$ to any $\theta_n^*$, the gradient flow $\gamma_{\theta_n^*}$ starting at $\theta_n^*$ clearly satisfies 
    \[
        \limftyt \gamma_{\theta_n}(t) = \theta_n^* \in \Rmuzero \cut Q^*. 
    \]
    Conversely, if such a $U$ exists, we can shrink it if necessary, so that for any $\theta_0 \in U$, $\gamma_{\theta_0}$ satisfies 
    \[
        \limftyt \gamma_{\theta_0}(t) \in Q^*. 
    \]
    Since any two point in $Q^*$ represent the same model $f^*$, we clearly have 
    \[
        g(\theta^*, x) = f^*(x) = g(\limftyt \gamma_{\theta_0}(t), x) 
    \]
    for all $x \in \bR^d$. Therefore, $\theta^*$ is recovery stable.
\end{proof}

Below we use Theorem \ref{Thm Recovery stability} to illustrate when the points in each branch $Q^r$ of $Q^*$ become recovery stable as sample size increases.

\begin{table}[H]
\centering
\begin{tabular}{llllll}
\multicolumn{6}{c}{Sample Size and Recovery Stability of Points in $Q^*$\vs{0.4em}}\\
\hline
\multicolumn{1}{|l|}{\textbf{Sample size/Branches}} & \multicolumn{1}{l|}{$Q^{m_0}$} & \multicolumn{1}{l|}{...} & \multicolumn{1}{l|}{$Q^r$} & \multicolumn{1}{l|}{...} & \multicolumn{1}{l|}{$Q^m$} \\ \hline
\multicolumn{1}{|l|}{$\le (d+1)m_0$}                & \multicolumn{1}{l|}{\ding{55}}              & \multicolumn{1}{l|}{$...$}    & \multicolumn{1}{l|}{\textbf{\ding{55}}}              & \multicolumn{1}{l|}{$...$}    & \multicolumn{1}{l|}{\textbf{\ding{55}}}              \\ \hline
\multicolumn{1}{|l|}{$\ge m+m_0d$}                  & \multicolumn{1}{l|}{}              & \multicolumn{1}{l|}{}    & \multicolumn{1}{l|}{}              & \multicolumn{1}{l|}{}    & \multicolumn{1}{l|}{$\checkmark$}  \\ \hline
\multicolumn{1}{|l|}{$\vdots$}                      & \multicolumn{1}{l|}{}              & \multicolumn{1}{l|}{}    & \multicolumn{1}{l|}{}              & \multicolumn{1}{l|}{...} & \multicolumn{1}{l|}{$\vdots$}      \\ \hline
\multicolumn{1}{|l|}{$\ge r + (m+m_0-r)d$}          & \multicolumn{1}{l|}{}              & \multicolumn{1}{l|}{}    & \multicolumn{1}{l|}{$\checkmark$}  & \multicolumn{1}{l|}{...} & \multicolumn{1}{l|}{$\checkmark$}   \\ \hline
\multicolumn{1}{|l|}{$\vdots$}                      & \multicolumn{1}{l|}{}              & \multicolumn{1}{l|}{...} & \multicolumn{1}{l|}{$\vdots$}      & \multicolumn{1}{l|}{...} & \multicolumn{1}{l|}{$\vdots$}      \\ \hline
\multicolumn{1}{|l|}{$\ge m_0+md$}                  & \multicolumn{1}{l|}{$\checkmark$}  & \multicolumn{1}{l|}{...} & \multicolumn{1}{l|}{$\checkmark$}  & \multicolumn{1}{l|}{...} & \multicolumn{1}{l|}{$\checkmark$}  \\ \hline
\multicolumn{1}{|l|}{$> (d+1)m$}                              & \multicolumn{5}{l|}{\quad\quad\quad\quad\quad\,$\checkmark^*$} \\ \hline
\multicolumn{6}{|l|}{$\checkmark^*$: any point in $Q^*$ is recovery stable}\\ \hline
\end{tabular}
\caption{How sample size determines the recovery stability of points in branches of $Q^*$. The left-most column lists the important sample size thresholds. As shown in the table, when $n \le (d+1)m_0$, no point in $Q^*$ is recovery stable. For any $r \in \{m_0, ..., m\}$, when the sample size $n \ge r + (m+m_0-r)d$, in each of the branches $Q^{r}, Q^{r+1}, ..., Q^m$ almost all points are recovery stable. Moreover, when $n > (d+1)m$, i.e., when we are in underparameterized regime, any point in $Q^*$ is recovery stable.}
\end{table}

\section{Conclusion and Discussion}\label{Section Conclusion and discussion}

In this paper, we analyzed the geometry and dynamics of the loss landscape for two-layer neural networks, focusing on the vicinity of global minima. We showed that the global minima with zero generalization error can be partitioned into distinct branches, which become geometrically separated from other global minima as the sample size increases. We identified the sample size thresholds for this separation and demonstrated that, for sufficiently large sample sizes, the loss function $R$ is Morse--Bott at almost all points in certain branches, ensuring non-degenerate Hessians along their normal bundles. Our analysis revealed that gradient flows sufficiently close to the global minima converge to points within these minima, with linear convergence near Morse--Bott branches and sublinear convergence near others. We also introduced the concept of ``recovery stability'', showing that almost all points in certain branches are recovery stable when the sample size is large enough.\\

Our results provide a detailed understanding of the loss landscape and training dynamics of two-layer neural networks, explaining their ability to find well-generalizing solutions even in the overparameterized regime. This work lays a foundation for further studies on the global recovery capabilities and generalization performance of neural networks. We demonstrated that two-layer neural networks can locally recover the target function in the overparameterization regime, guaranteed by the separation of branches of global minima and the convergence properties of gradient flows.\\

Finally, we point out several possible future works. First, the analysis in this paper could be extended to deeper neural networks with more than two layers to understand the increased complexity in the loss landscape. Second, we shall figure out how different activation functions, various initialization schemes and regularization techniques would impact generalization capabilities, training dynamics and performance. Furthermore, we expect empirical studies to validate the theoretical findings and explore their practical implications. Finally, future research on global recovery and generalization could provide a comprehensive understanding of neural networks' ability to generalize from a global perspective.

\section{Acknowledgement} 
This work is sponsored by the National Key R\&D Program of China Grant No. 2022YFA1008200 (T. L., Y. Z.), the National Natural Science Foundation of China Grant No. 12101401 (T. L.), No. 12101402 (Y. Z.),  Shanghai Municipal Science and Technology Key Project No. 22JC1401500 (T. L.), the Lingang Laboratory Grant No.LG-QS-202202-08 (Y. Z.), Shanghai Municipal of Science and Technology Major Project No. 2021SHZDZX0102.

\clearpage
%\section*{Reference}
\bibliography{main_2}

\end{document}